\documentclass{article}

\usepackage{arxiv}

\usepackage{graphicx}
\usepackage{booktabs}
\usepackage{microtype}
\usepackage{placeins}

\usepackage{amsmath}
\usepackage{amssymb}
\usepackage{mathtools}

\usepackage{amsthm}
\usepackage{xcolor}

\usepackage{hyperref}
\usepackage{tikz}
\usepackage{subcaption}
\usepackage[capitalize,noabbrev]{cleveref}
\usepackage{enumitem}
\usepackage{natbib}

\theoremstyle{plain}
\newtheorem{theorem}{Theorem}[section]
\newtheorem{proposition}[theorem]{Proposition}
\newtheorem{lemma}[theorem]{Lemma}
\newtheorem{corollary}[theorem]{Corollary}

\theoremstyle{definition}
\newtheorem{definition}[theorem]{Definition}
\newtheorem{assumption}[theorem]{Assumption}

\theoremstyle{remark}
\newtheorem{remark}[theorem]{Remark}

\newcommand{\bx}{\mathbf{x}}
\newcommand{\bc}{\mathbf{c}}
\newcommand{\bz}{\mathbf{z}}
\newcommand{\by}{\mathbf{y}}
\newcommand{\bn}{\mathbf{n}}
\newcommand{\bw}{\mathbf{w}}
\newcommand{\bv}{\mathbf{v}}
\newcommand{\bs}{\mathbf{s}}
\newcommand{\bzero}{\mathbf{0}}
\newcommand{\bI}{\mathbf{I}}
\newcommand{\bX}{\mathbf{X}}
\newcommand{\bY}{\mathbf{Y}}

\newcommand{\bA}{\mathbf{A}}
\newcommand{\bV}{\mathbf{V}}
\newcommand{\bU}{\mathbf{U}}
\newcommand{\bG}{\mathbf{G}}
\DeclareMathOperator{\divr}{div}
\DeclareMathOperator{\reach}{reach}

\renewcommand{\and}{\hspace{0.8em}}
\hypersetup{
    pdftitle={Understanding Latent Diffusability via Fisher Geometry},
    pdfauthor={Jing Gu, Morteza Mardani, Wonjun Lee, Dongmian Zou, Gilad Lerman},
    colorlinks=true,
    linkcolor=blue,
    citecolor=blue,
    urlcolor=blue
}

\title{Understanding Latent Diffusability via Fisher Geometry}

\author{
Jing Gu\thanks{School of Mathematics, University of Minnesota}
\and
Morteza Mardani\thanks{NVIDIA}
\and
Wonjun Lee\thanks{Department of Mathematics, The Ohio State University}
\and
Dongmian Zou\thanks{Zu Chongzhi Center and DIRC, Duke Kunshan University}
\and
Gilad Lerman\footnotemark[1]
}

\date{}

\begin{document}

\maketitle

\begin{abstract}
Diffusion models often degrade in latent spaces, yet the formal causes remain poorly understood. We quantify latent-space diffusability via the rate of change of the Minimum Mean Squared Error (MMSE) along the diffusion trajectory. Our framework decomposes this MMSE rate into contributions from Fisher Information (FI) and Fisher Information Rate (FIR). We show that isometric embeddings preserve intrinsic FI and establish quantitative intrinsic-FI bounds for a broader class of bi-Lipschitz encoders with controlled weak volume distortion, whereas FIR is governed by the interplay between encoder and data geometries. Our analysis separates four geometric contributions in local stability bounds for Gaussian-smoothed FIR: dimensional compression, tangential distortion, high-frequency encoder curvature, and curvature of data manifold. Experiments across diverse autoencoding architectures provide qualitative support for the geometric mechanisms identified by the theory and show that empirical FI and FIR track several measures of generation quality and latent-space geometry in the settings tested. We establish FI and FIR as a comprehensive analytical framework for understanding latent diffusability. 
\end{abstract}

\section{Introduction}\label{sec:introduction}

Diffusion models have revolutionized generative modeling by transforming noise into structured data through a learned reverse process \citep{ho2020denoising}. While these models excel in high-dimensional pixel spaces, computational constraints often necessitate their application within the compressed latent spaces of pre-trained autoencoders \citep{rombach2022high}. However, a poorly understood phenomenon remains: diffusion models often degrade or fail when applied directly to the latent spaces of standard autoencoders, including Variational Autoencoders (VAEs) \citep{kingma2014auto, shi2025latent}. Skorokhodov et al.~\citep{skorokhodov2025diffusability} refer to this mismatch between diffusion dynamics in pixel space and latent space as the \emph{diffusability} problem. 

Current explanations for this failure typically point to the ``geometric distortion'' introduced by standard encoders, which warp the data manifold and create regions that diverge from the Gaussian prior assumed by the diffusion process \citep{arvanitidis2018latent, lobashev2025hessian}. In response, methods have emerged to mitigate this issue, such as geometry-preserving adaptations~\citep{lee2025gpe}, explicit target latent geometries \citep{cho2023hyperbolic,saez2023neural,yue2026image}, and modified training objectives \citep{kouzelis2025eq,yang2026latent,heek2026unified,baade2026latent} (see Appendix~\ref{sec:related_works} for a comprehensive discussion of these methods). 

Despite these empirical advances, the underlying causes of latent diffusion failure remain theoretically obscured. Crucially, the literature lacks a formal, quantitative framework for measuring how ``diffusible'' a latent space is \emph{before} committing to the costly training of a full diffusion model. 

We demystify latent-space diffusability through the rigorous lens of Fisher geometry \citep{amari2016information}, complementing recent information-geometric perspectives on diffusion trajectories \citep{karczewski2025spacetimediffusionmodelsinformation}. We analyze diffusability via \emph{denoising complexity}, quantified by the rate of change of the Minimum Mean Squared Error (MMSE) along the diffusion trajectory. We show that this denoising complexity fundamentally decomposes into two terms: the Fisher Information (FI) and its rate of dissipation, which we define as the Fisher Information Rate (FIR). We establish change-of-variables bounds for intrinsic FI and localized stability bounds for ambient FIR, linking both quantities to encoder and manifold geometry. This defines model-independent population diagnostics and identifies local geometric conditions associated with FIR preservation. Scalable and statistically reliable pre-training estimation in high dimensions remains a separate practical challenge.

Our main contributions are as follows:
\begin{enumerate}[leftmargin=*, nosep]
\item \textbf{Theoretical Framework:} We formalize the relationship between latent diffusability, denoising complexity (MMSE), and Fisher Information Geometry.
\item \textbf{Geometric Bounds:}
We derive quantitative intrinsic-FI bounds for bi-Lipschitz encoders
with controlled volume distortion
(\S\ref{sec:isometry_fisher_preservation}), together with sufficient
local curvature and near-isometry conditions for controlling ambient
Gaussian-smoothed FIR
(\S\ref{sec:FIR}). These results separate four geometric mechanisms:
dimensional compression, tangential distortion, high-frequency
encoder curvature, and intrinsic manifold curvature.
\item \textbf{Empirical Evaluation:} Experiments on toy, FFHQ, and ImageNet-based representations provide qualitative support for the geometric mechanisms identified by the theory and show that empirical FI and FIR track several measures of latent diffusion performance in the settings tested.
\end{enumerate}

\section{Background}\label{sec:background}

\textbf{Score-based Diffusion Models.} 
Continuous-time score-based models typically transform a data distribution into a noise prior through a forward diffusion process. Let $\mu$ be a Borel probability measure on $\mathbb{R}^D$ concentrated on a low-dimensional manifold $\mathcal{M} \subset \mathbb{R}^D$. We consider a diffusion process parameterized by the noise variance $\tau := \sigma^2(t)$, where a clean sample $\bx_0 \sim \mu$ is corrupted by standard Gaussian noise $\bn \sim \mathcal{N}(\bzero, \bI_D)$, yielding the observation $\bx = \bx_0 + \sqrt{\tau}\bn$ \citep{song2020score}. 
The marginal density of these noisy observations, known as the heat-smoothed density $p_\tau(\bx)$, and the corresponding probability flow ODE evolve as:
\begin{equation}
\label{eq:heat_and_ode}
\frac{\partial p_\tau}{\partial \tau} = \frac{1}{2} \Delta_{\bx} p_\tau \quad \text{and} \quad d\bx = -\frac{1}{2} \nabla_{\bx} \log p_\tau(\bx) d\tau.
\end{equation}
Generating samples requires a neural approximation of the \emph{score function} 
\begin{equation}
\label{eq:score_def}
\bs_\tau(\bx) = \nabla_{\bx} \log p_\tau(\bx).    
\end{equation}
 Manifold geometry complicates learning $\bs_\tau$ via high-frequency variations in normal directions.

\textbf{Latent Diffusion and the Geometry Gap.} 
To reduce computational costs, diffusion often operates within the compressed latent space $\mathcal{Z}$ of a pre-trained autoencoder \citep{rombach2022high}. However, standard VAEs regularize $\mathcal{Z}$ toward a Gaussian prior using Kullback-Leibler divergence, often inducing ``manifold crinkling'' \citep{arvanitidis2018latent}. This phenomenon forces distant points on the data manifold $\mathcal{X}$ into adjacent latent coordinates to satisfy the prior, causing Euclidean distances in $\mathcal{Z}$ to misrepresent the intrinsic data geometry. This disparity creates a fundamental ``geometry gap.'' Recent architectures mitigate these distortions by regularizing latent structure: Geometry Preserving Encoder/Decoder (GPE) aligns latent and data-space Euclidean geometries \citep{lee2025gpe}, while spherical and manifold autoencoders employ non-Euclidean latent spaces to capture complex topologies \citep{saez2023neural, yue2026image}.

\textbf{The Diffusability Issue.}
We use the population denoising resistance in
\eqref{eq:denoising_resistance} as an information-theoretic notion of
diffusability. In this framework, FIR measures the average local
variation of the true population score and its contribution to the
rate of change of MMSE. By itself, FIR does not constitute a theorem
about approximation error, optimization difficulty, or sample
complexity for a finite neural score model.

\textbf{Spectral Artifacts in Latent Spaces.} Compressed latent spaces often exacerbate these irregularities. Spectral analyses show autoencoders introduce high-frequency components or singularities, impeding coarse-to-fine synthesis \citep{skorokhodov2025diffusability}. Even with smoother encodings \citep{lee2025enhancing, lee2025latent}, diffusion remains unstable if the encoded manifold in $\mathcal{Z}$ lacks structural regularity.

\textbf{Limitations of Metric Regularization.} Recent techniques constrain the encoder Jacobian to enforce local isometry \citep{chen2020learning, palma2025enforcing}. While empirically addressing the geometry gap, they lack a formal information-geometric analysis of diffusability. To address this gap, we leverage Fisher geometry to establish a rigorous link between encoder geometry and denoising complexity, explaining why certain latent spaces remain difficult to diffuse despite geometric-preservation properties.

\textbf{Information-Geometric Formalism.} 
We unify the analysis of both data and latent spaces by considering a
generic Borel probability measure $\mu$ on $\mathbb R^k$. This
encompasses both the pixel-space setting, where $k=D$ and $\mu$ is
supported on $\mathcal M$, and the latent-space setting, where $k=d$
and $\mu$ is the pushforward measure under the encoder $E$.

For $k\ge1$ and $\tau>0$, write
\[
    \gamma_\tau^k
    :=
    \mathcal N(\bzero,\tau I_k).
\]
When the ambient dimension is clear, we write simply $\gamma_\tau$.
The heat-smoothed measure
\[
    \mu_\tau
    :=
    \mu*\gamma_\tau^k
\]
admits a probability density $p_\tau(\bx)$ with score
$\bs_\tau(\bx)$ defined in \eqref{eq:score_def}. Whenever $\mu$,
$\tau$, and $k$ are clear from context, we may omit the superscript
$(k)$ and the explicit argument $\mu_\tau$ from our notation.

The \emph{Fisher Information} (FI) of $\mu_\tau$ is its square-averaged score and the \emph{Fisher Information Rate} (FIR) is the dissipation rate of FI. That is, 
\begin{equation}
\label{eq:FI_FIR_def}
\mathcal{I}^{(k)} \equiv \mathcal{I}^{(k)}(\mu_\tau) := \mathbb{E}_{p_\tau} \left[ \|\bs_\tau(\bx)\|^2 \right] \quad \text{and} \quad
\mathcal{R} \equiv \mathcal{R}^{(k)}(\mu_\tau) := -\frac{d}{d\tau} \mathcal{I}^{(k)}(\mu_\tau).
\end{equation}

\textbf{Standing Regularity Assumptions.}
For identities involving MMSE, we assume that the clean measure has
finite second moment:
\[
    \int_{\mathbb R^k}\|x\|^2\,d\mu(x)<\infty.
\]
For the heat-semigroup differentiation and integration-by-parts
identities below, we assume that the displayed quantities are finite
and that differentiation under the integral sign and cutoff-based
integrations by parts are justified. These conditions hold for the
compactly supported and Gaussian measures used in our theoretical
and numerical examples.

The next proposition, proved in Appendix~\ref{sec:proof_FIR}, establishes a representation of FIR, which emerges in  information theory \citep{costa1985similarity}, diffusion semigroups \citep{bakry1985diffusions}, and optimal transport \citep{villani2009optimal}. Let $\mathbf{H}_\tau = \nabla^2_{\bx} \log p_\tau(\bx)$.

\begin{proposition}[Hessian representation of FIR]\label{prop:FIR_identity}
For every $\tau>0$, 
\begin{equation}\label{eq:FIR_identity_measure}
\mathcal{R}^{(k)}(\mu_\tau) = \mathbb{E}_{p_\tau} \left[ \|\nabla_{\bx} \bs_\tau(\bx)\|_F^2 \right] = \mathbb{E}_{p_\tau} \left[ \mathrm{Tr}(\mathbf{H}_\tau^2) \right].
\end{equation}
\end{proposition}

\textbf{Computational Efficiency.}
While the FIR involves the Frobenius norm of the score Jacobian, it remains computationally tractable for high-dimensional neural networks. Using the Hutchinson trace estimator \citep{hutchinson1989stochastic}, we approximate the norm as $\mathcal{R}(\tau) = \mathbb{E}_{\bx, \mathbf{v}} [\|\nabla_{\bx} \bs_\tau(\bx) \mathbf{v}\|^2]$, where $\mathbf{v}$ is a random noise vector. The term $\nabla_{\bx} \bs_\tau(\bx) \mathbf{v}$ is evaluated as a Jacobian-vector product (JVP) in a single forward-mode pass, bypassing the need to explicitly compute or store the full Jacobian matrix (see Appendix~\ref{app:FI_estimation}).

\section{Analysis of Diffusability via Fisher Information}\label{sec:FI}

To formalize ``diffusability,'' we analyze the sensitivity of the global Minimum Mean Squared Error (MMSE) to $\tau$. 
Following the setting of \S\ref{sec:background}, where $k \in \{D, d\}$, 
the error in estimating the clean signal $\bx_0$ from its noisy observation $\bx$ at variance $\tau$ is defined as $\text{MMSE}^{(k)}(\mu_\tau) = \mathbb{E}[\|\bx_0 - \mathbb{E}[\bx_0|\bx]\|^2]$.

The I-MMSE identity \citep{Guo05_fisherL2}, proved in Appendix~\ref{sec:appendix_prove_4_5}, implies that the optimal denoising performance is fundamentally constrained by this information content:
\begin{equation}
\label{eq:MMSE_and_I}
\text{MMSE}^{(k)}(\mu_\tau) = \tau k - \tau^2\mathcal{I}^{(k)}(\mu_\tau).
\end{equation}

Differentiating \eqref{eq:MMSE_and_I} and applying the definition of FIR yields the decomposition:
\begin{equation}
\underbrace{\frac{d}{d\tau}\text{MMSE}^{(k)}(\mu_\tau)}_{\text{Denoising Resistance}} =
\underbrace{k - 2\tau \mathcal{I}^{(k)}(\mu_\tau)}_{\text{Intrinsic Noise Gain}} + \underbrace{\tau^2 \mathcal{R}^{(k)}(\mu_\tau).}_{\text{Geometric Complexity Penalty}}
\label{eq:denoising_resistance}
\end{equation}

This decomposition separates two contributions to population
denoising resistance: the \emph{Intrinsic Noise Gain}, which describes
the baseline dimensional and information-dependent contribution, and
the \emph{Geometric Complexity Penalty}, which is governed by the
variation of the population score.

A latent representation can alter either contribution. In
\S\ref{sec:isometry_fisher_preservation}, we study how deterministic
geometric encoders affect intrinsic FI and the Intrinsic Noise Gain.
In \S\ref{sec:FIR}, we derive sufficient local geometric conditions
for controlling the Geometric Complexity Penalty. We describe the essence of several theoretical
results in the main text, while their complete formal statements,
assumptions, and proofs are provided in the corresponding appendix
sections referenced below.

\subsection{Intrinsic Fisher Information under Geometric Encoders}
\label{sec:isometry_fisher_preservation}

We first study how intrinsic Fisher information changes when a measure
on the data manifold is pushed forward deterministically by an
encoder. The result isolates two effects: tangential metric distortion
and variation of the encoder's tangential volume Jacobian. This
intrinsic comparison is distinct from the ambient Gaussian-smoothed
data and latent channels considered later.

\paragraph{Intrinsic FI under bi-Lipschitz encoders.}
Proposition~\ref{prop:fisher_bi_lip}, stated and proved in
Appendix~\ref{sec:proof_fisher_stability}, considers smooth
$m$-dimensional manifolds $\mathcal M\subset\mathbb R^D$ and
$\mathcal N\subset\mathbb R^d$, together with a bi-Lipschitz
homeomorphism $E:\mathcal M\to\mathcal N$. Its lower and upper
intrinsic Lipschitz constants are denoted by $c$ and $C$.
Let
$\mathcal I_{\mathcal M}$ be the intrinsic FI of a measure $\nu$ on
$\mathcal M$, and let $\mathcal I_{\mathcal N}$ be that of its
pushforward $E_\#\nu$. Define the tangential Jacobian
\(
\mathcal J_E(x)=\sqrt{\det(DE(x)^\top DE(x))}
\),
its tangential gradient
\(
g_E=\nabla_{\mathcal M}\log\mathcal J_E
\),
and the Jacobian-variation term
\(
\widetilde\varepsilon^2
=
\int_{\mathcal M}\|g_E(x)\|^2\,d\nu(x).
\)

Under the weak regularity conditions given in the appendix, the
pushforward has finite intrinsic FI, and for every auxiliary parameter
$\theta\in(0,1)$,
\[
\frac{1-\theta}{C^2}\mathcal I_{\mathcal M}
-
\frac{(\theta^{-1}-1)\widetilde\varepsilon^2}{C^2}
\le
\mathcal I_{\mathcal N}
\le
\frac{1+\theta}{c^2}\mathcal I_{\mathcal M}
+
\frac{(\theta^{-1}+1)\widetilde\varepsilon^2}{c^2}.
\]
Thus, $c$ and $C$ control tangential metric distortion, while
$\widetilde\varepsilon$ controls variation of the tangential volume
Jacobian. In particular, an isometric encoder preserves intrinsic FI
exactly.

Proposition~\ref{prop:fisher_bi_lip} concerns intrinsic FI under
deterministic pushforward and does not directly compare the ambient
Gaussian-smoothed data and latent channels, since nonlinear encoding
generally does not commute with Gaussian smoothing.

For piecewise-smooth encoders, the proposition requires
$\log\mathcal J_E$ to have no jump across activation interfaces.
Thus, Leaky ReLU with slope parameter $\alpha>0$ is bi-Lipschitz but,
for $\alpha\ne1$, is not covered directly by the proposition.
Gaussian smoothing nevertheless gives finite FI for any latent
probability measure $\mu_Z:=E_\#\nu$, with
\(
\mathcal I^{(d)}(\mu_{Z,\tau})\le d/\tau
\)
for every $\tau>0$.
The full discussion of piecewise-smooth encoders and Gaussian regularization is provided in Appendix~\ref{sec:proof_fisher_stability}. Exact smoothed-density and FI calculations for a Gaussian toy model with coordinatewise ReLU and Leaky-ReLU encoders are given in Appendix~\ref{app:gaussian_activation_smoothed_FI}.

Corollary~\ref{cor:dim_compression} further isolates ambient
dimensional compression. On flat supports, preserving tangential FI
changes the Intrinsic Noise Gain by exactly $D-d$. On curved supports,
$D-d$ is instead the flat tangent-space baseline, with additional
curvature-dependent corrections.

\subsection{Fisher Information Rate and Second-Order Distortion}\label{sec:FIR}

The dominant factor in diffusion difficulty at higher noise levels is the \emph{Geometric Complexity Penalty} governed by the FIR, $\mathcal{R}^{(k)}(\mu_\tau)$. High-frequency manifold crinkling or structural irregularities inflate the score Jacobian $\nabla_{\bx} \bs_\tau$, causing denoising errors to compound quadratically as indicated by the factor $\tau^2$ in \eqref{eq:denoising_resistance}. Therefore, effective diffusion requires keeping this penalty low.

The FIR is mathematically governed by the magnitude of the score Jacobian, equivalently the Hessian of the log-density. At the level of deterministic intrinsic pushforwards, differentiating the score change-of-variables identity produces terms involving both the encoder's second derivatives and variation of its volume Jacobian; we record this calculation in Appendix~\ref{app:curvature_injection}.  This intrinsic calculation motivates the need for second-order encoder control, but it is distinct from the comparison of ambient Gaussian-smoothed laws. The latter comparison is established directly by Theorem~\ref{thm:FIR_stability_nonlinear} and Corollary~\ref{cor:combined_local_bound}. Nonsmooth activation interfaces fall outside the $C^2$ theorem and can generate large small-noise score curvature after Gaussian smoothing, as illustrated by the ReLU and Leaky-ReLU experiments. 
Because this injected curvature directly distorts the natural geometric complexity of the data, we need a metric to quantify it. To formalize this, we define the relative deviation between data and latent space FIRs.

\begin{definition}[FIR Deviation]
\label{def:fir_dis}

Let the intrinsic dimension of $\mu$ on $\mathbb R^D$ be $m$, with
$m\le d\le D$, and let
\(
    \mu_Z:=E_\#\mu
\)
be its pushforward on $\mathbb R^d$. Let $\mu_\tau$ and
$\mu_{Z,\tau}$ denote the corresponding Gaussian-smoothed measures.

The FIR deviation is
\[
    \mathcal D_{\mathcal R}(\tau)
    :=
    \left|
        1-
        \frac{
            \mathcal R^{(D)}(\mu_\tau)
        }{
            \mathcal R^{(d)}(\mu_{Z,\tau})
        }
    \right|.
\]

When $m<d\le D$, we additionally define the dimension-normalized FIR
deviation by
\[
    \mathcal D_{\mathcal R}^{\mathrm{sc}}(\tau)
    :=
    \left|
        1-
        \left(
            \frac{d-m}{D-m}
        \right)
        \frac{
            \mathcal R^{(D)}(\mu_\tau)
        }{
            \mathcal R^{(d)}(\mu_{Z,\tau})
        }
    \right|.
\]
\end{definition}

The unscaled deviation compares the full ambient FIRs and remains
meaningful in the equal-dimensional case $m=d=D$. On flat
$m$-dimensional supports with $m<d\le D$,
Lemma~\ref{lem:normal_term} shows that the factor
\(
    {(d-m)}/{(D-m)}
\)
compensates for the universal normal-direction FIR
contribution. On curved supports, it should instead be interpreted as
a leading codimension normalization: curvature-dependent corrections
remain and are quantified locally by
Corollary~\ref{cor:combined_local_bound}.

\subsection{Stability Theorems on Manifolds of Positive Reach}
To analyze FIR stability, we first establish a flat, linear baseline, and then
extend the comparison locally to nonlinear
encoders and curved manifolds.

We quantify the curvature of $\mathcal M$ by its reach
\(
\rho_{\mathcal M}:=\reach(\mathcal M),
\)
defined as the supremum of all radii $r>0$ such that every point
within distance $r$ of $\mathcal M$ has a unique nearest point on
$\mathcal M$.

\begin{assumption}[\textbf{Flat Support and Linear Mapping}]
\label{ass:flat_linear_encoder}
Let $M\subset\mathbb{R}^D$ be an $m$-dimensional linear subspace
with an isometry $\bU:\mathbb{R}^m\to M$. 
Let $\rho$ be a probability measure on $\mathbb{R}^m$ satisfying
$\operatorname{supp}(\rho) \subseteq B_m(\bzero,R_\rho)$ 
for some $R_\rho<\infty$. Let
$\mu := (\bU_\#\rho)\otimes \delta_{\bzero}$ 
be the corresponding probability measure on $\mathbb{R}^D$, where
$\delta_{\bzero}$ is the Dirac measure at $\bzero$ in the orthogonal
complement of $M$. 
Let $m\le d\le D$, and fix an isometry
$\bV:\mathbb{R}^m\to\mathbb{R}^d$. Assume the encoder
$E:\mathbb{R}^D\to\mathbb{R}^d$ is linear on $M$ and satisfies
$E(\bU\bx)=\bV(\bA\bx)$ for an invertible matrix $\bA\in\mathbb{R}^{m\times m}$.
\end{assumption}
\begin{theorem}[\textbf{Linear Stability}]
\label{thm:FIR_stability_d_less_D_linear}
Under Assumption~\ref{ass:flat_linear_encoder}, let
$\mu_Z:=E_\#\mu$. Assume that $\bA$ satisfies the near-isometry
condition $\|\bA^\top\bA-\bI_m\|_2\le\delta$ for some $\delta\in(0,1)$. 
Then there exists a constant $C_m>0$, depending only on $m$, such
that, for every $\tau\ge R_\rho^2$, 
the following statements hold:
\begin{itemize}[leftmargin=*,nosep]
    \item If $m=d=D$, then $\mathcal{D}_{\mathcal{R}}(\tau)
        \le C_m\delta.$
    \item If $m<d\le D$, then $\mathcal{D}_{\mathcal{R}}^{\mathrm{sc}}(\tau)
        \le
        C_m\left(
            \frac{D-d}{D-m}
            +(d-m)\delta
        \right).$
\end{itemize}
\end{theorem}

The condition $\tau\ge R_\rho^2$ uniformly controls the clean tangential signal and its near-isometric perturbation, allowing Lemma~\ref{lem:abstract_FIR_stability} to yield a constant depending only on $m$. Compact support is sufficient but not necessary for linear FIR stability. The unbounded Gaussian distributions used in the toy experiments are handled separately using the exact formulas in Appendix~\ref{app:plot_gaussian_toy}.




\paragraph{Local nonlinear and curved stability of FIR.}
To extend the linear analysis to curved manifolds and nonlinear encoders, we localize the data distribution to
\(\mathcal M\cap B(x_0,r\sqrt{\tau}),\)
the neighborhood visible at diffusion scale $\sqrt{\tau}$. A canonical construction of this localized distribution is given in Definition~\ref{def:truncated_local_prior} in Appendix~\ref{app:proof_thm_3_6}. The result below applies more generally to any probability measure supported on this neighborhood.

Under this localization, Theorem~\ref{thm:FIR_stability_nonlinear}, stated and proved in Appendix~\ref{app:proof_thm_3_6}, considers a $C^2$ encoder $E:\mathbb R^D\to\mathbb R^d$ on an $m$-dimensional $C^2$ manifold $\mathcal M\subset\mathbb R^D$ with positive reach $\rho_{\mathcal M}$.
Define 
\(A_{x_0}:=DE(x_0)|_{T_{x_0}\mathcal M}\)
and the tangent-affine approximation
\(F_{x_0}(x):=E(x_0)+A_{x_0}P_{T_{x_0}\mathcal M}(x-x_0).\)
The local tangential distortion is
\(\delta_{x_0}:=\|A_{x_0}^{\top}A_{x_0}-I_m\|_{\mathrm{op}}
\le\delta_*<1.\)

Let $\lambda_{x_0,\tau}^{(r)}$ denote the localized measure, and define
\(P_{x_0,\tau}:=(E_\#\lambda_{x_0,\tau}^{(r)})*\gamma_\tau\),
\(Q_{x_0,\tau}:=((F_{x_0})_\#\lambda_{x_0,\tau}^{(r)})*\gamma_\tau\),
and the local affine residual
\(\epsilon_{x_0,\tau}
:=
\sup_{x\in\operatorname{supp}\lambda_{x_0,\tau}^{(r)}}
\|E(x)-F_{x_0}(x)\|\). If
\(\epsilon_{x_0,\tau}\le a_0\sqrt{\tau},\)
then
\(|1-\mathcal R(P_{x_0,\tau})/\mathcal R(Q_{x_0,\tau})|
\le (C/b_{x_0})\epsilon_{x_0,\tau}/\sqrt{\tau},\)
where $C=C(d,r,\delta_*,a_0)$, and
\(b_{x_0}:=\max\{d-m,d/[1+r^2(1+\delta_*)]^2\}.\)

The theorem further relates the $\epsilon_{x_0,\tau}$ to the geometry of the encoder and data manifold. When the encoder differential and second derivative are locally bounded by $J$ and $H$, respectively, $\epsilon_{x_0,\tau}$ is of order $r^2\tau(H+J/\rho_{\mathcal M})$. 
Consequently, whenever
\(\frac{r^2\sqrt{\tau}}{2}
\left(H+\frac{J}{\rho_{\mathcal M}}\right)
\le a_0,\)
$$
\left|
    1-
    \frac{\mathcal R(P_{x_0,\tau})}
         {\mathcal R(Q_{x_0,\tau})}
\right|
\le
\frac{Cr^2}{2b_{x_0}}
\sqrt{\tau}
\left(
    H+\frac{J}{\rho_{\mathcal M}}
\right).
$$

\paragraph{Combined local data-versus-latent bound.}
Corollary~\ref{cor:combined_local_bound}, stated and proved in Appendix~\ref{app:proof_thm_3_6}, combines encoder nonlinearity and manifold curvature to compare localized data-space and latent-space FIRs directly. Let \(P^X_{x_0,\tau}\) denote the Gaussian-smoothed law of the centered local data \(x-x_0\) in \(\mathbb R^D\), and let \(P^Z_{x_0,\tau}\) denote the Gaussian-smoothed law of the centered encoded data \(E(x)-E(x_0)\) in \(\mathbb R^d\). Under the local near-isometry condition \(\delta_{x_0}\le\delta_*<1\) and the diffusion-scale assumptions detailed in the appendix, there exist positive constants \(C_A\), \(C_E\), and \(C_M\) such that
\[
\left|
    \tau^2
    \left[
        \mathcal R(P^X_{x_0,\tau})
        -
        \mathcal R(P^Z_{x_0,\tau})
    \right]
    -(D-d)
\right|
\le
C_A\delta_{x_0}
+
C_Er^2\sqrt{\tau}\,H
+
C_Mr^2\sqrt{\tau}
\frac{1+J}{\rho_{\mathcal M}}.
\]

Thus, \(D-d\) is the flat normal-direction compression baseline, while the three terms on the right quantify tangential distortion, encoder nonlinearity, and manifold curvature, respectively.

\textbf{Geometric Interpretation of FIR Deviation.}
Corollary~\ref{cor:combined_local_bound} combines the comparison of
the curved manifold with its tangent plane, the near-isometric linear
encoding of the tangent distribution, and the nonlinear encoder with
its tangent-affine approximation. Together with
Theorem~\ref{thm:FIR_stability_d_less_D_linear}, this identifies the
following four geometric penalties.

\textbf{1. Encoder-Induced Distortions:}
Our bounds identify local near-isometry and controlled encoder
curvature as sufficient controls on FIR distortion. These controls are
represented by two terms:

\begin{itemize}[leftmargin=*,nosep]
    \item \emph{Tangential Distortion ($\delta$):} Local metric
    stretching or compression
    (Figure~\ref{fig:FIR_flat_encoder_geometry}(b)).

    \item \emph{High-Frequency Penalty
    ($\epsilon/\sqrt{\tau}$):} This quantity compares the encoder
    displacement with the diffusion length scale. A fixed-amplitude distortion becomes more visible as $\tau$ decreases. For a $C^2$ encoder
    on an $O(\sqrt{\tau})$ neighborhood, the Taylor residual is
    $O(\tau)$, producing a local FIR penalty of order
    $O(\sqrt{\tau})$. Persistent low-noise inflation may therefore
    indicate nonsmooth behavior, increasing fine-scale curvature,
    global folding, or interactions between multiple local branches.
\end{itemize}

\textbf{2. Intrinsic Data Penalties:}
The geometry of the data introduces its own penalties. 
\begin{itemize}[leftmargin=*,nosep]
    \item \textit{Dimension Mismatch ($D-d$):}
    The exact normal-direction contribution of dimensional
    compression on flat supports and the leading local compression
    contribution on curved supports.

    \item \emph{Curvature-Dependent Penalty
    ($O(r^2\sqrt{\tau}(1+J)/\rho_{\mathcal M})$):}
    The deviation of the manifold from its tangent plane produces a
    term proportional to
    \(r^2\sqrt{\tau}/\rho_{\mathcal M}\). The corresponding latent-space error is additionally amplified
    by the local encoder norm $J$. The resulting term
    combines intrinsic manifold curvature with the encoder's sensitivity
    to normal displacement.

\end{itemize}

\section{Experiments}\label{sec:experiments}

We compare FI, FIR, and FIR deviations for pixel- and latent-space
diffusion models. We begin with Gaussian toy data and simple encoders to illustrate the geometric mechanisms studied in \S\ref{sec:FI} and to compare with exact population formulas where available. We then move to real image data and more complex autoencoders. Implementation details are provided in Appendix~\ref{app:exp_detail},
and additional FFHQ experiments with VAE, KL-AE, and NVAE encodings
are reported in Appendix~\ref{app:FFHQ_exp}.

\subsection{FI and FIR Comparison in the Gaussian Toy Setting}
\label{sec:gaussian_toy}


\textbf{Experimental Setup.} We consider a 2D Gaussian dataset and its pointwise nonlinear encodings. Specifically, we generate training data \(\bx \sim \mathcal{N}(\bzero, \bI_2)\) and construct encoded samples \(\bz= E(\bx)\), where the encoder \(E\) is applied pointwise to each coordinate of \(\bx\). We then train a tiny diffusion model\footnote{We use the open-source implementation available at \url{https://github.com/tanelp/tiny-diffusion}.} with identical architecture and training settings separately on the original data \(\bx\) and the encoded data \(\bz\). 
The numerical estimation details are presented in
Appendix~\ref{app:FI_estimation}.

\begin{figure}[t!bp]
\centering
\captionsetup[subfigure]{skip=0pt}

\begin{subfigure}{0.46\linewidth}
    \includegraphics[width=\linewidth]{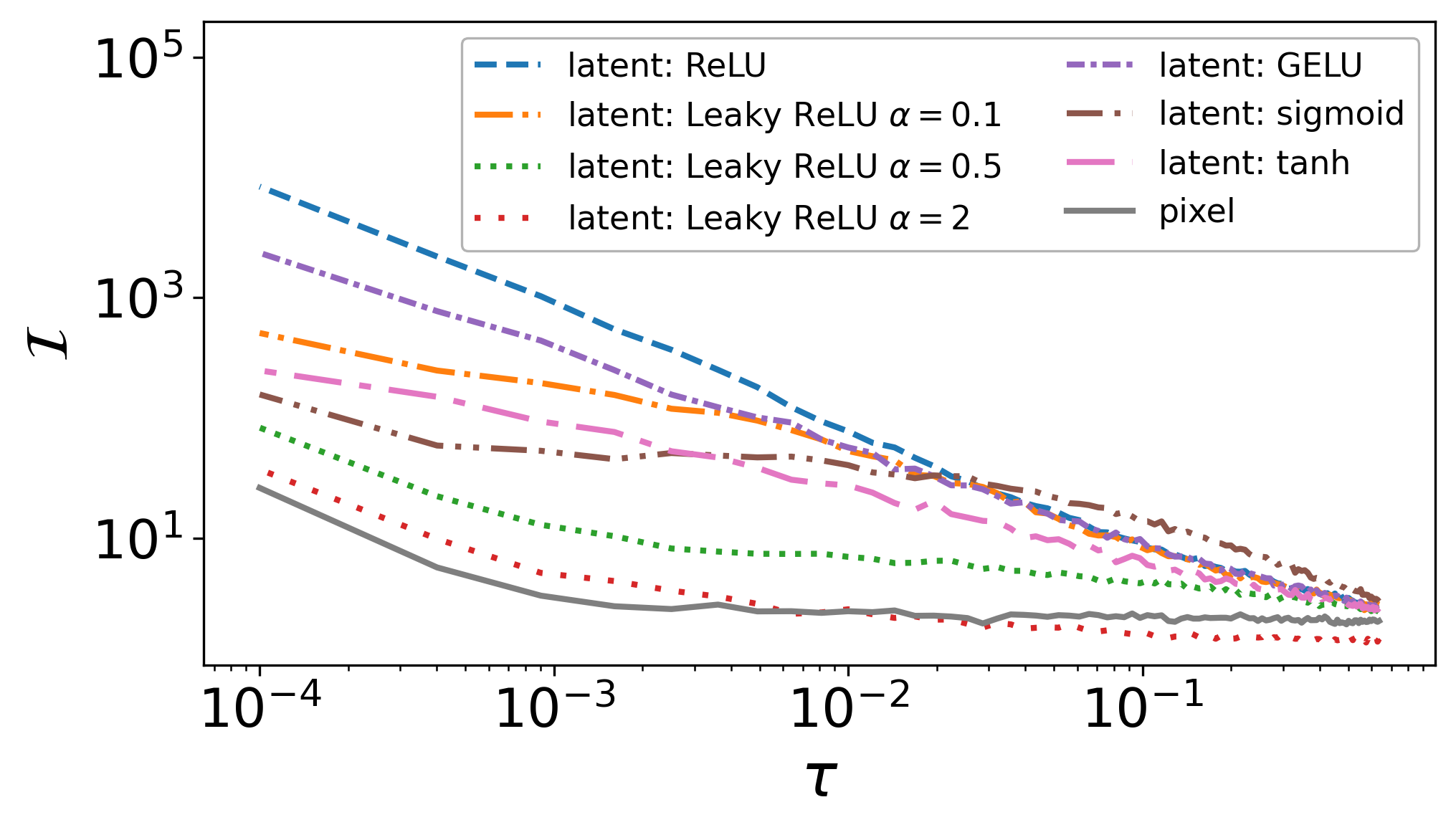}
    \caption{}
\end{subfigure}
\begin{subfigure}{0.46\linewidth}
    \includegraphics[width=\linewidth]{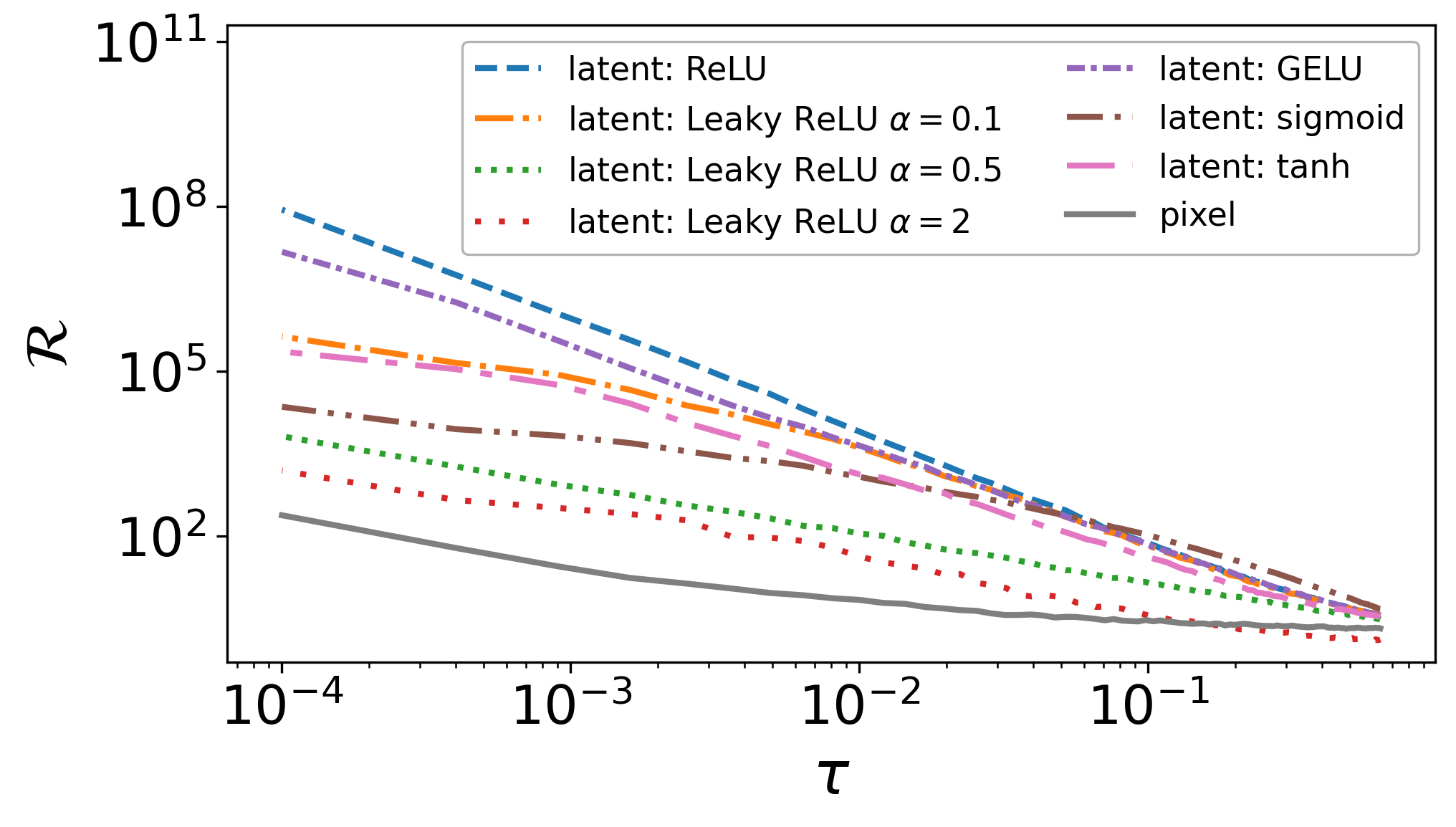}
    \caption{}
\end{subfigure}


\begin{subfigure}{0.46\linewidth}
    \includegraphics[width=\linewidth]{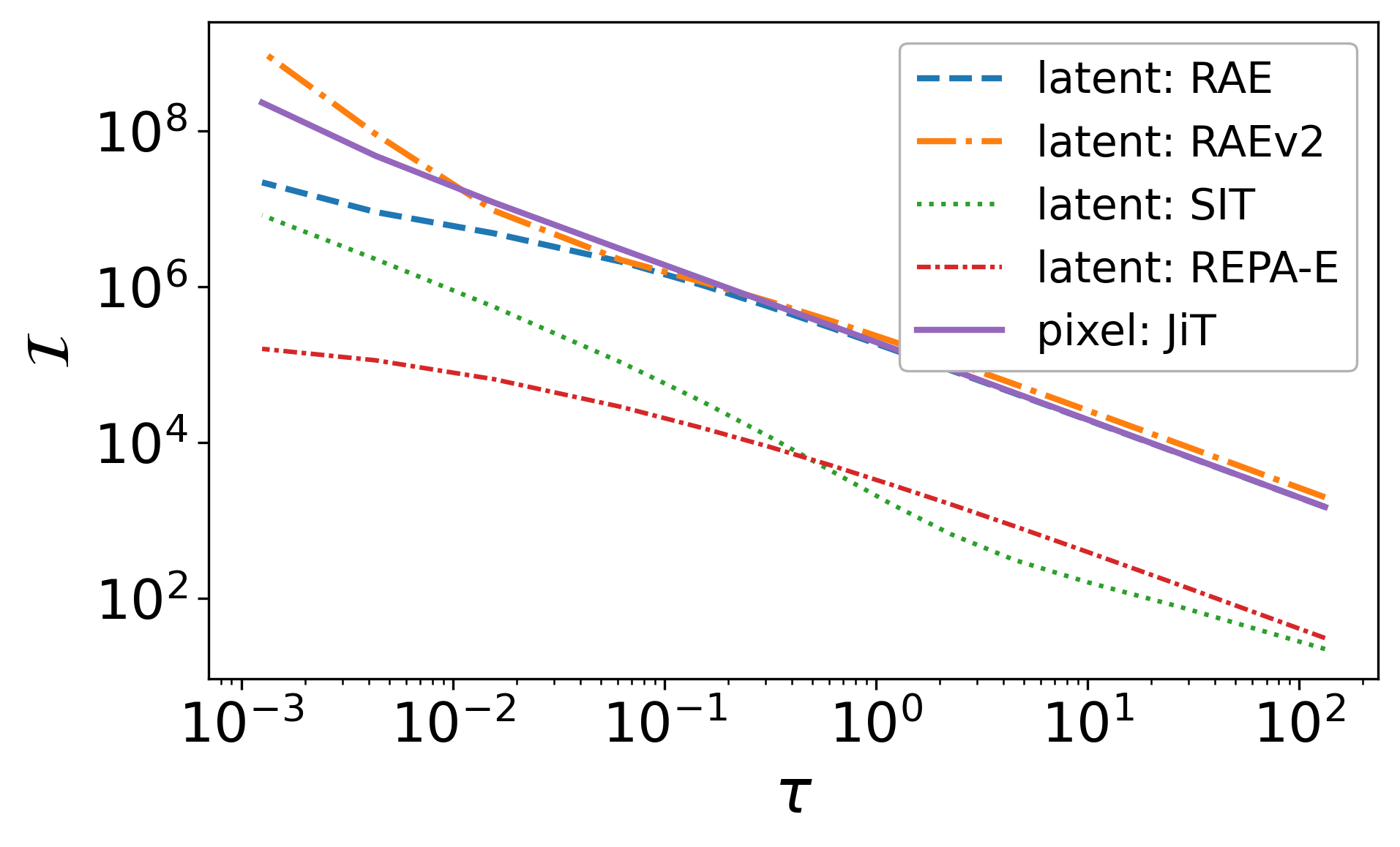}
    \caption{}
\end{subfigure}
\begin{subfigure}{0.46\linewidth}
    \includegraphics[width=\linewidth]{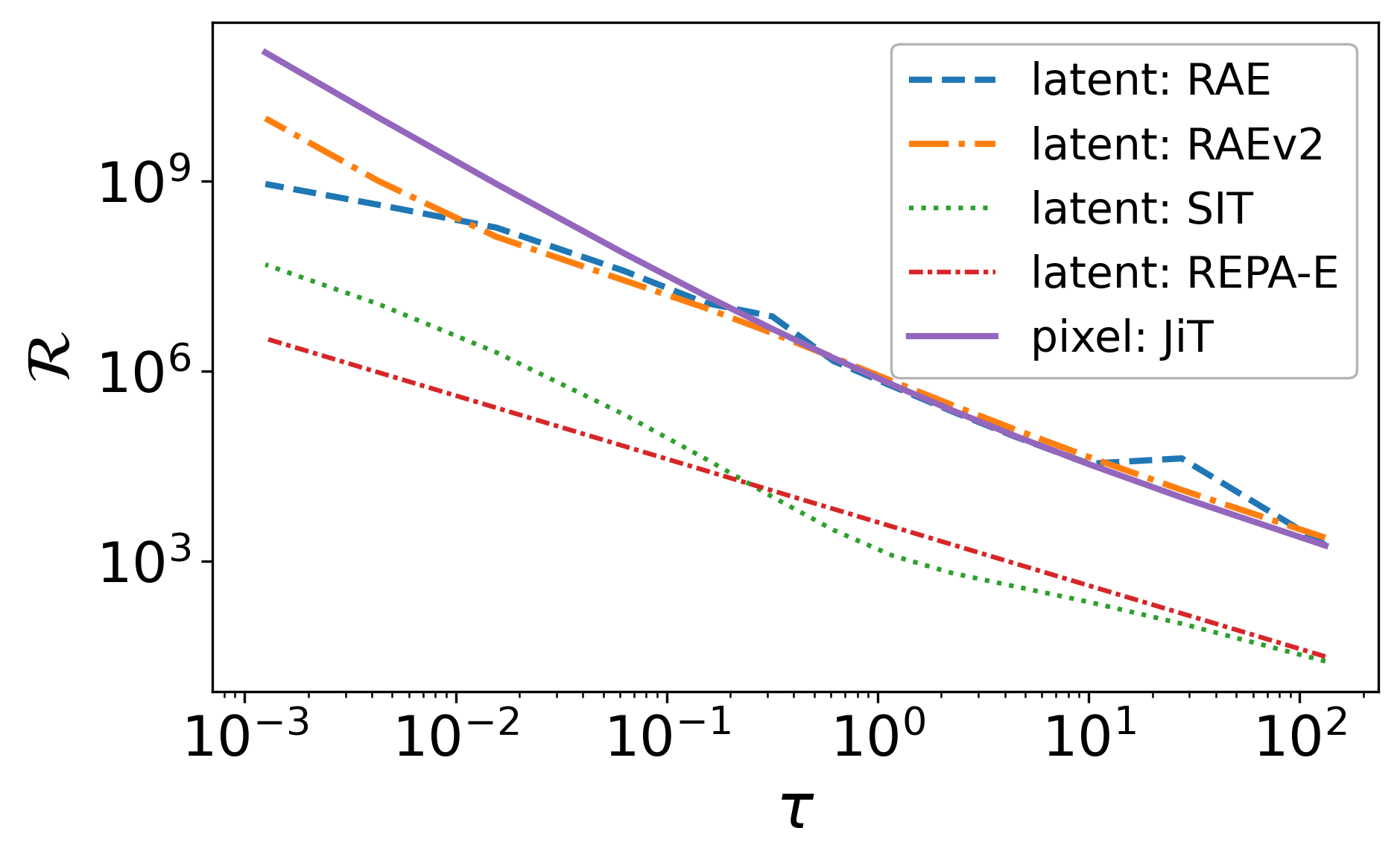}
    \caption{}
\end{subfigure}
\caption{
Values of $\mathcal I$ (left) and $\mathcal R$ (right) plotted against
the noise variance $\tau$.
(a,b) Gaussian toy experiments using tiny diffusion models.
\emph{Pixel} denotes models trained on
$\bx\sim\mathcal N(\bzero,\bI_2)$, while \emph{latent} denotes models
trained on $\bz=E(\bx)$, with the pointwise activation $E$ indicated
in the legend.
(c,d) Results for recent image generative models. JiT operates
directly in pixel space, whereas SiT, RAE, RAEv2, and REPA-E operate
in latent space. For (c,d), we show
$\sqrt{\tau}\in[0.01,80]$, excluding smaller noise levels because of
numerical instability.
}
\label{fig:combined_FI_FIR}
\label{fig:gaussian_activation}
\label{app:fig:RAE_FI}
\end{figure}


\textbf{Results.}
Figure~\ref{fig:gaussian_activation}(a,b) compares FI
($\mathcal I$) and FIR ($\mathcal R$) computed using diffusion models
trained on the original data $\bx$ and encoded data $\bz=E(\bx)$. We
consider encoders given by common activation functions, including
ReLU, Leaky ReLU with negative slope $\alpha$, GELU, sigmoid, and
tanh. An expanded comparison of ReLU and Leaky ReLU over additional
values of $\alpha$ is provided in
Figure~\ref{fig:gaussian_relu_activation} in
Appendix~\ref{app:plot_gaussian_toy}.
Appendix~\ref{app:gaussian_activation_smoothed_FI} also gives the
exact Gaussian-smoothed population density and a model-free FI formula
for ReLU and Leaky ReLU. More generally,
Lemma~\ref{lem:gaussian_smoothing_finite_FI} guarantees that every
two-dimensional Gaussian-smoothed latent law satisfies
\(
\mathcal I^{(2)}(\mu_{Z,\tau})\le 2/\tau.
\)

Figure~\ref{fig:gaussian_activation}(a) compares FI across encoders
against the pixel baseline. ReLU exhibits the largest observed
deviation. Leaky ReLU is globally bi-Lipschitz for every $\alpha>0$,
and its metric constants approach those of an isometry as
$\alpha\to1$. Correspondingly, its empirical FI curve approaches the
pixel baseline, as further illustrated by the expanded comparison in
Figure~\ref{fig:gaussian_relu_activation}. However, the exact
Leaky-ReLU map has an interface jump in its Jacobian determinant, so
Proposition~\ref{prop:fisher_bi_lip} does not directly predict the
Gaussian-smoothed FI curves. The exact expression
\eqref{eq:activation_exact_smoothed_FI} provides the appropriate
population benchmark for this experiment. Sigmoid and tanh do not possess a positive global lower Lipschitz
constant on the support of a Gaussian input, and GELU is not globally
injective. We therefore interpret their FI curves empirically rather
than as direct applications of
Proposition~\ref{prop:fisher_bi_lip}.

Figure~\ref{fig:gaussian_activation}(b) compares FIR across encoders.
ReLU again exhibits the largest observed deviation. ReLU and Leaky
ReLU are piecewise linear, with distributional second derivatives
concentrated at the activation axes, and therefore fall outside the
$C^2$ assumptions of
Theorem~\ref{thm:FIR_stability_nonlinear}. For Leaky ReLU, the
strength of the interface singularity is proportional to
$|1-\alpha|$, which decreases as $\alpha$ approaches $1$, consistent
with the observed FIR trend.

\subsection{FIR Deviation in the Gaussian Toy Setting}
\label{sec:gaussian_deviation}


\textbf{Experimental Setup.} 
Consider an intrinsic manifold of dimension $m=2$ with data $\by \sim \mathcal{N}(\bzero, \bI_2)$. We embed $\by$ into $\bx = (y_1, y_2, 0, \ldots, 0) \in \mathbb{R}^D$, and construct latent representations $\bz= E(\bx) \in \mathbb{R}^d$, where $D,d \ge 2$. Using the same tiny diffusion model, training protocol, and FIR estimation procedure described in \S\ref{sec:gaussian_toy}, we compute \( \mathcal{R}^{(D)}(\mu_\tau)\) and \( \mathcal{R}^{(d)}((\mu_Z)_\tau) \) using diffusion models trained on \(\bx\) and \(\bz\), respectively, and use these values to compute the FIR deviations \(\mathcal{D}_{\mathcal{R}}(\tau)\) and $\mathcal{D}_{\mathcal R}^{\mathrm{sc}}(\tau)$.

\begin{figure}[t!bp]
\centering
\captionsetup[subfigure]{skip=0pt}

\begin{subfigure}[t]{0.32\linewidth}
    \centering
    \includegraphics[width=\linewidth]{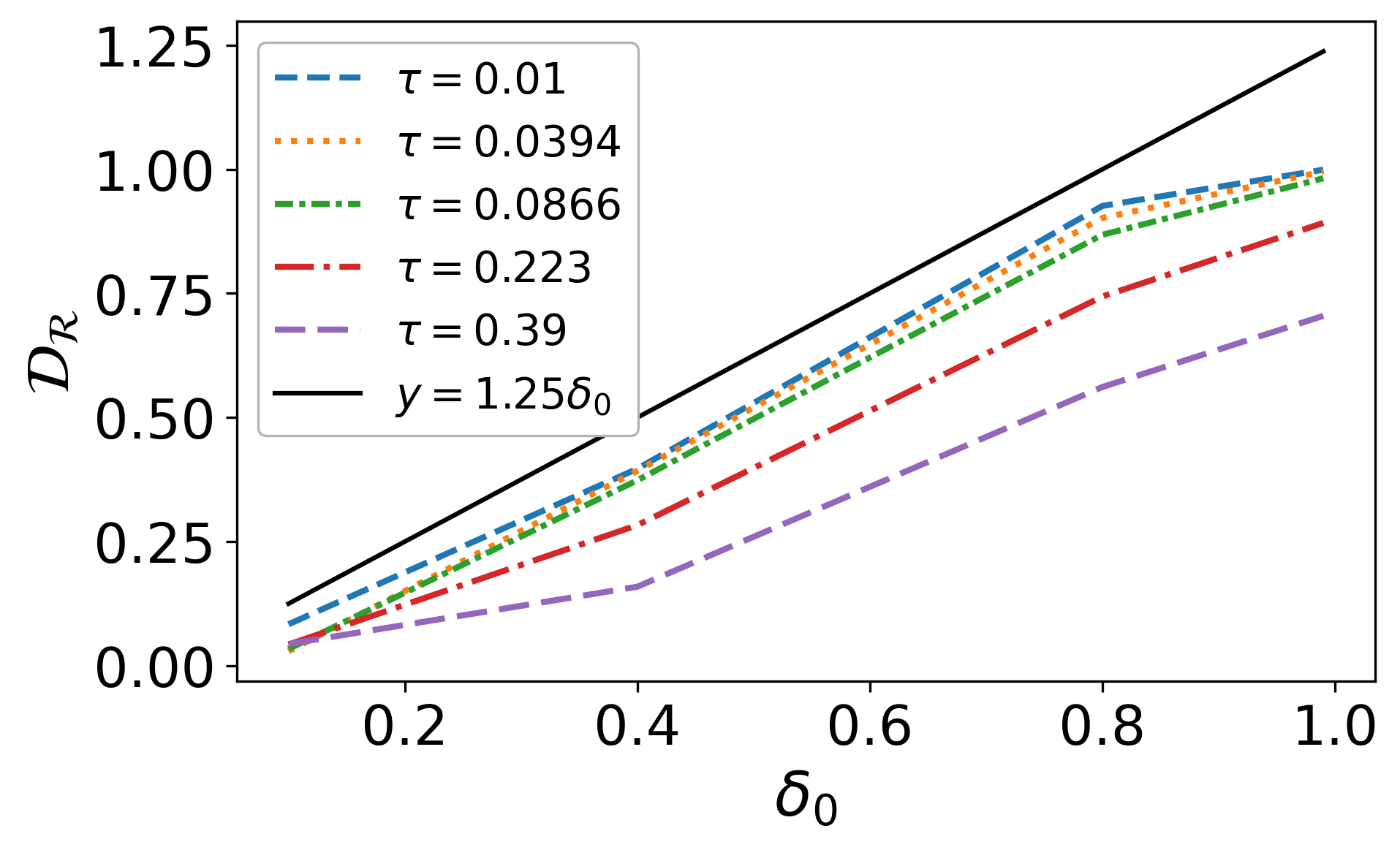}
    \caption{}
\end{subfigure}
\hfill
\begin{subfigure}[t]{0.32\linewidth}
    \centering
    \includegraphics[width=\linewidth]{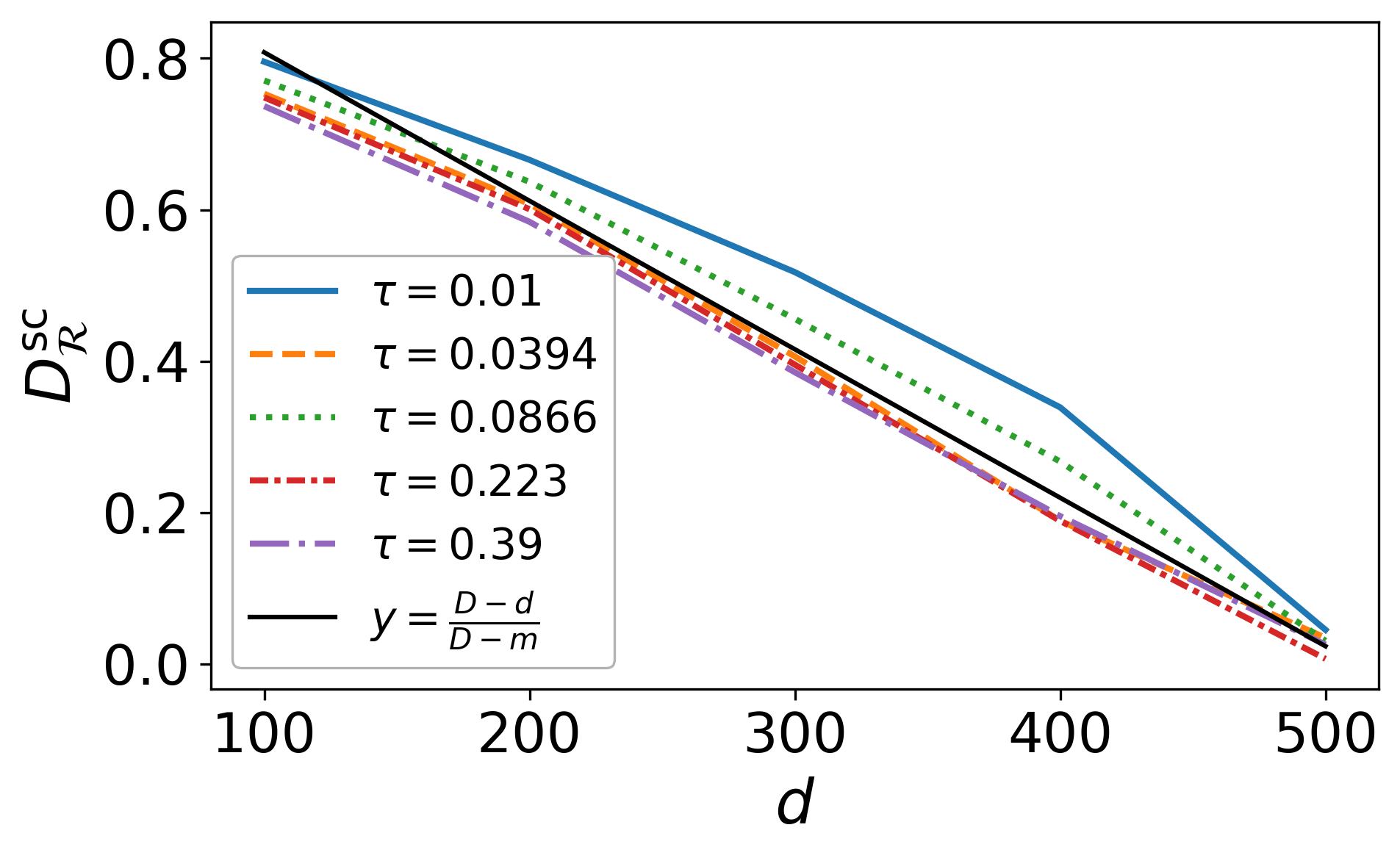}
    \caption{}
\end{subfigure}
\hfill
\begin{subfigure}[t]{0.32\linewidth}
    \centering
    \includegraphics[width=\linewidth]{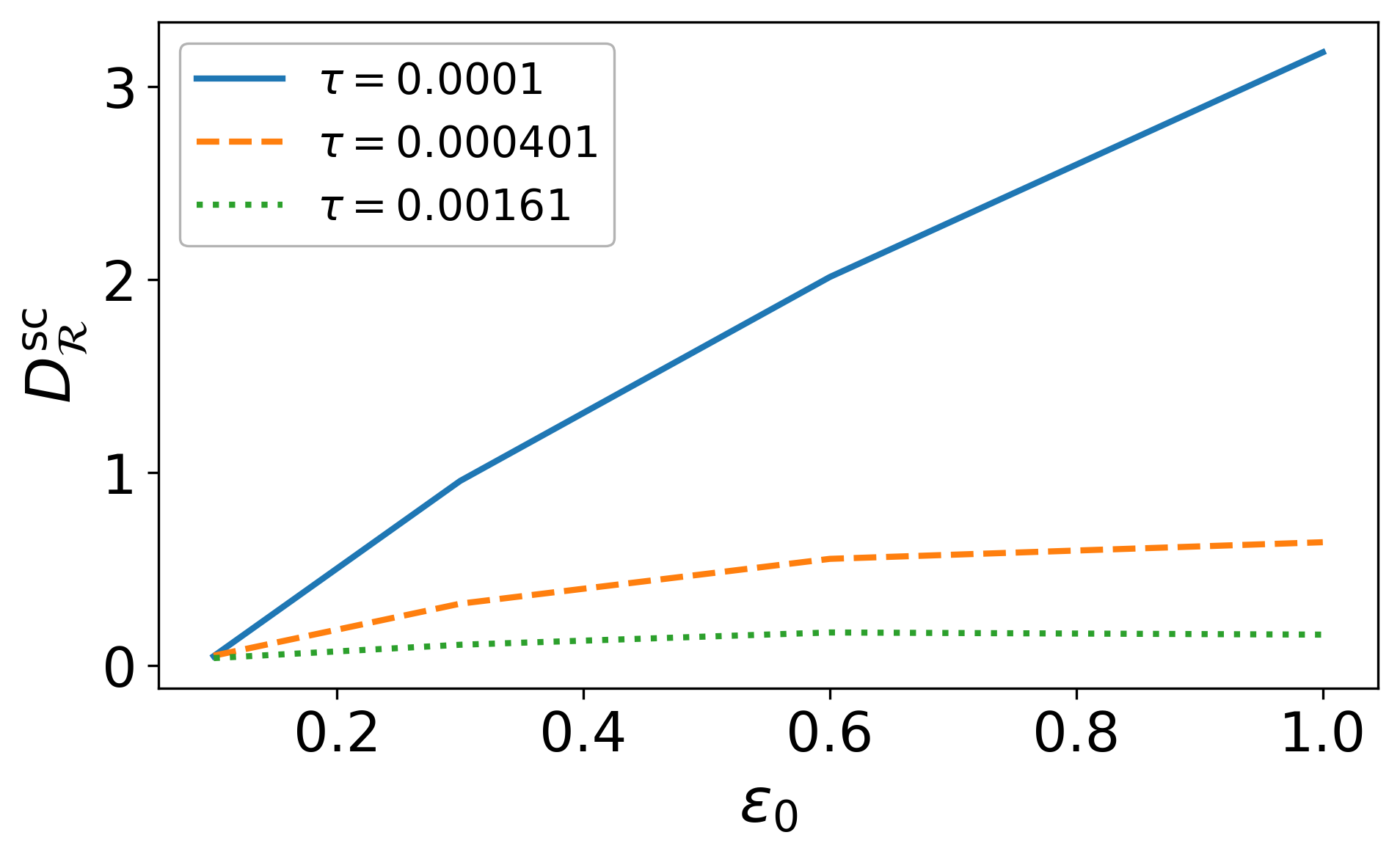}
    \caption{}
\end{subfigure}

\caption{
FIR deviation $\mathcal{D}_{\mathcal R}$ vs.\ (a) $\delta_0$, 
(b) $d$, and (c) $\varepsilon_0$ in toy settings.
Data $\by\sim\mathcal N(\bzero,\bI_2)$ are embedded as
$\bx=(y_1,y_2,0,\ldots,0)\in\mathbb R^D$ and encoded as
$\bz=E(\bx)\in\mathbb R^d$.
We compute $\mathcal{D}_{\mathcal R}$ from
$\mathcal{R}^{(D)}(\mu_\tau)$ and
$\mathcal{R}^{(d)}((\mu_Z)_\tau)$ using diffusion models trained on
$\bx$ and $\bz$.
Curves correspond to fixed noise variance $\tau$.
The solid lines $y=1.25\delta_0$ in (a) and
$y=(D-d)/(D-2)$ in (b) serve as linear references.
In (a), $D=d=2$ and
$E(\bx)=\bA\bx$, where
$\bA=\mathrm{diag}(\sqrt{1+\delta_0},\sqrt{1-\delta_0})$.
In (b), $D=512$ and
$\bz=(y_1,y_2,0,\ldots,0)\in\mathbb R^d$.
In (c), $D=d=3$ and
$E((x_1,x_2,0),\varepsilon_0)
=
\bigl(
\sin(\varepsilon_0x_1)/\varepsilon_0,\,
x_2,\,
(1-\cos(\varepsilon_0x_1))/\varepsilon_0
\bigr)$.
}
\label{fig:gaussian_FIR_deviation}
\end{figure}

\textbf{Results.}
For all experiments in this section, exact Gaussian calculations for
the two linear settings and additional plots of
$\mathcal{D}_{\mathcal R}$ versus the noise variance $\tau$ are
provided in Appendix~\ref{app:plot_gaussian_toy}. The geometric
parameters for the nonlinear setting are also discussed there.

We begin with the case $D=d=m=2$. The encoded data are obtained via
the linear transformation $\bz=\bA\bx$, where
\(
    \bA
    =
    \operatorname{diag}
    \left(
        \sqrt{1+\delta_0},
        \sqrt{1-\delta_0}
    \right).
\)
In this setting,
\(
    \|\bA^\top\bA-\bI_2\|_2=\delta_0.
\)
Figure~\ref{fig:gaussian_FIR_deviation}(a) shows
$\mathcal D_{\mathcal R}$ as a function of $\delta_0$ for fixed
noise variances $\tau$. We report curves at relatively small noise variances
$\tau$, since for large $\tau$ the FIR values become very small due
to heavy smoothing, making the relative deviation numerically
unstable and less informative. The empirical curves increase monotonically
and appear approximately linear over much of the tested range.
The exact population calculation in
\eqref{eq:exact_gaussian_delta_deviation} shows, however, that the
deviation is quadratic near $\delta_0=0$ and satisfies
\(
    \mathcal D_{\mathcal R}(\tau)
    \le
    3\delta_0^2
    \le
    3\delta_0.
\)
Thus, the experiment is consistent with an
$\mathcal O(\delta_0)$ stability bound, although its exact
small-$\delta_0$ behavior is sharper.

Next, we consider $D=512$ and varying $d$. We take
\(
    \bz=(y_1,y_2,0,\ldots,0)\in\mathbb{R}^d.
\)
Here, the encoder acts as the identity on the intrinsic subspace, so
$\bA=\bI_m$. Equation~\eqref{eq:exact_gaussian_dimension_deviation}
gives
\(
    \mathcal{D}_{\mathcal R}^{\mathrm{sc}}(\tau)
    \le
    \frac{D-d}{D-m}.
\)
Figure~\ref{fig:gaussian_FIR_deviation}(b) reports
$\mathcal{D}_{\mathcal R}^{\mathrm{sc}}$ as a function of $d$.
The observed downward trend is consistent with this exact upper
bound. Finally, we consider \(D = d = 3\) and \(\bz= E(\bx,\varepsilon_0)\), where $E(\bx, \varepsilon_0)$ is a smooth nonlinear map that locally bends the 2D data manifold into a 3D cylinder of radius $1/\varepsilon_0$, defined as
$E((x_1,x_2,0),\varepsilon_0)
=
\bigl(
\sin(\varepsilon_0 x_1)/\varepsilon_0,\;
x_2,\;
(1-\cos(\varepsilon_0 x_1))/ \varepsilon_0
\bigr)$ with $\varepsilon_0>0$. 
For this construction, the tangential Jacobian is an isometry, so
the local near-isometry parameter satisfies $\delta_{x_0}=0$.
Moreover,
\(
    \|D^2E\|_{\mathrm{op}}=\varepsilon_0,
\)
while the input plane has infinite reach. Therefore, for the
$r\sqrt{\tau}$-localized model of
Theorem~\ref{thm:FIR_stability_nonlinear}, the local Taylor residual
satisfies
\(
    \epsilon_{x_0,\tau}
    \le
    \frac{r^2\tau}{2}\varepsilon_0.
\)
Consequently, the theorem gives the local nonlinear-versus-affine
bound
\(
    \left|
        1-
        \frac{
            \mathcal R(P_{x_0,\tau})
        }{
            \mathcal R(Q_{x_0,\tau})
        }
    \right|
    \le
    C r^2\varepsilon_0\sqrt{\tau}.
\)

Figure~\ref{fig:gaussian_FIR_deviation}(c) evaluates the global FIR
deviation for an unbounded Gaussian input and therefore is not a
direct numerical verification of this localized theorem. At each
fixed noise level, however, the observed approximately linear
dependence on $\varepsilon_0$ is qualitatively consistent with the
theorem's linear dependence on the encoder-curvature magnitude. The linear trend is most evident for small \(\tau\). For larger \(\tau\), the deviation values are much smaller, which may make the FIR estimation error more prominent and thus visually weaken the linear pattern.

\subsection{FI and FIR Comparison across ImageNet Latent Diffusion Models}\label{sec:exp_RAE}

\paragraph{Experimental Setup.} We evaluate five recent image-generation models on ImageNet at $256\times256$ resolution: RAE~\citep{zheng2025diffusion}, RAEv2~\citep{singh2026improved}, SiT~\citep{ma2024sit}, REPA-E~\citep{leng2025repa}, and JiT~\citep{li2026back}. JiT performs generation directly in pixel space, whereas the other methods operate on latent representations produced by different image encoders. These models are selected because they use closely related transformer-based generative architectures of comparable scale, allowing us to examine how the choice of representation affects FI and FIR. We use the official implementations and publicly available pretrained checkpoints.


\paragraph{Results.} 

\begin{table}[t]
\centering
\caption{
Generation, geometric, and spectral evaluation results. Generation
quality is evaluated using FID, spatial FID (sFID), Inception Score
(IS), precision, and recall. Straightness and efficiency measure the
geometry of ODE sampling trajectories, while SEC measures the fraction
of latent spectral energy outside the low-frequency region, with
subscripts indicating the frequency threshold. Geometric and spectral
metrics are reported only for the latent diffusion models.
}
\resizebox{\linewidth}{!}{
\begin{tabular}{lccccc|cccc}
\toprule
& \multicolumn{5}{c|}{Generation Metrics}
& \multicolumn{4}{c}{Geometric and Spectral Metrics} \\
\cmidrule(lr){2-6}
\cmidrule(lr){7-10}
Model
& FID$\downarrow$
& sFID$\downarrow$
& IS$\uparrow$
& Precision$\uparrow$
& Recall$\uparrow$
& Straightness$\uparrow$
& Efficiency$\uparrow$
& SEC$_{0.25}\downarrow$
& SEC$_{0.5}\downarrow$ \\
\midrule
RAE
& 2.432 & 8.340 & 243.794 & 0.720 & 0.640
& 0.999 & 0.982 & 0.277 & 0.112 \\
RAEv2
& 2.370 & 7.458 & 225.068 & 0.730 & 0.633
& 0.999 & 0.949 & 0.309 & 0.137 \\
SiT
& 7.945 & 9.442 & 123.962 & 0.594 & 0.678
& 1.000 & 0.851 & 0.411 & 0.209 \\
REPA-E
& 3.496 & 7.973 & 159.359 & 0.692 & 0.630
& 1.000 & 0.875 & 0.312 & 0.102 \\
JiT
& 3.169 & 8.449 & 65.185 & 0.692 & 0.631
& -- & -- & -- & -- \\
\bottomrule
\end{tabular}
}
\label{tab:RAE_metrics}
\end{table}

Figure~\ref{app:fig:RAE_FI}(c,d) compares the FI and FIR curves across the evaluated models. 
Among the latent diffusion models, RAE and RAEv2 are substantially closer to the pixel-space JiT baseline in both FI and FIR. 
In contrast, SiT and REPA-E show consistently lower FI/FIR values and larger deviations from the pixel-space reference. 
This suggests that the RAE-based latent spaces retain more of the diffusion-relevant structure of the original image space, leading to latent distributions whose complexity more closely matches the pixel-space baseline.

Table~\ref{tab:RAE_metrics} reports both generation performance and
geometric diagnostics of latent diffusability~\citep{zhong2026diffusing}.
RAE and RAEv2 achieve the best FID scores and highest Inception Scores,
while maintaining strong sFID and precision. They also exhibit the
highest trajectory efficiency and relatively low SEC values,
suggesting more direct ODE transport and smoother latent spectral
structure. These results are broadly consistent with the FI/FIR
comparison, in which RAE and RAEv2 are the latent models closest to the
pixel-space JiT baseline, suggesting that FI and FIR capture properties
of latent representations relevant to diffusion-based generation.
The agreement between FI/FIR, generation quality, and latent-space
diagnostics may partly reflect encoder choice. The DINO-style
representation encoders used by RAE and RAEv2 may preserve semantic
and local neighborhood structure better than the VAE-style encoders
used by SiT and REPA-E, yielding more diffusible latent spaces.

\section{Conclusions}\label{sec:conclusion}

We introduced a formal information-geometric framework to demystify the ``diffusability'' of latent spaces in generative modeling. By decomposing denoising complexity 
into FI and FIR, 
we identified geometric requirements for successful latent diffusion.  Our theoretical analysis explicitly decouples the effects of dimensional compression from intrinsic geometric distortion, revealing that latent-space degradation is primarily driven by local curvature and nonlinearity injected by the encoder. 

Specifically, we identified near-isometry and controlled local
nonlinearity as complementary geometric requirements for preserving
the FI- and FIR-based components of denoising complexity. Our toy,
FFHQ, and ImageNet-based experiments show that empirical FI and FIR
track several measures of generation quality and latent-space
geometry in the settings tested.

FI and FIR are population quantities defined independently
of a trained diffusion model. The current scalable estimators,
however, generally use scores obtained from trained models, while
direct empirical-mixture estimation may become unreliable in high
dimension. Developing statistically reliable and scalable
pre-training estimators therefore remains an important direction for
future work.

\FloatBarrier

\section*{Acknowledgements}
This work was supported in part by the NVIDIA Academic Grants Program, startup funding from The Ohio State University (W.~Lee) and NSF award DMS 2427955 (G.~Lerman).

\bibliography{refs}
\bibliographystyle{plain}

\newpage
\appendix
\section*{Appendix}
The appendices are organized as follows: Appendix \ref{sec:related_works} reviews related literature, while Appendix \ref{sec:proofs} contains all missing proofs and derivations. Practical estimation of FI and FIR is detailed in Appendix \ref{app:FI_estimation}, and Appendix \ref{app:exp_detail} provides comprehensive experimental setups. 
Additional mathematical derivations and numerical supplements, relevant to the toy experimental settings, are provided in Appendix \ref{app:plot_gaussian_toy}. 
Finally, Appendices \ref{app:model_performance}--\ref{app:spectra} present further numerical results and spectral analyses, relevant to the real experimental settings.

\section{Related Works}
\label{sec:related_works}
The term ``diffusability'' was first introduced by Skorokhodov et al.~\citep{skorokhodov2025diffusability} to describe the extent to which a latent space facilitates the diffusion process. The authors argued that high-frequency components in latent representations reduce diffusability, though they did not provide a rigorous quantitative measure for this effect. Subsequently, \citep{zheng2026diffusion} offered further empirical evidence demonstrating that standard representation encoders often induce latent distributions that are inherently difficult for diffusion models to fit. 

Several follow-up works have aimed to improve diffusability through specific architectural or training modifications. For instance, \citep{kouzelis2025eq} promoted diffusability by enforcing equivariance in the latent representation, while \citep{yang2026latent} targeted the corruption robustness of the latents. Taking a different angle, \citep{heek2026unified} improved diffusability by training the encoder to produce latents that already resemble mildly noised diffusion states, allowing the diffusion prior to denoise from a smoother, better-aligned distribution. Rather than modifying the latent distribution directly, \citep{baade2026latent} proposed denoising latent features early in the pipeline and using them to condition later pixel-space denoising, effectively simplifying the overall problem.

Another distinct line of work focuses on explicitly constraining the latent geometry. In \citep{yue2026image}, latent codes are normalized to lie uniformly on a hypersphere. More broadly, alternative geometric priors for autoencoders have gained traction, including the use of hyperbolic latent spaces \citep{cho2023hyperbolic,saez2023neural}.

While these prior works have successfully improved diffusability through empirical regularization and structural constraints, they largely lack a unified mathematical framework to quantify why certain representation spaces are easier to model than others. In contrast, our work provides a rigorous geometric analysis of diffusability. By explicitly bounding the Fisher Information Rate (FIR) of the encoded distributions, we establish formal guarantees on how linear encoders and geometric constraints fundamentally alter the stability of the diffusion process.

\section{Supplementary Derivations and Proofs} \label{sec:proofs}


\subsection{Proof of Proposition~\ref{prop:FIR_identity}}\label{sec:proof_FIR}

Fix $\tau > 0$. 
Let $p_\tau(\bx)$ be the smooth density of the noisy measure $\mu_\tau$. As defined in \S\ref{sec:background}, the density is given by the convolution of the clean measure $\mu$ with the Gaussian kernel 
\begin{equation}
\label{eq:gaussian_kernel}
 \phi_\tau(\by) = (2\pi\tau)^{-k/2} \exp(-\|\by\|^2 / (2\tau))   
\end{equation}
as follows:
\begin{equation}
\label{eq:density_conv_def_app}
p_\tau(\bx) = (\mu * \phi_\tau)(\bx) = \int_{\mathbb{R}^k} \phi_\tau(\bx - \bx_0) \, d\mu(\bx_0).
\end{equation}
In view of this expression, \[     p_\tau\in C^\infty(\mathbb R^k),     \qquad     p_\tau>0. \] The integrations by parts below are justified using compactly supported cutoff functions and the standing regularity assumptions stated in the main text; they do not require $p_\tau$ itself to have Gaussian tails. We express the Fisher Information as:
\begin{equation}
\mathcal{I}^{(k)}(\mu_\tau) = \int_{\mathbb{R}^k} p_\tau(\bx) \|\nabla_{\bx} \log p_\tau(\bx)\|^2 \, d\bx = \int_{\mathbb{R}^k} \frac{\|\nabla_{\bx} p_\tau(\bx)\|^2}{p_\tau(\bx)} \, d\bx.
\end{equation}
Using the probabilistic heat equation $\frac{\partial p_\tau}{\partial \tau} = \frac{1}{2}\Delta_{\bx} p_\tau$ from \eqref{eq:heat_and_ode}, we differentiate $\mathcal{I}^{(k)}$ with respect to $\tau$:
\begin{align}
\frac{d}{d\tau}\mathcal{I}^{(k)}(\mu_\tau) &= \int_{\mathbb{R}^k} \left( \frac{2\nabla_{\bx} p_\tau \cdot \nabla_{\bx}(\partial_\tau p_\tau)}{p_\tau} - \frac{\|\nabla_{\bx} p_\tau\|^2}{p_\tau^2} \partial_\tau p_\tau \right) d\bx \nonumber \\
&= \int_{\mathbb{R}^k} \left( \frac{\nabla_{\bx} p_\tau \cdot \nabla_{\bx}(\Delta_{\bx} p_\tau)}{p_\tau} - \frac{1}{2}\frac{\|\nabla_{\bx} p_\tau\|^2}{p_\tau^2} \Delta_{\bx} p_\tau \right) d\bx. \label{eq:FI_deriv_expanded}
\end{align}
Integrating the first term by parts using a compactly supported cutoff and then passing to the limit:
\begin{align}
\int_{\mathbb{R}^k}
\frac{\nabla_{\bx}p_\tau\cdot
\nabla_{\bx}(\Delta_{\bx}p_\tau)}
{p_\tau}\,d\bx
&=
-\int_{\mathbb{R}^k}
\Delta_{\bx}p_\tau\,
\operatorname{div}
\left(
    \frac{\nabla_{\bx}p_\tau}{p_\tau}
\right)d\bx
\nonumber\\
&= -\int_{\mathbb{R}^k} \Delta_{\bx} p_\tau \left( \frac{\Delta_{\bx} p_\tau}{p_\tau} - \frac{\|\nabla_{\bx} p_\tau\|^2}{p_\tau^2} \right) d\bx.
\end{align}
Substituting this back into \eqref{eq:FI_deriv_expanded} yields:
\begin{equation}
\frac{d}{d\tau}\mathcal{I}^{(k)}(\mu_\tau) = \int_{\mathbb{R}^k} \left[ -\frac{(\Delta_{\bx} p_\tau)^2}{p_\tau} + \frac{1}{2}\Delta_{\bx} p_\tau \frac{\|\nabla_{\bx} p_\tau\|^2}{p_\tau^2} \right] d\bx.
\end{equation}
To simplify, let $f = \log p_\tau$, $A = \Delta_{\bx} f$, and $B = \|\nabla_{\bx} f\|^2$. Noting that $\frac{\Delta_{\bx} p_\tau}{p_\tau} = \Delta_{\bx} \log p_\tau + \|\nabla_{\bx} \log p_\tau\|^2 = A + B$, we have:
\begin{align}
\frac{d}{d\tau}\mathcal{I}^{(k)}(\mu_\tau) &= \int_{\mathbb{R}^k} p_\tau \left[ -(A+B)^2 + \frac{1}{2}(A+B)B \right] d\bx \nonumber \\
&= \int_{\mathbb{R}^k} p_\tau \left[ -A^2 - \frac{3}{2}AB - \frac{1}{2}B^2 \right] d\bx. \label{eq:diff_FI_expanded_app}
\end{align}
Now, let $\mathbf{H}_\tau = \nabla^2_{\bx} \log p_\tau$ be the Hessian. The Euclidean Bochner identity for $f = \log p_\tau$ states:
\begin{equation}
\frac{1}{2} \Delta_{\bx} \|\nabla_{\bx} f\|^2 = \|\nabla^2_{\bx} f\|_F^2 + \nabla_{\bx} f \cdot \nabla_{\bx} (\Delta_{\bx} f).
\end{equation}
Integrating $\frac{1}{2} \Delta_{\bx} B$ and $-\nabla_{\bx} f \cdot \nabla_{\bx} A$ against the density $p_\tau$ via integration by parts, and noting that $\divr(p_\tau \nabla_{\bx} f) = \nabla_{\bx} p_\tau \cdot \nabla_{\bx} f + p_\tau \Delta_{\bx} f = p_\tau(B+A)$, we obtain:
\begin{align*}
\int_{\mathbb{R}^k} p_\tau \frac{1}{2} \Delta_{\bx} B \, d\bx 
&= \int_{\mathbb{R}^k} \frac{1}{2} B \Delta_{\bx} p_\tau \, d\bx = \int_{\mathbb{R}^k} p_\tau \left( \frac{1}{2} AB + \frac{1}{2} B^2 \right) d\bx, \\
-\int_{\mathbb{R}^k} p_\tau \nabla_{\bx} f \cdot \nabla_{\bx} A \, d\bx &= \int_{\mathbb{R}^k} A \, \divr(p_\tau \nabla_{\bx} f) \, d\bx = \int_{\mathbb{R}^k} p_\tau (A^2 + AB) \, d\bx.
\end{align*}
Summing these yields the weighted Bochner identity: 
\begin{equation}
\mathbb{E}_{p_\tau}[\|\mathbf{H}_\tau\|_F^2] = \mathbb{E}_{p_\tau}\left[A^2 + \frac{3}{2}AB + \frac{1}{2}B^2\right].
\end{equation}
Comparing this to \eqref{eq:diff_FI_expanded_app}, we find:
\begin{equation}
\label{eq:express_deriv_I_app}
\frac{d}{d\tau}\mathcal{I}^{(k)}(\mu_\tau) = - \int_{\mathbb{R}^k} p_\tau \|\mathbf{H}_\tau\|_F^2 \, d\bx = - \mathbb{E}_{p_\tau} \left[ \mathrm{Tr}(\mathbf{H}_\tau^2) \right]. 
\end{equation}
By \eqref{eq:FI_FIR_def}, $\mathcal{R}^{(k)}(\mu_\tau) = -\frac{d}{d\tau}\mathcal{I}^{(k)}(\mu_\tau)$. Substituting the score $\bs_\tau(\bx) = \nabla_{\bx} \log p_\tau$, we have $\nabla_{\bx} \bs_\tau = \mathbf{H}_\tau$, and $\mathcal{R}^{(k)}(\mu_\tau) = \mathbb{E}_{p_\tau} [ \|\nabla_{\bx} \bs_\tau(\bx)\|_F^2 ]$, which completes the proof.


\subsection{Derivation of \eqref{eq:MMSE_and_I} and \eqref{eq:denoising_resistance}}\label{sec:appendix_prove_4_5}

In this section, we provide the detailed derivation for the I-MMSE identity and the Denoising Resistance decomposition. 
Consider a clean signal $\bx_0 \in \mathbb{R}^k$ drawn from a distribution $\mu$, and its corruption by additive white Gaussian noise $\bn \sim \mathcal{N}(\bzero, \bI_k)$, resulting in the noisy observation:
\begin{equation*}
\bx = \bx_0 + \sqrt{\tau}\bn,
\end{equation*}
where $\tau$ is the noise variance. The total variance of the corruption is $\mathbb{E}[\|\bx_0 - \bx\|^2] = \mathbb{E}[\|-\sqrt{\tau}\bn\|^2] = \tau k$. 

Let $\hat{\bx}_0(\bx) = \mathbb{E}[\bx_0 \mid \bx]$ be the MMSE estimator. We can decompose the total noise variance by adding and subtracting this estimator:
\begin{multline}
\mathbb{E}[\|\bx_0 - \bx\|^2] = \mathbb{E}[\|(\bx_0 - \hat{\bx}_0(\bx)) + (\hat{\bx}_0(\bx) - \bx)\|^2]  \\
= \mathbb{E}[\|\bx_0 - \hat{\bx}_0(\bx)\|^2] + \mathbb{E}[\|\hat{\bx}_0(\bx) - \bx\|^2] + 2\mathbb{E}[(\bx_0 - \hat{\bx}_0(\bx))^\top(\hat{\bx}_0(\bx) - \bx)]. \label{eq:var_decomp}
\end{multline}
The cross-term vanishes by the tower property of conditional expectation. Since $(\hat{\bx}_0(\bx) - \bx)$ is deterministic given $\bx$, we have:
\begin{align}
\mathbb{E}[(\bx_0 - \hat{\bx}_0(\bx))^\top(\hat{\bx}_0(\bx) - \bx)] &= \mathbb{E}_{\bx}\left[ \mathbb{E}[(\bx_0 - \hat{\bx}_0(\bx))^\top(\hat{\bx}_0(\bx) - \bx) \mid \bx] \right] \nonumber \\
&= \mathbb{E}_{\bx}\left[ (\mathbb{E}[\bx_0 \mid \bx] - \hat{\bx}_0(\bx))^\top(\hat{\bx}_0(\bx) - \bx) \right] \nonumber \\
&= \mathbb{E}_{\bx}\left[ (\hat{\bx}_0(\bx) - \hat{\bx}_0(\bx))^\top(\hat{\bx}_0(\bx) - \bx) \right] = 0.
\end{align}
Next, we invoke Tweedie's formula \citep{Efron01122011}, which relates the residual of the MMSE estimator to the score function $\bs_\tau(\bx) = \nabla_{\bx} \log p_\tau(\bx)$:
\begin{equation}
\label{eq:tweedie_app}
\hat{\bx}_0(\bx) - \bx = \tau \nabla_{\bx} \log p_\tau(\bx) = \tau \bs_\tau(\bx).
\end{equation}
Substituting \eqref{eq:tweedie_app} and the definition $\text{MMSE}^{(k)}(\mu_\tau) = \mathbb{E}[\|\bx_0 - \hat{\bx}_0(\bx)\|^2]$ into \eqref{eq:var_decomp} yields:
\begin{equation}
\tau k = \text{MMSE}^{(k)}(\mu_\tau) + \tau^2 \mathbb{E}_{p_\tau}[\|\bs_\tau(\bx)\|^2] = \text{MMSE}^{(k)}(\mu_\tau) + \tau^2 \mathcal{I}^{(k)}(\mu_\tau).
\end{equation}
Rearranging for $\text{MMSE}^{(k)}(\mu_\tau)$ directly recovers the I-MMSE identity \eqref{eq:MMSE_and_I} as established in \citep{Guo05_fisherL2}.

To derive the Denoising Resistance, we differentiate \eqref{eq:MMSE_and_I} with respect to the noise variance $\tau$ using the product rule:
\begin{equation}\label{eq:diff_step1_app}
\frac{d}{d\tau} \text{MMSE}^{(k)}(\mu_\tau) = \frac{d}{d\tau} \left( k\tau - \tau^2 \mathcal{I}^{(k)}(\mu_\tau) \right) = k - 2\tau \mathcal{I}^{(k)}(\mu_\tau) - \tau^2 \frac{d}{d\tau} \mathcal{I}^{(k)}(\mu_\tau).
\end{equation}
Substituting our definition of the Fisher Information Rate (FIR) from \eqref{eq:FI_FIR_def}, namely $\mathcal{R}^{(k)}(\mu_\tau) = -\frac{d}{d\tau}\mathcal{I}^{(k)}(\mu_\tau)$, we obtain the fundamental decomposition:
\begin{equation}
\label{eq:denoising_resistance_final}
\frac{d}{d\tau} \text{MMSE}^{(k)}(\mu_\tau) = \underbrace{k - 2\tau \mathcal{I}^{(k)}(\mu_\tau)}_{\text{Intrinsic Noise Gain}} + \underbrace{\tau^2 \mathcal{R}^{(k)}(\mu_\tau)}_{\text{Geometric Complexity Penalty}}.
\end{equation}
As noted in the main text, the \emph{Intrinsic Noise Gain} represents the baseline linear increase in estimation error due to dimensionality $k$, attenuated by the current magnitude of information $\mathcal{I}^{(k)}$. Conversely, the \emph{Geometric Complexity Penalty} represents the additional difficulty imposed by the local curvature of the density. By Proposition~\ref{prop:FIR_identity}, this term is governed by $\mathbb{E}_{p_\tau}[\|\nabla \bs_\tau\|_F^2]$, indicating that distributions with rapidly changing score functions (high curvature) exhibit a faster degradation in denoising performance as noise scales.

\subsection{Proof of Proposition \ref{prop:fisher_bi_lip}}\label{sec:proof_fisher_stability}

\begin{proposition}[Intrinsic Fisher Information under Bi-Lipschitz Encoders]
\label{prop:fisher_bi_lip}

Let $\mathcal M\subset\mathbb R^D$ and
$\mathcal N\subset\mathbb R^d$ be smooth $m$-dimensional manifolds,
and let
\(
    E:\mathcal M\longrightarrow\mathcal N
\)
be a bi-Lipschitz homeomorphism onto $\mathcal N$. Assume that
\(
    c\,d_{\mathcal M}(x,y)
    \le
    d_{\mathcal N}(E(x),E(y))
    \le
    C\,d_{\mathcal M}(x,y),
\)
for all $x,y\in\mathcal M$ and some constants
$0<c\le C<\infty$, where $d_{\mathcal M}$ and $d_{\mathcal N}$
denote the intrinsic Riemannian distances. 
By Rademacher's theorem in local charts, the tangential differential
\(
    DE(x):T_x\mathcal M\longrightarrow T_{E(x)}\mathcal N
\)
exists for almost every $x\in\mathcal M$, and its singular values
belong to $[c,C]$ almost everywhere. Define the tangential Jacobian
\(
    \mathcal J_E(x)
    :=
    \sqrt{
        \det
        \left(
            DE(x)^\top DE(x)
        \right)
    }
\)
at points of differentiability. 
Let
\[
    d\nu(x)
    =
    p_{\mathcal M}(x)\,
    d\operatorname{vol}_{\mathcal M}(x),
\]
where $p_{\mathcal M}>0$ almost everywhere and
\[
    \log p_{\mathcal M}
    \in
    W_{\mathrm{loc}}^{1,1}(\mathcal M).
\]
Define the weak logarithmic score by
\[
    s_{\mathcal M}
    :=
    \nabla_{\mathcal M}\log p_{\mathcal M},
\]
and assume
\[
    s_{\mathcal M}\in L^2(\nu).
\]
Define
\[
    \mathcal I_{\mathcal M}
    :=
    \int_{\mathcal M}
    \|s_{\mathcal M}(x)\|^2\,d\nu(x)
    <\infty.
\]
Assume also that
\(
    \log\mathcal J_E
    \in
    W_{\mathrm{loc}}^{1,1}(\mathcal M),
\)
and define its weak tangential gradient by
\(
    g_E
    :=
    \nabla_{\mathcal M}\log\mathcal J_E.
\)
Suppose that
\[
    \widetilde\varepsilon^2
    :=
    \int_{\mathcal M}
    \|g_E(x)\|^2\,d\nu(x)
    <\infty.
\]

Let
\(
    \widetilde\nu:=E_\#\nu.
\)
Then $\widetilde\nu$ has finite intrinsic Fisher information on
$\mathcal N$, and the following exact identity holds:
\[
    \mathcal I_{\mathcal N}
    =
    \int_{\mathcal M}
    \left\|
        DE(x)^{-\top}
        \bigl(
            s_{\mathcal M}(x)-g_E(x)
        \bigr)
    \right\|^2
    d\nu(x),
\]
where $DE(x)^{-\top}$ denotes the inverse adjoint between the tangent
spaces.

Consequently, for every $\theta\in(0,1)$,
\[
\frac{1-\theta}{C^2}\mathcal I_{\mathcal M}
-
\frac{
    (\theta^{-1}-1)\widetilde\varepsilon^2
}{
    C^2
}
\le
\mathcal I_{\mathcal N}
\le
\frac{1+\theta}{c^2}\mathcal I_{\mathcal M}
+
\frac{
    (\theta^{-1}+1)\widetilde\varepsilon^2
}{
    c^2
}.
\]
If $E$ is an isometry, then
\(
    \mathcal I_{\mathcal N}
    =
    \mathcal I_{\mathcal M}.
\)
\end{proposition}

\begin{proof}
Let $\widetilde p_{\mathcal N}$ denote the density of
$\widetilde\nu=E_\#\nu$ with respect to the volume measure on
$\mathcal N$. The area formula for bi-Lipschitz maps gives, for almost
every $x\in\mathcal M$,
\[
    \widetilde p_{\mathcal N}(E(x))
    =
    \frac{
        p_{\mathcal M}(x)
    }{
        \mathcal J_E(x)
    }.
\]
Hence
\[
    \log\widetilde p_{\mathcal N}(E(x))
    =
    \log p_{\mathcal M}(x)
    -
    \log\mathcal J_E(x).
\]
The Sobolev chain rule therefore gives
\[
    DE(x)^\top
    \widetilde s_{\mathcal N}(E(x))
    =
    s_{\mathcal M}(x)-g_E(x)
\]
for almost every $x$, where
\[
    \widetilde s_{\mathcal N}
    :=
    \nabla_{\mathcal N}\log\widetilde p_{\mathcal N}.
\]
Since $DE(x)$ is invertible between the tangent spaces almost
everywhere,
\[
    \widetilde s_{\mathcal N}(E(x))
    =
    DE(x)^{-\top}
    \bigl(
        s_{\mathcal M}(x)-g_E(x)
    \bigr).
\]
Changing variables using $\widetilde\nu=E_\#\nu$ yields the exact
identity
\[
\begin{aligned}
\mathcal I_{\mathcal N}
&=
\int_{\mathcal N}
\|\widetilde s_{\mathcal N}(z)\|^2\,d\widetilde\nu(z)
\\
&=
\int_{\mathcal M}
\left\|
    DE(x)^{-\top}
    \bigl(
        s_{\mathcal M}(x)-g_E(x)
    \bigr)
\right\|^2
d\nu(x).
\end{aligned}
\]

The singular values of $DE(x)$ belong to $[c,C]$ almost everywhere,
so
\[
    \frac{1}{C}
    \|w\|
    \le
    \|DE(x)^{-\top}w\|
    \le
    \frac{1}{c}
    \|w\|.
\]
For every $\theta\in(0,1)$,
\[
    \|a-b\|^2
    \ge
    (1-\theta)\|a\|^2
    -
    (\theta^{-1}-1)\|b\|^2
\]
and
\[
    \|a-b\|^2
    \le
    (1+\theta)\|a\|^2
    +
    (\theta^{-1}+1)\|b\|^2.
\]
Apply these inequalities with
\[
    a=s_{\mathcal M}(x),
    \qquad
    b=g_E(x),
\]
and integrate to obtain the two stated bounds.

If $E$ is an isometry, then
\[
    DE(x)^\top DE(x)=I_m,
    \qquad
    \mathcal J_E(x)=1
\]
almost everywhere. 
Thus $g_E=0$ and
\[
    \mathcal I_{\mathcal N}
    =
    \mathcal I_{\mathcal M}.
\]
\end{proof}

\begin{remark}[Piecewise-smooth encoders and activation interfaces]
The proposition allows weakly differentiable bi-Lipschitz encoders,
but it still requires
\(
    \log\mathcal J_E
\)
to have an $L^2(\nu)$ weak gradient. For a piecewise-smooth encoder,
this condition holds when its logarithmic volume Jacobian has no jump
across the interfaces, together with the stated square-integrability
on each smooth region.

If $\log\mathcal J_E$ has a nonzero jump across a codimension-one
interface and the input density has a positive trace there, the
deterministic pushforward density generally has a jump. Under the
Sobolev definition of Fisher information compatible with the
heat-semigroup identities, such a jump produces infinite Fisher
information.

Leaky ReLU is globally bi-Lipschitz for every $\alpha>0$, but for
$\alpha\ne1$ its Jacobian determinant jumps across the activation
hyperplanes. It is therefore not covered directly by
Proposition~\ref{prop:fisher_bi_lip}, despite being bi-Lipschitz.

Gaussian smoothing nevertheless regularizes its latent law. More
generally, as shown in
Appendix~\ref{sec:gaussian_smoothing_fi}, if $\mu_Z:=E_\#\nu$
is any latent probability measure, then
Lemma~\ref{lem:gaussian_smoothing_finite_FI} gives
\[
    \mathcal I^{(d)}(\mu_{Z,\tau})
    \le
    \frac{d}{\tau},
    \qquad
    \tau>0.
\]
Thus, the Gaussian-smoothed Leaky-ReLU latent distribution has finite
FI at every positive noise variance, even though the unsmoothed
deterministic pushforward may fail to have finite Sobolev FI. For the
Gaussian toy setting, Appendix~\ref{app:gaussian_activation_smoothed_FI}
provides exact smoothed densities and a one-dimensional integral
representation of the corresponding Fisher information.
\end{remark}

\subsection{Gaussian Smoothing and Fisher Information}
\label{sec:gaussian_smoothing_fi}

The following result applies to an arbitrary probability measure and
does not require the encoder or the unsmoothed latent law to be smooth.

\begin{lemma}[Gaussian smoothing regularizes Fisher information]
\label{lem:gaussian_smoothing_finite_FI}
Let $\eta$ be an arbitrary Borel probability measure on
$\mathbb R^d$, and let
\[
    \eta_\tau
    :=
    \eta*\phi_\tau^{(d)},
    \qquad
    \tau>0,
\]
where $\phi_\tau^{(d)}$ is the centered Gaussian density with
covariance $\tau I_d$. Then $\eta_\tau$ has a strictly positive
$C^\infty$ density $p_\tau$, and
\[
    \mathcal I^{(d)}(\eta_\tau)
    \le
    \frac{d}{\tau}.
\]
In particular, the Gaussian-smoothed measure has finite Fisher
information for every $\tau>0$, without any moment or smoothness
assumption on $\eta$.
\end{lemma}

\begin{proof}
Let $Z\sim\eta$ and $G\sim\mathcal N(\bzero,I_d)$ be independent, and
define
\[
    Y=Z+\sqrt{\tau}\,G.
\]
The density of $Y$ is
\[
    p_\tau(y)
    =
    \int_{\mathbb R^d}
    \phi_\tau^{(d)}(y-z)\,d\eta(z).
\]
Differentiating under the integral gives
\[
\begin{aligned}
\nabla p_\tau(y)
&=
-\frac{1}{\tau}
\int_{\mathbb R^d}
(y-z)\phi_\tau^{(d)}(y-z)\,d\eta(z)
\\
&=
-\frac{p_\tau(y)}{\sqrt{\tau}}\,
\mathbb E[G\mid Y=y].
\end{aligned}
\]
Consequently, the score satisfies
\[
    \nabla\log p_\tau(y)
    =
    -\frac{1}{\sqrt{\tau}}\,
    \mathbb E[G\mid Y=y].
\]
Therefore,
\[
\begin{aligned}
\mathcal I^{(d)}(\eta_\tau)
&=
\frac{1}{\tau}
\mathbb E
\left[
    \left\|
        \mathbb E[G\mid Y]
    \right\|^2
\right]
\\
&\le
\frac{1}{\tau}
\mathbb E\|G\|^2
=
\frac{d}{\tau},
\end{aligned}
\]
where the inequality follows from the $L^2$ contraction property of
conditional expectation.
\end{proof}

\subsection{Dimensional Compression on Flat Supports} \label{sec:cor_fisher_stability}

\begin{corollary}[Exact dimensional compression penalty on flat supports]
\label{cor:dim_compression}

Let $\mu_X$ and $\mu_Z$ be probability measures supported on affine
$m$-dimensional subspaces of $\mathbb R^D$ and $\mathbb R^d$,
respectively. Let
\[
    \mathcal I_{\mathrm{tan}}^X(\tau)
    \quad\text{and}\quad
    \mathcal I_{\mathrm{tan}}^Z(\tau)
\]
denote the Fisher informations of their corresponding tangential
coordinate laws after Gaussian smoothing at variance $\tau$.

Then the ambient Fisher informations decompose exactly as
\[
    \mathcal I^{(D)}(\mu_{X,\tau})
    =
    \mathcal I_{\mathrm{tan}}^X(\tau)
    +
    \frac{D-m}{\tau},
\]
and
\[
    \mathcal I^{(d)}(\mu_{Z,\tau})
    =
    \mathcal I_{\mathrm{tan}}^Z(\tau)
    +
    \frac{d-m}{\tau}.
\]

Define the Intrinsic Noise Gain by
\[
    G^{(k)}(\eta_\tau)
    :=
    k-2\tau\mathcal I^{(k)}(\eta_\tau).
\]
Then
\[
G^{(d)}(\mu_{Z,\tau})
-
G^{(D)}(\mu_{X,\tau})
=
D-d
+
2\tau
\left(
    \mathcal I_{\mathrm{tan}}^X(\tau)
    -
    \mathcal I_{\mathrm{tan}}^Z(\tau)
\right).
\]
In particular, if the tangential Fisher information is preserved,
\[
    \mathcal I_{\mathrm{tan}}^X(\tau)
    =
    \mathcal I_{\mathrm{tan}}^Z(\tau),
\]
then
\[
    G^{(d)}(\mu_{Z,\tau})
    -
    G^{(D)}(\mu_{X,\tau})
    =
    D-d.
\]
\end{corollary}

\begin{proof}
Translations and orthogonal transformations preserve Fisher
information, so it suffices to identify the affine supports with
\[
    \mathbb R^m\times\{\bzero\}
    \subset\mathbb R^D
\]
and
\[
    \mathbb R^m\times\{\bzero\}
    \subset\mathbb R^d.
\]
Let $\rho_X$ and $\rho_Z$ denote the corresponding tangential
coordinate laws. Gaussian smoothing gives the exact factorizations
\[
    p_{X,\tau}(y,w)
    =
    (\rho_X*\phi_\tau^{(m)})(y)
    \phi_\tau^{(D-m)}(w)
\]
and
\[
    p_{Z,\tau}(y,v)
    =
    (\rho_Z*\phi_\tau^{(m)})(y)
    \phi_\tau^{(d-m)}(v).
\]
The score vectors split orthogonally into tangential and normal
components. Since
\[
    \mathcal I
    \bigl(
        \mathcal N(\bzero,\tau I_\ell)
    \bigr)
    =
    \frac{\ell}{\tau},
\]
we obtain
\[
    \mathcal I^{(D)}(\mu_{X,\tau})
    =
    \mathcal I_{\mathrm{tan}}^X(\tau)
    +
    \frac{D-m}{\tau}
\]
and
\[
    \mathcal I^{(d)}(\mu_{Z,\tau})
    =
    \mathcal I_{\mathrm{tan}}^Z(\tau)
    +
    \frac{d-m}{\tau}.
\]
Substitution into the definition of $G^{(k)}$ proves
\[
G^{(d)}(\mu_{Z,\tau})
-
G^{(D)}(\mu_{X,\tau})
=
D-d
+
2\tau
\left(
    \mathcal I_{\mathrm{tan}}^X(\tau)
    -
    \mathcal I_{\mathrm{tan}}^Z(\tau)
\right).
\]
The final assertion follows when the tangential Fisher informations
agree.
\end{proof}

\begin{remark}[Curved supports]
The exact decomposition in
Corollary~\ref{cor:dim_compression} relies on the existence of a
single fixed normal subspace. For a curved manifold, the tangent and
normal spaces vary with the base point, and Gaussian smoothing mixes
contributions from neighboring tangent spaces. Consequently, the
ambient smoothed density does not globally factor into independent
tangential and normal components.

The tangent-plane model identifies
\[
    \frac{k-m}{\tau}
\]
as the flat local baseline contributed by normal Gaussian directions.
For a curved manifold, establishing a global small-noise expansion
with this leading term and a controlled curvature remainder requires
additional assumptions and analysis. We do not claim such a global
expansion here.

Accordingly, on curved supports, the dimensional shift $D-d$ should
be interpreted as the flat tangent-space compression baseline rather
than as an exact global identity. The localized results of
Theorem~\ref{thm:FIR_stability_nonlinear} and
Corollary~\ref{cor:combined_local_bound} quantify the corresponding
curvature corrections for FIR.
\end{remark}


\subsection{Score Jacobian and Curvature Injection}\label{app:curvature_injection}

This subsection records a formal local differentiation of the
intrinsic score-transformation identity. We therefore impose stronger
regularity than in Proposition~\ref{prop:fisher_bi_lip}: assume here
that $E$ is $C^3$, that its tangential differential has full rank, and
that the intrinsic densities are sufficiently smooth for the score
fields below to be differentiable. These assumptions are used only in
this explanatory subsection and not in
Theorem~\ref{thm:FIR_stability_nonlinear}.

Let
\[
    \mathcal J_E(x)
    :=
    \sqrt{
        \det
        \left(
            DE(x)^\top DE(x)
        \right)
    }
\]
be the tangential volume Jacobian and define
\[
    g(x)
    :=
    \nabla_{\mathcal M}\log\mathcal J_E(x).
\]
The exact intrinsic score transformation from
Proposition~\ref{prop:fisher_bi_lip} is
\[
    s_{\mathcal M}(x)
    =
    DE(x)^\top
    \widetilde s_{\mathcal N}(E(x))
    +
    g(x).
\]

Under these additional assumptions, differentiating the identity in flat local coordinates gives
\[
\begin{aligned}
\nabla s_{\mathcal M}(x)
={}&
\underbrace{
DE(x)^\top
\bigl(
    \nabla\widetilde s_{\mathcal N}
\bigr)(E(x))
DE(x)
}_{\text{inherited score curvature}}
\\
&+
\underbrace{
\sum_{i=1}^d
[\widetilde s_{\mathcal N}(E(x))]_i
\nabla^2E_i(x)
}_{\text{encoder curvature injection}}
+
\underbrace{
\nabla g(x)
}_{\text{volume-distortion curvature}}.
\end{aligned}
\]
On a general manifold, the corresponding statement is written using
covariant derivatives.

Thus, encoder second derivatives contribute directly to the
data-versus-latent score-Jacobian discrepancy. The
volume-distortion term must also be controlled; a small pointwise
metric distortion alone does not imply that $\nabla g$ is small.

\subsection{Normal-Direction Decomposition for Subspace Measures}

The following lemma establishes the ``Dimensional Compression Penalty'' by showing that normal noise directions contribute a universal baseline to the FIR. This term depends purely on the noise scale $\tau$ and the codimension $D-m$, but is constant across the manifold and independent of the data distribution. It is used to prove \Cref{thm:FIR_stability_d_less_D_linear} in Appendix \ref{app:proof_thm_1}.

\begin{lemma}[Exact normal-direction contribution]
\label{lem:normal_term}
Let $1\le m\le D$ and identify $\mathbb R^D\cong \mathbb R^m\times\mathbb R^{D-m}$.
Let $\rho$ be a Borel probability measure on $\mathbb R^m$, and define a measure
$\mu$ on $\mathbb R^D$ by $\mu := \rho\otimes \delta_{\bzero}$. Recalling that $\phi_\tau$ denotes the Gaussian kernel defined in \eqref{eq:gaussian_kernel}, we note that for every variance $\tau>0$, the smoothed density $\mu_\tau$ factorizes as:
\[
\mu_\tau(\by,\bw) = (\rho*\phi_\tau^{(m)})(\by)\,\phi_\tau^{(D-m)}(\bw), \qquad (\by,\bw)\in\mathbb R^m\times\mathbb R^{D-m},
\]
and the Fisher Information Rate satisfies the exact splitting:
\[
\mathcal{R}^{(D)}(\mu_\tau) = \mathcal{R}^{(m)}(\rho_\tau) + \frac{D-m}{\tau^2}.
\]
\end{lemma}

\begin{proof}
Fix $\tau>0$. By the definition of convolution and the product structure of the standard Gaussian kernel, we have $\mu_\tau(\by,\bw) = \int_{\mathbb R^m} \phi_\tau^{(m)}(\by-\by')\,\phi_\tau^{(D-m)}(\bw)\,\rho(d\by')$. This gives the claimed factorization.
Using $\log\mu_\tau(\by,\bw) = \log(\rho*\phi_\tau^{(m)})(\by) + \log\phi_\tau^{(D-m)}(\bw)$, the Hessian $\nabla^2\log\mu_\tau$ is block diagonal:
\[
\nabla^2_{(\by,\bw)}\log\mu_\tau = \text{diag}\left(\nabla_{\by}^2\log(\rho*\phi_\tau^{(m)})(\by), \nabla_{\bw}^2\log\phi_\tau^{(D-m)}(\bw)\right).
\]
Consequently, $\|\nabla^2\log\mu_\tau\|_F^2 = \|\nabla_{\by}^2\log\rho_\tau\|_F^2 + \|\nabla_{\bw}^2\log\phi_\tau^{(D-m)}\|_F^2$. 
Applying the identity from \Cref{prop:FIR_identity} and using Fubini's theorem to integrate over $\mathbb{R}^m \times \mathbb{R}^{D-m}$:
\[
\mathcal{R}^{(D)}(\mu_\tau) = \mathcal{R}^{(m)}(\rho_\tau) + \int_{\mathbb R^{D-m}} \phi_\tau^{(D-m)}(\bw) \|\nabla_{\bw}^2\log\phi_\tau^{(D-m)}(\bw)\|_F^2 d\bw.
\]
Since $\nabla_{\bw}^2\log\phi_\tau^{(D-m)}(\bw)=-\frac{1}{\tau}\bI_{D-m}$, the second term evaluates exactly to $\frac{D-m}{\tau^2}$, noting $\int \phi_\tau = 1$.
\end{proof}

\subsection{Proof of \Cref{thm:FIR_stability_d_less_D_linear} (Linear Stability)}
\label{app:proof_thm_1}

By Lemma~\ref{lem:normal_term}, $\mathcal{R}^{(D)}(\mu_\tau) = \mathcal{R}^{(m)}(\rho_\tau) + \frac{D-m}{\tau^2}$. Similarly, $\mu_Z$ is supported on an $m$-plane in $\mathbb{R}^d$ with tangential coordinate law $\bA_\#\rho$. Thus, 
$$\mathcal{R}^{(d)}(\mu_{Z,\tau}) = \mathcal{R}^{(m)}((\bA_\#\rho)_\tau) + \frac{d-m}{\tau^2}.$$
The difference is:
\begin{equation}\label{eq:diff_linear}
\mathcal{R}^{(D)}(\mu_\tau) - \mathcal{R}^{(d)}(\mu_{Z,\tau}) = \left(\mathcal{R}^{(m)}(\rho_\tau) - \mathcal{R}^{(m)}((\bA_\#\rho)_\tau)\right) + \frac{D-d}{\tau^2}.
\end{equation}
For convenience, set
\[
    R_p(\tau):=\mathcal{R}^{(m)}(\rho_\tau),
    \qquad
    R_q(\tau):=
    \mathcal{R}^{(m)}((\bA_\#\rho)_\tau).
\]

Let
\[
    \bA=\mathbf{O}\mathbf{S}
\]
be the polar decomposition of $\bA$, where $\mathbf{O}$ is orthogonal
and
\[
    \mathbf{S}=(\bA^\top\bA)^{1/2}.
\]
Since the Gaussian noise is isotropic, orthogonal transformations
preserve FIR. Therefore,
\[
    R_q(\tau)
    =
    \mathcal{R}^{(m)}((\mathbf{S}_\#\rho)_\tau).
\]

Let $\sigma_1,\ldots,\sigma_m$ be the eigenvalues of $\mathbf{S}$.
The near-isometry condition implies
\[
    1-\delta\le \sigma_i^2\le 1+\delta.
\]
Consequently,
\[
    |\sigma_i-1|
    =
    \frac{|\sigma_i^2-1|}{\sigma_i+1}
    \le \delta,
\]
and hence
\[
    \|\mathbf{S}-\bI_m\|_2\le\delta.
\]

For every $\bx\in\operatorname{supp}(\rho)$, the assumption
$\tau\ge R_\rho^2$ gives
\[
    \|\mathbf{S}\bx\|
    \le
    \sqrt{1+\delta}\,\|\bx\|
    \le
    \sqrt{2}\,R_\rho
    \le
    \sqrt{2\tau},
\]
and
\[
    \|\bx-\mathbf{S}\bx\|
    \le
    \delta R_\rho
    \le
    \delta\sqrt{\tau}.
\]

We apply Lemma~\ref{lem:abstract_FIR_stability}, in dimension $m$,
to the maps
\[
    G(\bx)=\bx,
    \qquad
    F(\bx)=\mathbf{S}\bx,
\]
with common prior $\rho$, center $\bc=\bzero$, $L=\sqrt{2}$, and
$a_0=1$. We obtain
\begin{equation}
\label{eq:linear_tangential_absolute}
    \tau^2
    |R_p(\tau)-R_q(\tau)|
    \le
    C_m^{(1)}\delta,
\end{equation}
where $C_m^{(1)}$ depends only on $m$.

We also need a uniform upper bound for $R_p(\tau)$. Let
\[
    \bar{\rho}
    :=
    (T_{1/\sqrt{\tau}})_\#\rho,
    \qquad
    T_{1/\sqrt{\tau}}(\bx)
    :=
    \frac{\bx}{\sqrt{\tau}}.
\]
Since $\tau\ge R_\rho^2$,
\[
    \operatorname{supp}(\bar{\rho})
    \subseteq B_m(\bzero,1).
\]
Let $\bar{\bX}\sim\bar{\rho}$ and
\[
    \bar{\bY}
    =
    \bar{\bX}+\bG,
    \qquad
    \bG\sim\mathcal{N}(\bzero,\bI_m).
\]
The second-order Tweedie identity gives
\[
    \nabla^2\log(\bar{\rho}*\gamma_1)(\by)
    =
    -\bI_m
    +
    \operatorname{Cov}
    \bigl(
        \bar{\bX}\mid\bar{\bY}=\by
    \bigr).
\]
Because $\|\bar{\bX}\|\le1$ almost surely,
\[
    \left\|
        \operatorname{Cov}
        \bigl(
            \bar{\bX}\mid\bar{\bY}=\by
        \bigr)
    \right\|_F
    \le
    \operatorname{Tr}
    \operatorname{Cov}
    \bigl(
        \bar{\bX}\mid\bar{\bY}=\by
    \bigr)
    \le1.
\]
It follows that
\begin{equation}
\label{eq:linear_tangential_upper}
    \tau^2R_p(\tau)
    \le
    K_m,
    \qquad
    K_m:=(\sqrt{m}+1)^2.
\end{equation}

We now consider the two cases in the theorem.

\paragraph{Case $m=d=D$.}
In this case,
\[
    \mathcal{D}_{\mathcal{R}}(\tau)
    =
    \frac{
        |R_p(\tau)-R_q(\tau)|
    }{
        R_q(\tau)
    }.
\]
Lemma~\ref{lem:universal_FIR_lower_bound}, applied to
$F(\bx)=\mathbf{S}\bx$ with $L=\sqrt{2}$, gives
\[
    \tau^2R_q(\tau)
    \ge
    \frac{m}{(1+2)^2}
    =
    \frac{m}{9}.
\]
Together with \eqref{eq:linear_tangential_absolute}, this gives
\[
    \mathcal{D}_{\mathcal{R}}(\tau)
    \le
    \frac{9C_m^{(1)}}{m}\delta.
\]
After enlarging the dimension-dependent constant,
\[
    \mathcal{D}_{\mathcal{R}}(\tau)
    \le C_m\delta.
\]

\paragraph{Case $m<d\le D$.}
Set
\[
    \alpha:=\frac{d-m}{D-m}.
\]
By Lemma~\ref{lem:normal_term},
\[
    \mathcal{R}^{(D)}(\mu_\tau)
    =
    R_p(\tau)+\frac{D-m}{\tau^2},
\]
and
\[
    \mathcal{R}^{(d)}(\mu_{Z,\tau})
    =
    R_q(\tau)+\frac{d-m}{\tau^2}.
\]
Since $\alpha(D-m)=d-m$, the normal-direction terms cancel:
\begin{align}
    \mathcal{D}_{\mathcal{R}}^{\mathrm{sc}}(\tau)
    &=
    \left|
        1-
        \alpha
        \frac{
            R_p(\tau)+(D-m)/\tau^2
        }{
            R_q(\tau)+(d-m)/\tau^2
        }
    \right|
    \nonumber\\
    &=
    \frac{
        |R_q(\tau)-\alpha R_p(\tau)|
    }{
        R_q(\tau)+(d-m)/\tau^2
    }.
\label{eq:scaled_linear_exact}
\end{align}
The numerator satisfies
\[
    |R_q-\alpha R_p|
    \le
    |R_q-R_p|
    +(1-\alpha)R_p,
\]
where
\[
    1-\alpha
    =
    \frac{D-d}{D-m}.
\]
Since $d-m\ge1$,
\[
    R_q(\tau)+\frac{d-m}{\tau^2}
    \ge
    \frac{1}{\tau^2}.
\]
Using \eqref{eq:linear_tangential_absolute} and
\eqref{eq:linear_tangential_upper}, we obtain
\[
    \mathcal{D}_{\mathcal{R}}^{\mathrm{sc}}(\tau)
    \le
    C_m^{(1)}\delta
    +
    K_m\frac{D-d}{D-m}.
\]
After enlarging $C_m$ and using $d-m\ge1$, this implies
\[
    \mathcal{D}_{\mathcal{R}}^{\mathrm{sc}}(\tau)
    \le
    C_m
    \left(
        \frac{D-d}{D-m}
        +(d-m)\delta
    \right).
\]
\qed

\subsection{Diffusion-Scale Localization and Proof of Theorem~\ref{thm:FIR_stability_nonlinear} (Localized Nonlinear Stability)}
\label{app:proof_thm_3_6}

We first give a canonical construction of a localized prior at the
diffusion scale. This construction makes precise the localization used
in the main text, although the theorem below applies to any probability
measure satisfying the same support condition.

\begin{definition}[Truncated local prior at diffusion scale]
\label{def:truncated_local_prior}
Let $\mathcal M\subset\mathbb R^D$ be an $m$-dimensional embedded
manifold, let $\mu$ be a probability measure supported on
$\mathcal M$, and fix
\[
    x_0\in\operatorname{supp}\mu,
    \qquad
    r>0,
    \qquad
    \tau>0.
\]
Define the unnormalized localization kernel
\[
    K_\tau(x\mid x_0)
    :=
    \exp
    \left(
        -\frac{\|x-x_0\|^2}{2\tau}
    \right)
\]
and the truncated localized prior
\[
d\lambda_{x_0,\tau}^{(r)}(x)
=
\frac{
    \mathbf 1_{\{\|x-x_0\|\le r\sqrt{\tau}\}}
    K_\tau(x\mid x_0)\,d\mu(x)
}{
    \displaystyle
    \int_{\{\|u-x_0\|\le r\sqrt{\tau}\}}
    K_\tau(u\mid x_0)\,d\mu(u)
}.
\]
The denominator is positive because
$x_0\in\operatorname{supp}\mu$.

When this prior is used to define a Gaussian-smoothed local law, the
prior is held fixed while differentiating with respect to the
smoothing variance. More precisely, for any measurable map
\[
    T:\mathcal M\to\mathbb R^k,
\]
introduce an auxiliary smoothing variance $s$ and define
\[
    P_s
    :=
    \bigl(
        T_\#\lambda_{x_0,\tau}^{(r)}
    \bigr)
    *
    \gamma_s^k.
\]
The FIR of the localized law is evaluated at $s=\tau$, with the
derivative acting only on the Gaussian smoothing parameter $s$ and
not on the $\tau$-dependence of
$\lambda_{x_0,\tau}^{(r)}$.
\end{definition}

The prior in Definition~\ref{def:truncated_local_prior} is one
canonical localization at the diffusion scale. The theorem below is
deliberately more general: it applies to any probability measure
$\lambda_{x_0,\tau}^{(r)}$ supported on
\[
    \mathcal M\cap B(x_0,r\sqrt{\tau}),
\]
whether or not it is obtained from the truncated-kernel construction.
Comparing the truncated prior with an untruncated localization would
require additional lower-mass or doubling assumptions on $\mu$ and is
not needed below.

\begin{theorem}[Local nonlinear and curved stability of FIR]
\label{thm:FIR_stability_nonlinear}

Let $\mathcal M\subset\mathbb R^D$ be a $C^2$,
$m$-dimensional embedded manifold with positive reach
\[
    \rho_{\mathcal M}
    :=
    \reach(\mathcal M)
    >
    0.
\]
Fix $x_0\in\mathcal M$, $r>0$, and $\tau>0$, 
and let $\lambda_{x_0,\tau}^{(r)}$ be a probability measure supported
on
\(
    \mathcal M\cap B(x_0,r\sqrt{\tau}).
\)
Let $E:\mathbb R^D\to\mathbb R^d$ be $C^2$ on a neighborhood
containing all line segments joining $x_0$ to points in this set.
Define
\(
    A_{x_0}
    :=
    DE(x_0)\big|_{T_{x_0}\mathcal M},
\)
and the tangent-affine approximation
\(
    F_{x_0}(x)
    :=
    E(x_0)
    +
    A_{x_0}P_{T_{x_0}\mathcal M}(x-x_0).
\)
Whenever $A_{x_0}^\top A_{x_0}$ is written, we represent
$A_{x_0}$ in an arbitrary orthonormal basis of
$T_{x_0}\mathcal M$; the resulting operator norm is
basis-independent. 
Define the local tangential distortion
\(
    \delta_{x_0}
    :=
    \|A_{x_0}^{\top}A_{x_0}-I_m\|_{\mathrm{op}},
\)
and assume
\(
    \delta_{x_0}\le\delta_*<1.
\)
Let
\(
    P_{x_0,\tau}
    :=
    (E_\#\lambda_{x_0,\tau}^{(r)})*\gamma_\tau,
\) 
\(
    Q_{x_0,\tau}
    :=
    ((F_{x_0})_\#\lambda_{x_0,\tau}^{(r)})*\gamma_\tau.
\)
Define the local affine residual
\(
    \epsilon_{x_0,\tau}
    :=
    \sup_{x\in\operatorname{supp}\lambda_{x_0,\tau}^{(r)}}
    \|E(x)-F_{x_0}(x)\|.
\)
If
\(
    \epsilon_{x_0,\tau}\le a_0\sqrt{\tau},
\)
then there exists
\(
    C=C(d,r,\delta_*,a_0)
\)
such that
\(
    \tau^2
    \left|
        \mathcal R(P_{x_0,\tau})
        -
        \mathcal R(Q_{x_0,\tau})
    \right|
    \le
    C\frac{\epsilon_{x_0,\tau}}{\sqrt{\tau}}.
\)
Furthermore,
\(
    \left|
        1-
        \frac{
            \mathcal R(P_{x_0,\tau})
        }{
            \mathcal R(Q_{x_0,\tau})
        }
    \right|
    \le
    \frac{C}{b_{x_0}}
    \frac{\epsilon_{x_0,\tau}}{\sqrt{\tau}},
\)
where one may take
\(
    b_{x_0}
    :=
    \max
    \left\{
        d-m,\,
        \frac{
            d
        }{
            \left(1+r^2(1+\delta_*)\right)^2
        }
    \right\}.
\)
Moreover, suppose
\(
    \|DE(x_0)\|_{\mathrm{op}}\le J
\)
and
\(
    \sup_{y\in\mathcal U_{x_0,\tau}}
    \|D^2E(y)\|_{\mathrm{op}}
    \le H,
\)
where
\(    \|D^2E(y)\|_{\mathrm{op}}
    :=
    \sup_{\|u\|=\|v\|=1}
    \|D^2E(y)[u,v]\|,
\)
and $\mathcal U_{x_0,\tau}$ contains every segment $[x_0,x]$
for
$x\in\operatorname{supp}\lambda_{x_0,\tau}^{(r)}$.
Then
\(
    \epsilon_{x_0,\tau}
    \le
    \frac{r^2\tau}{2}
    \left(
        H+\frac{J}{\rho_{\mathcal M}}
    \right).
\)
Consequently, whenever
\(
    \frac{r^2\sqrt{\tau}}{2}
    \left(
        H+\frac{J}{\rho_{\mathcal M}}
    \right)
    \le a_0,
\)
the preceding hypotheses are satisfied and
\[    \left|
        1-
        \frac{
            \mathcal R(P_{x_0,\tau})
        }{
            \mathcal R(Q_{x_0,\tau})
        }
    \right|
    \le
    \frac{Cr^2}{2b_{x_0}}
    \sqrt{\tau}
    \left(
        H+\frac{J}{\rho_{\mathcal M}}
    \right).
\]
\end{theorem}

To prove Theorem~\ref{thm:FIR_stability_nonlinear}, we first establish the following two lemmas:
\begin{lemma}[Gaussian-smoothed FIR stability under bounded perturbations]
\label{lem:abstract_FIR_stability}
For a probability measure $\lambda$ on a measurable space $\mathcal X$, and measurable maps $G,F:\mathcal X\to\mathbb R^d$, define
$$
P_\tau=(G_\#\lambda)*\gamma_\tau, \qquad Q_\tau=(F_\#\lambda)*\gamma_\tau.
$$
Explicitly define the functional:
$$
\mathcal R_\tau(\nu) := \int \left\| \nabla^2\log(\nu*\gamma_\tau)(z) \right\|_F^2 \,d(\nu*\gamma_\tau)(z).
$$
Assume that for some $c\in\mathbb R^d$, $L\ge0$, $a_0>0$, and $\epsilon\ge0$,
$$
\sup_{x\in\operatorname{supp}\lambda}\|F(x)-c\| \le L\sqrt{\tau},
$$
and
$$
\sup_{x\in\operatorname{supp}\lambda}\|G(x)-F(x)\| \le \epsilon, \qquad \frac{\epsilon}{\sqrt{\tau}}\le a_0.
$$
Then there exists a constant $C(d,L,a_0)$ such that
$$
\tau^2 \left| \mathcal R_\tau(G_\#\lambda) - \mathcal R_\tau(F_\#\lambda) \right| \le C(d,L,a_0)\frac{\epsilon}{\sqrt{\tau}}.
$$
\end{lemma}

\begin{lemma}[Universal FIR lower bound]
\label{lem:universal_FIR_lower_bound}
Under the assumptions of Lemma~\ref{lem:abstract_FIR_stability}, we have:
$$
\tau^2\mathcal R(Q_\tau) \ge \frac{d}{(1+L^2)^2}.
$$
If $F(\operatorname{supp}\lambda)$ lies in an affine $m$-plane with $m<d$, then we also have:
$$
\tau^2\mathcal R(Q_\tau) \ge d-m.
$$
Define $b(d,m,L) := \max\left\{ d-m,\, \frac{d}{(1+L^2)^2} \right\}$, where the first term is interpreted as zero when $m=d$. Then:
$$
\left| 1-\frac{\mathcal R(P_\tau)}{\mathcal R(Q_\tau)} \right| \le \frac{C(d,L,a_0)}{b(d,m,L)} \frac{\epsilon}{\sqrt{\tau}}.
$$
\end{lemma}

\begin{proof}[Proof of Lemma \ref{lem:abstract_FIR_stability}]
We prove the lemma in five steps.

\textbf{Step 1: Remove $\tau$ by scaling.} \\
Set
\[
    U=\frac{G(X)-c}{\sqrt{\tau}},
    \qquad
    V=\frac{F(X)-c}{\sqrt{\tau}},
    \qquad
    a=\frac{\epsilon}{\sqrt{\tau}}.
\]
By the assumptions of the lemma,
\[
    \|V\|\le L,
    \qquad
    \|U-V\|\le a\le a_0,
    \qquad
    \|U\|\le M:=L+a_0.
\]
Let $\varphi$ denote the standard $d$-dimensional Gaussian density. Define the scaled densities:
$$ p(z)=\int\varphi(z-U(x))\,d\lambda(x), \qquad q(z)=\int\varphi(z-V(x))\,d\lambda(x). $$
Because $\mathcal{R}(P_\tau)=\tau^{-2}\mathcal{R}(p)$ and $\mathcal{R}(Q_\tau)=\tau^{-2}\mathcal{R}(q)$, it is sufficient to prove:
\begin{equation}
|\mathcal{R}(p)-\mathcal{R}(q)|\le C(d,L,a_0)a. \label{eq:scaled_R_bound}
\end{equation}

\textbf{Step 2: Joint Gaussian Channels and Total Variation.} \\
Define the joint distributions on $\mathcal{X}\times\mathbb{R}^d$:
$$ d\mathbf{P}(x,z) = d\lambda(x)\,\varphi(z-U(x))\,dz, \qquad d\mathbf{Q}(x,z) = d\lambda(x)\,\varphi(z-V(x))\,dz. $$
Because the conditional Gaussians have equal covariance, the KL divergence is:
$$ D(\mathbf{P}\Vert\mathbf{Q}) = \frac{1}{2} \mathbb{E}_\lambda\|U-V\|^2 \le \frac{a^2}{2}. $$
By the chain rule for relative entropy, $D(\mathbf{P}\Vert\mathbf{Q}) = D(p\Vert q) + \mathbb{E}_{Z\sim p} D(\pi_p^Z\Vert\pi_q^Z)$, which implies $D(p\Vert q)\le \frac{a^2}{2}$ and $\mathbb{E}_{Z\sim p} D(\pi_p^Z\Vert\pi_q^Z) \le \frac{a^2}{2}$. Pinsker's inequality subsequently gives:
\begin{equation*}
\operatorname{TV}(p,q)\le\frac{a}{2}, \qquad \mathbb{E}_{Z\sim p} \operatorname{TV}(\pi_p^Z,\pi_q^Z)^2 \le\frac{a^2}{4}. 
\end{equation*}

\textbf{Step 3: Posterior Covariance Stability.} \\
Let $C_p(z)=\operatorname{Cov}_{\pi_p^z}(U)$ and $C_q(z)=\operatorname{Cov}_{\pi_q^z}(V)$. We decompose the difference:
$$ \|C_p(z)-C_q(z)\|_F \le \left\| \operatorname{Cov}_{\pi_p^z}(U) - \operatorname{Cov}_{\pi_p^z}(V) \right\|_F + \left\| \operatorname{Cov}_{\pi_p^z}(V) - \operatorname{Cov}_{\pi_q^z}(V) \right\|_F. $$
The first term is bounded by $2a(M+L)$ since $\|U-V\|\le a$. For the second term, bounded support ($\|V\|\le L$) yields a bound of $6L^2\operatorname{TV}(\pi_p^z,\pi_q^z)$. Squaring and integrating under $p$ gives the Covariance Stability Lemma:
\begin{equation}
\mathbb{E}_{Z\sim p}\|C_p(Z)-C_q(Z)\|_F^2 \le C_0(L,a_0)a^2, \label{eq:cov_stability}
\end{equation}
where $C_0(L,a_0) = 8(2L+a_0)^2+18L^4$.

\textbf{Step 4: Hessian Bound via Tweedie's Identity.} \\
The second-order Tweedie identities give $H_p(z) := \nabla^2\log p(z) = -I_d+C_p(z)$ and $H_q(z) = -I_d+C_q(z)$. Thus, \eqref{eq:cov_stability} directly implies $\mathbb{E}_p\|H_p-H_q\|_F^2 \le C_0a^2$. Furthermore, $\|C_p(z)\|_F \le \operatorname{Tr}C_p(z) \le M^2$, so pointwise:
$$ \|H_p(z)\|_F \le K_p := \sqrt{d}+M^2, \qquad \|H_q(z)\|_F \le K_q := \sqrt{d}+L^2. $$

\textbf{Step 5: Density Change and Hessian Difference.} \\
We decompose the FIR difference:
$$ \mathcal{R}(p)-\mathcal{R}(q) = \int \left( \|H_p(z)\|_F^2-\|H_q(z)\|_F^2 \right)p(z)\,dz + \int \|H_q(z)\|_F^2 \bigl(p(z)-q(z)\bigr)\,dz. $$
Using Cauchy-Schwarz and the bounds from Step 4, the first term is bounded by $\sqrt{C_0}(K_p+K_q)a$. The change-of-density term is bounded by $2K_q^2\operatorname{TV}(p,q) \le K_q^2a$. Combining these proves \eqref{eq:scaled_R_bound}:
$$ |\mathcal{R}(p)-\mathcal{R}(q)| \le \left[ \sqrt{C_0}(K_p+K_q)+K_q^2 \right]a. $$
Scaling back proves the absolute FIR bound stated in Lemma~\ref{lem:abstract_FIR_stability}.

\end{proof}

\begin{proof}[Proof of Lemma \ref{lem:universal_FIR_lower_bound}]
Let $Z = F(X) + \sqrt{\tau}\, G$, where $\|F(X)-c\| \le L\sqrt{\tau}$.
After scaling to unit noise, define $W = \frac{Z-c}{\sqrt{\tau}} = V + G$, where $\|V\| \le L$.
Therefore, the covariance satisfies:
$$
\operatorname{Cov}(W) = I + \operatorname{Cov}(V) \preceq (1+L^2)I.
$$
Let
\[
    J(q)
    :=
    \mathbb E_q[s_q(W)s_q(W)^\top]
\]
be the Fisher information matrix of $W$. Because $q$ is a bounded
Gaussian location mixture, $q$ and its derivatives have Gaussian
tails, so the following integrations by parts have no boundary
terms.

Writing $m_q:=\mathbb E_q[W]$, integration by parts gives
\[
    \mathbb E_q
    \left[
        (W-m_q)s_q(W)^\top
    \right]
    =
    -I_d.
\]
The covariance matrix of the joint random vector
\[
    \bigl(W-m_q,s_q(W)\bigr)
\]
is positive semidefinite. Taking its Schur complement yields the
matrix information inequality
\[
    J(q)
    \succeq
    \operatorname{Cov}(W)^{-1}
    \succeq
    \frac{1}{1+L^2}I_d.
\]

A second integration by parts gives
\[
    \mathbb E_q[\nabla^2\log q(W)]
    =
    -J(q).
\]
Hence, Jensen's inequality gives
\[
\begin{aligned}
\mathcal R(q)
&=
\mathbb E_q
\|\nabla^2\log q(W)\|_F^2
\\
&\ge
\left\|
    \mathbb E_q[\nabla^2\log q(W)]
\right\|_F^2
\\
&=
\|J(q)\|_F^2
\\
&\ge
\frac{d}{(1+L^2)^2}.
\end{aligned}
\]

Scaling back to the variance $\tau$ yields:
$$
\tau^2\mathcal R(Q_\tau) \ge \frac{d}{(1+L^2)^2}.
$$
If $F(\operatorname{supp}\lambda)$ lies in an affine $m$-plane with $m<d$, the orthogonal noise components contribute exactly $d-m$, yielding the secondary bound:
$$
\tau^2\mathcal R(Q_\tau) \ge d-m.
$$
Combining these lower bounds with the absolute difference established in Lemma \ref{lem:abstract_FIR_stability} yields the stated relative bound.
\end{proof}

\begin{proof}[Proof of
Theorem~\ref{thm:FIR_stability_nonlinear}]

Apply Lemma~\ref{lem:abstract_FIR_stability} with
\[
    G(x)=E(x),
    \qquad
    F(x)=F_{x_0}(x),
    \qquad
    c=E(x_0),
\]
and common prior $\lambda_{x_0,\tau}^{(r)}$.

Since
\[
    \|A_{x_0}\|_{\mathrm{op}}^2
    \le 1+\delta_*,
\]
for every
$x\in\operatorname{supp}\lambda_{x_0,\tau}^{(r)}$ we have
\[
\|F_{x_0}(x)-E(x_0)\|
=
\|A_{x_0}P_{T_{x_0}\mathcal M}(x-x_0)\|
\le
r\sqrt{1+\delta_*}\sqrt{\tau}.
\]
Thus the bounded-image assumption in
Lemma~\ref{lem:abstract_FIR_stability} holds with
\[
    L=r\sqrt{1+\delta_*}.
\]
The perturbation assumption is exactly
\[
    \sup_{\operatorname{supp}\lambda_{x_0,\tau}^{(r)}}
    \|E(x)-F_{x_0}(x)\|
    =
    \epsilon_{x_0,\tau}.
\]
The absolute FIR estimate therefore follows from
Lemma~\ref{lem:abstract_FIR_stability}.

Moreover, the image of $F_{x_0}$ lies in the affine space
\[
    E(x_0)+\operatorname{range}(A_{x_0}),
\]
whose dimension is $m$ because $\delta_*<1$ implies that
$A_{x_0}$ is injective on $T_{x_0}\mathcal M$.
Combining the two lower bounds in
Lemma~\ref{lem:universal_FIR_lower_bound} gives
\[
    \tau^2\mathcal R(Q_{x_0,\tau})
    \ge
    b_{x_0},
\]
where
\[
    b_{x_0}
    =
    \max
    \left\{
        d-m,\,
        \frac{
            d
        }{
            \left(1+r^2(1+\delta_*)\right)^2
        }
    \right\}.
\]
Dividing the absolute bound by this lower bound proves the relative
estimate.

It remains to derive the explicit residual estimate. For every
$x\in\operatorname{supp}\lambda_{x_0,\tau}^{(r)}$,
\begin{align*}
E(x)-F_{x_0}(x)
=
E(x)-E(x_0)-DE(x_0)(x-x_0)
+
DE(x_0)
(I-P_{T_{x_0}\mathcal M})(x-x_0).
\end{align*}
The second-order Taylor theorem gives
\[
    \|E(x)-E(x_0)-DE(x_0)(x-x_0)\|
    \le
    \frac{H}{2}\|x-x_0\|^2.
\]
Positive reach gives
\[
    \|
        (I-P_{T_{x_0}\mathcal M})(x-x_0)
    \|
    \le
    \frac{\|x-x_0\|^2}{2\rho_{\mathcal M}}.
\]
Consequently,
\[
\begin{aligned}
\|E(x)-F_{x_0}(x)\|
&\le
\frac12
\left(
    H+\frac{J}{\rho_{\mathcal M}}
\right)
\|x-x_0\|^2
\\
&\le
\frac{r^2\tau}{2}
\left(
    H+\frac{J}{\rho_{\mathcal M}}
\right).
\end{aligned}
\]
Taking the supremum proves the claimed residual bound and completes
the proof.
\end{proof}

\subsection{Proof of Corollary~\ref{cor:combined_local_bound}}

\begin{corollary}[Combined local data-versus-latent bound]
\label{cor:combined_local_bound}

Under the setup of
Theorem~\ref{thm:FIR_stability_nonlinear}, suppose
\(
    \delta_{x_0}
    =
    \|A_{x_0}^{\top}A_{x_0}-I_m\|_{\mathrm{op}}
    \le\delta_*<1,
\)
and assume the local derivative bounds
\(
    \|DE(x_0)\|_{\mathrm{op}}\le J,
    \qquad
    \sup_{y\in\mathcal U_{x_0,\tau}}
    \|D^2E(y)\|_{\mathrm{op}}\le H.
\)
Fix constants $a_X,a_Z>0$ and suppose
\(
    \frac{r^2\sqrt{\tau}}{2\rho_{\mathcal M}}
    \le
    a_X,
\)
\(
    \frac{r^2\sqrt{\tau}}{2}
    \left(
        H+\frac{J}{\rho_{\mathcal M}}
    \right)
    \le a_Z.
\)
Define the maps
\begin{align*}
G_X(x)
&:=
x-x_0,
&
F_X(x)
&:=
P_{T_{x_0}\mathcal M}(x-x_0),
\\
F_Z(x)
&:=
A_{x_0}P_{T_{x_0}\mathcal M}(x-x_0),
&
G_Z(x)
&:=
E(x)-E(x_0).
\end{align*}
Define the corresponding local Gaussian-smoothed laws by
\begin{align*}
P^X_{x_0,\tau}
&:=
(G_X{}_\#\lambda_{x_0,\tau}^{(r)})*\gamma_\tau^D,
&
Q^X_{x_0,\tau}
&:=
(F_X{}_\#\lambda_{x_0,\tau}^{(r)})*\gamma_\tau^D,
\\
Q^Z_{x_0,\tau}
&:=
(F_Z{}_\#\lambda_{x_0,\tau}^{(r)})*\gamma_\tau^d,
&
P^Z_{x_0,\tau}
&:=
(G_Z{}_\#\lambda_{x_0,\tau}^{(r)})*\gamma_\tau^d.
\end{align*}
Then there exist constants $C_A$, $C_E$, $C_M>0$, 
depending only on
$D$, $d$, $m$, $r$, $\delta_*$, $a_X$, $a_Z$,
such that
\[
\left|
    \tau^2
    \left[
        \mathcal R(P^X_{x_0,\tau})
        -
        \mathcal R(P^Z_{x_0,\tau})
    \right]
    -(D-d)
\right|
\le
C_A\delta_{x_0}
+
C_Er^2\sqrt{\tau}\,H
+
C_Mr^2\sqrt{\tau}
\frac{1+J}{\rho_{\mathcal M}}.
\]
\end{corollary}

\begin{proof}
For brevity, write
\[
    \lambda:=\lambda_{x_0,\tau}^{(r)}
\]
and let
\[
    P_T:=P_{T_{x_0}\mathcal M}.
\]

We decompose
\begin{align}
&
\tau^2
\left[
    \mathcal R(P^X_{x_0,\tau})
    -
    \mathcal R(P^Z_{x_0,\tau})
\right]
-(D-d)
\nonumber\\
={}&
\tau^2
\left[
    \mathcal R(P^X_{x_0,\tau})
    -
    \mathcal R(Q^X_{x_0,\tau})
\right]
\label{eq:cor_decomposition}\\
&+
\left\{
    \tau^2
    \left[
        \mathcal R(Q^X_{x_0,\tau})
        -
        \mathcal R(Q^Z_{x_0,\tau})
    \right]
    -(D-d)
\right\}
\nonumber\\
&+
\tau^2
\left[
    \mathcal R(Q^Z_{x_0,\tau})
    -
    \mathcal R(P^Z_{x_0,\tau})
\right].
\nonumber
\end{align}

\paragraph{Step 1: Manifold versus tangent plane.}
Apply Lemma~\ref{lem:abstract_FIR_stability} in dimension $D$ to
\[
    G_X(x)=x-x_0,
    \qquad
    F_X(x)=P_T(x-x_0).
\]
On the support of $\lambda$,
\[
    \|F_X(x)\|\le r\sqrt{\tau}.
\]
The positive-reach estimate gives
\[
\begin{aligned}
\|G_X(x)-F_X(x)\|
&=
\|(I-P_T)(x-x_0)\|
\\
&\le
\frac{\|x-x_0\|^2}{2\rho_{\mathcal M}}
\\
&\le
\frac{r^2\tau}{2\rho_{\mathcal M}}.
\end{aligned}
\]
Since
\[
    \frac{r^2\sqrt{\tau}}{2\rho_{\mathcal M}}\le a_X,
\]
Lemma~\ref{lem:abstract_FIR_stability} yields
\begin{equation}
\label{eq:cor_manifold_tangent}
    \tau^2
    \left|
        \mathcal R(P^X_{x_0,\tau})
        -
        \mathcal R(Q^X_{x_0,\tau})
    \right|
    \le
    C_X
    \frac{r^2\sqrt{\tau}}{2\rho_{\mathcal M}},
\end{equation}
where $C_X$ depends only on $D,r,a_X$.

\paragraph{Step 2: Tangential linear distortion.}
Choose an isometry
\[
    U_0:\mathbb R^m\longrightarrow T_{x_0}\mathcal M
\]
and define tangent coordinates
\[
    Y(x)
    :=
    U_0^\top P_T(x-x_0).
\]
Then
\[
    \|Y(x)\|\le r\sqrt{\tau}.
\]
Let
\[
    \overline A:=A_{x_0}U_0.
\]
Its polar decomposition has the form
\[
    \overline A=VS,
\]
where $V:\mathbb R^m\to\mathbb R^d$ is an isometry and
\[
    S=(\overline A^\top\overline A)^{1/2}.
\]
Since
\[
    \|\overline A^\top\overline A-I_m\|_{\mathrm{op}}
    =
    \delta_{x_0}
    \le\delta_*,
\]
the eigenvalues of $S$ belong to
\[
    [\sqrt{1-\delta_*},\sqrt{1+\delta_*}].
\]
Therefore,
\[
    \|S-I_m\|_{\mathrm{op}}
    \le
    c_*\delta_{x_0},
    \qquad
    c_*:=
    \frac{1}{1+\sqrt{1-\delta_*}}.
\]

Let
\[
    \eta:=Y_\#\lambda.
\]
By Lemma~\ref{lem:normal_term},
\[
    \mathcal R(Q^X_{x_0,\tau})
    =
    \mathcal R^{(m)}(\eta_\tau)
    +
    \frac{D-m}{\tau^2},
\]
whereas orthogonal invariance and the same normal decomposition give
\[
    \mathcal R(Q^Z_{x_0,\tau})
    =
    \mathcal R^{(m)}((S_\#\eta)_\tau)
    +
    \frac{d-m}{\tau^2}.
\]
Hence
\begin{align*}
&
\tau^2
\left[
    \mathcal R(Q^X_{x_0,\tau})
    -
    \mathcal R(Q^Z_{x_0,\tau})
\right]
-(D-d)
\\
&\qquad=
\tau^2
\left[
    \mathcal R^{(m)}(\eta_\tau)
    -
    \mathcal R^{(m)}((S_\#\eta)_\tau)
\right].
\end{align*}

For $y\in\operatorname{supp}\eta$,
\[
    \|Sy\|
    \le
    r\sqrt{1+\delta_*}\sqrt{\tau},
\]
and
\[
    \|Sy-y\|
    \le
    c_*r\delta_{x_0}\sqrt{\tau}.
\]
Applying Lemma~\ref{lem:abstract_FIR_stability} in dimension $m$
to the maps $y\mapsto y$ and $y\mapsto Sy$ gives
\begin{equation}
\label{eq:cor_tangential}
\left|
    \tau^2
    \left[
        \mathcal R(Q^X_{x_0,\tau})
        -
        \mathcal R(Q^Z_{x_0,\tau})
    \right]
    -(D-d)
\right|
\le
C_A\delta_{x_0},
\end{equation}
where $C_A$ depends only on $m,r,\delta_*$.

\paragraph{Step 3: Affine versus nonlinear encoding.}
Apply Lemma~\ref{lem:abstract_FIR_stability} in dimension $d$ to
\[
    G_Z(x)=E(x)-E(x_0),
    \qquad
    F_Z(x)=A_{x_0}P_T(x-x_0).
\]
We have
\[
    \|F_Z(x)\|
    \le
    r\sqrt{1+\delta_*}\sqrt{\tau}.
\]
The residual estimate in
Theorem~\ref{thm:FIR_stability_nonlinear} gives
\[
    \|G_Z(x)-F_Z(x)\|
    \le
    \frac{r^2\tau}{2}
    \left(
        H+\frac{J}{\rho_{\mathcal M}}
    \right).
\]
Since the second smallness assumption in the corollary holds,
Lemma~\ref{lem:abstract_FIR_stability} gives
\begin{equation}
\label{eq:cor_affine_nonlinear}
    \tau^2
    \left|
        \mathcal R(P^Z_{x_0,\tau})
        -
        \mathcal R(Q^Z_{x_0,\tau})
    \right|
    \le
    C_Z
    \frac{r^2\sqrt{\tau}}{2}
    \left(
        H+\frac{J}{\rho_{\mathcal M}}
    \right),
\end{equation}
where $C_Z$ depends only on $d,r,\delta_*,a_Z$.

Finally, apply the triangle inequality to
\eqref{eq:cor_decomposition} and use
\eqref{eq:cor_manifold_tangent},
\eqref{eq:cor_tangential}, and
\eqref{eq:cor_affine_nonlinear}. Absorbing numerical constants into
$C_A,C_E,C_M$ yields
\[
\begin{aligned}
&
\left|
    \tau^2
    \left[
        \mathcal R(P^X_{x_0,\tau})
        -
        \mathcal R(P^Z_{x_0,\tau})
    \right]
    -(D-d)
\right|
\\
&\qquad\le
C_A\delta_{x_0}
+
C_Er^2\sqrt{\tau}\,H
+
C_Mr^2\sqrt{\tau}
\frac{1+J}{\rho_{\mathcal M}}.
\end{aligned}
\]
\end{proof}


\subsection{Visual Interpretation of Theorem~\ref{thm:FIR_stability_nonlinear} and Corollary~\ref{cor:combined_local_bound}}

Note that Figure~\ref{fig:FIR_flat_encoder_geometry} demonstrates the encoder-induced distortion implied by Theorem~\ref{thm:FIR_stability_nonlinear} and Corollary~\ref{cor:combined_local_bound} discussed in the main text.

\begin{figure}[h!]
\centering
\resizebox{0.8\columnwidth}{!}{
\begin{tikzpicture}

\begin{scope}[yshift=4cm] 

\draw[thick] (-2,0) -- (2,0);
\node at (-2.3,0.9) {$M$};

\draw[dashed] (-0.8,-0.45) rectangle (0.8,0.45);
\node at (0,-0.95) { local neighborhood};

\draw[->,thick] (2.5,0) -- (4,0);
\node at (3.3,0.6) {$E$};

\draw[thick] (4.8,0) .. controls (6,0.35) and (7,-0.35) .. (8.2,0);

\draw[dashed] (6.1,-0.45) rectangle (7.5,0.45);
\node at (6.8,-0.95) { linear approx.\ region};

\node at (3,-1.7) { (a) Ideal encoder: local isometry and low curvature};

\end{scope}

\begin{scope} 

\draw[thick] (-2,0) -- (2,0);
\node at (-2.3,0.9) {$M$};

\draw[dashed] (-0.8,-0.45) rectangle (0.8,0.45);
\node at (0,-0.95) { local neighborhood};

\draw[->,thick] (2.5,0) -- (4,0);
\node at (3.3,0.6) {$E$};

\draw[thick] (4.8,0) -- (9.5,0);

\draw[dashed] (5.7,-0.18) rectangle (8.8,0.18);
\node at (7.4,-0.95) { distorted metric};

\node at (3,-1.7) { (b) Tangential Distortion: local stretching/compression with large $\delta$};

\end{scope}

\begin{scope}[yshift=-4cm] 

\draw[thick] (-2,0) -- (2,0);
\node at (-2.3,0.9) {$M$};

\draw[dashed] (-0.8,-0.45) rectangle (0.8,0.45);
\node at (0,-0.95) { local neighborhood};

\draw[->,thick] (2.5,0) -- (4,0);
\node at (3.3,0.6) {$E$};

\draw[dashed] (5.8,-0.21) rectangle (7.4,1.21);

\def\xL{4.8}
\def\xR{8.2}
\def\A{0.9} 

\draw[thick, domain=\xL:\xR, samples=150]
plot (\x, {\A*(1+cos(deg(3*pi*(\x-(\xL+\xR)/2)/(\xR-\xL))))/2});

\node at (6.8,-0.95) { linear approx.\ region};

\node at (3,-1.7) { (c)  Curvature Injection: high-frequency nonlinearity with large $\varepsilon$ and small $\tau$.};

\end{scope}

\end{tikzpicture}
}
\caption{Geometric interpretation of encoder-induced distortions.}
\label{fig:FIR_flat_encoder_geometry}
\end{figure}

\section{Estimation of FI and FIR}
\label{app:FI_estimation}

In this section, we describe how the Fisher Information (FI) and Fisher Information Rate (FIR) are computed in our experiments based on trained diffusion models. We specifically focus on the computational strategies that allow these metrics to scale to high-dimensional neural architectures.

\textbf{Fisher Information (FI).}
The FI can be expressed using the score function $\bs_\tau(\bx) = \nabla_{\bx} \log p_\tau(\bx)$. For a given noise variance $\tau > 0$ and dimension $D$, the FI is calculated as:
\begin{equation}\label{eqn:FI_estimation}
\begin{aligned}
\mathcal I^{(D)}(\mu_\tau)
&= \int_{\mathbb R^D} p_\tau(\bx)\,\bigl\| \nabla_{\bx} \log p_\tau(\bx) \bigr\|^2\,d\bx \\
&= \mathbb E_{\bx\sim p_\tau} \bigl[ \| \bs_\tau(\bx) \|^2 \bigr] \\
&\approx \frac{1}{N} \sum_{n=1}^N \bigl\| \bs_\tau(\bx^{(n)}) \bigr\|^2,
\end{aligned}
\end{equation}
where $\bx^{(n)}$ are samples drawn from the smoothed distribution $p_\tau$. In practice, these are obtained by adding Gaussian noise $\mathcal{N}(\bzero, \tau \bI)$ to clean data samples from the empirical distribution.

\textbf{Fisher Information Rate (FIR).}
The FIR is computed as the dissipation rate of the FI. Using the Hessian representation from Proposition~\ref{prop:FIR_identity}, we have:
\begin{align*}
\mathcal{R}^{(D)}(\mu_\tau)
&= -\frac{d}{d\tau}\,\mathcal I^{(D)}(\mu_\tau) \\
&= \mathbb E_{\bx\sim p_\tau} \Big[ \mathrm{Tr}\Big(\big(\nabla_{\bx} \bs_\tau(\bx)\big)^2\Big) \Big] \\
&\approx \frac{1}{N} \sum_{n=1}^N \mathrm{Tr}\Big(\big(\nabla_{\bx} \bs_\tau(\bx^{(n)})\big)^2\Big).
\end{align*}

For high-dimensional neural networks, explicitly forming the $D\times D$ score Jacobian requires $\mathcal O(D^2)$ storage and many derivative evaluations. We avoid forming the full Jacobian by using the \textbf{Hutchinson trace estimator} \citep{hutchinson1989stochastic}.  For the exact population score, 
\[     
\nabla_{\bx}\bs_\tau     =     \nabla_{\bx}^2\log p_\tau 
\] 
is a symmetric Hessian. Hence, 
\[     
\operatorname{Tr}     \left(         (\nabla_{\bx}\bs_\tau)^2     \right)     =     \|\nabla_{\bx}\bs_\tau\|_F^2, 
\] 
and the squared Frobenius norm can be written as an expectation over isotropic probe vectors $\bv$:
\begin{align*}
\mathrm{Tr}\!\big((\nabla_{\bx} \bs_\tau(\bx))^2\big)
&= \mathbb E_{\bv}\!\left[ \big\| \nabla_{\bx} \bs_\tau(\bx)\, \bv \big\|^2 \right] \\
&\approx \frac{1}{M} \sum_{m=1}^M \big\| \nabla_{\bx} \bs_\tau(\bx)\, \bv^{(m)} \big\|^2,
\end{align*}
where the probe vectors $\bv^{(m)}$ are drawn independently from a Rademacher distribution (entries in $\{-1, 1\}$). This leads to the empirical estimator:
\begin{equation}\label{eqn:FIR_estimation}
\mathcal{R}^{(D)}(\mu_\tau)
\approx \frac{1}{NM} \sum_{n=1}^N \sum_{m=1}^M \big\| \nabla_{\bx} \bs_\tau(\bx^{(n)})\, \bv^{(m)} \big\|^2.
\end{equation}

A learned score network need not have an exactly symmetric Jacobian.
Accordingly, the empirical estimator above should be viewed as a
plug-in estimator of the squared Frobenius norm of the learned score
Jacobian. It agrees with the population FIR functional when the
learned vector field equals the conservative population score.

\textbf{Implementation via Directional Derivatives.}
The term $\nabla_{\bx} \bs_\tau(\bx) \bv$ is a \textbf{Jacobian-vector product (JVP)}, representing the directional derivative of the score function. 
In modern automatic-differentiation frameworks
(e.g., \path{torch.autograd.functional.jvp} in PyTorch), this product
can be evaluated at a cost comparable to a small number of network
evaluations, without explicitly computing or storing the full
Jacobian matrix.

Alternatively, to avoid second-order automatic differentiation entirely or further reduce memory overhead, the term can be approximated via a first-order finite difference. For a small step size $\varepsilon > 0$:
\[
\nabla_{\bx} \bs_\tau(\bx)\,\bv \approx \frac{\bs_\tau(\bx + \varepsilon \bv) - \bs_\tau(\bx)}{\varepsilon}.
\]
This yields the gradient-free estimator:
\begin{equation}
\mathcal{R}^{(D)}(\mu_\tau) \approx \frac{1}{NM} \sum_{n=1}^N \sum_{m=1}^M \left\| \frac{\bs_\tau(\bx^{(n)} + \varepsilon \bv^{(m)}) - \bs_\tau(\bx^{(n)})}{\varepsilon} \right\|^2.
\label{eq:grad_free_FIR}
\end{equation}
In our experiments, the JVP approach was found to be highly accurate and numerically stable, providing a scalable means to monitor the geometric complexity penalty during the diffusion process.

\textbf{Score Function Parameterization in Experiments.}
In practice, the score function is obtained directly from the trained diffusion models, but the exact extraction depends on the model's training objective.

For the toy experiments described in \S\ref{sec:gaussian_toy} and \S\ref{sec:gaussian_deviation}, we use a noise-prediction diffusion model. The empirical score is approximated using the model's predicted noise $\boldsymbol{\epsilon}_{\theta}(\bx, \tau)$:
\[
\bs_\tau(\bx) = \nabla_{\bx}\log p_\tau(\bx) \approx -\frac{1}{\sqrt{\tau}}\, \boldsymbol{\epsilon}_{\theta}(\bx, \tau).
\]

For the experiments described in \S\ref{sec:exp_RAE}, JiT, SiT, RAE, RAEv2, and REPA-E employ flow-based generative models with a linear interpolation path

$$
\bx_t
=
\alpha_t\bx_0+\sigma_t\boldsymbol{\epsilon},
\qquad
\boldsymbol{\epsilon}\sim\mathcal{N}(\mathbf{0},\mathbf{I}),
$$

where $t$ denotes the flow time and $\sigma_t$ is the corresponding noise standard-deviation coefficient. For models that predict the clean sample $\hat{\bx}_{0,\theta}(\bx_t,t)$, we compute the score as

$$
\bs_t(\bx_t)
=
-\frac{
\bx_t-\alpha_t\hat{\bx}_{0,\theta}(\bx_t,t)
}{
\sigma_t^2
}.
$$

For models that instead predict the velocity field $\hat{\bv}_\theta(\bx_t,t)$, we use

$$
\bs_t(\bx_t)
=
\frac{
\alpha_t\hat{\bv}_\theta(\bx_t,t)-\dot{\alpha}_t\bx_t
}{
\sigma_t
\left(
\dot{\alpha}_t\sigma_t-\alpha_t\dot{\sigma}_t
\right)
},
$$

where $\dot{\alpha}_t=d\alpha_t/dt$ and
$\dot{\sigma}_t=d\sigma_t/dt$.

To express the flow models in the same variance parameterization used for the EDM experiments, we define

$$
\widetilde{\bx}_{\tau}
=
\frac{\bx_t}{\alpha_t},
\qquad
\tau
=
\left(\frac{\sigma_t}{\alpha_t}\right)^2.
$$

Then
$\widetilde{\bx}_{\tau}
=\bx_0+\sqrt{\tau}\boldsymbol{\epsilon}$,
and the corresponding score is

$$
\widetilde{\bs}_{\tau}
\bigl(\widetilde{\bx}_{\tau}\bigr)
=
\alpha_t\bs_t(\bx_t).
$$

For the experiments described in \S\ref{sec:exp_VAE} and \S\ref{sec:exp_NVAE}, we adopt the Elucidating Diffusion Models (EDM) framework \citep{karras2022elucidating}. Following EDM standard practices, we employ a denoiser parameterization $D_\theta(\bx, \tau)$ alongside their proposed preconditioning and noise scaling schedules. Unless otherwise specified, all core hyperparameters (e.g., network capacity, learning rate, and weighting schedules) are strictly borrowed from the default EDM configurations. Under this framework, the score is approximated directly via the denoiser:
\[
\bs_\tau(\bx) = \nabla_{\bx} \log p_\tau(\bx) \approx -\frac{\bx - D_\theta(\bx,\tau)}{\tau}.
\]

In all experiments, FI and FIR are computed using
Equation~\eqref{eqn:FI_estimation} and
Equation~\eqref{eqn:FIR_estimation}, respectively, with the required derivatives evaluated using automatic differentiation. For the large-scale models considered in
\S\ref{sec:exp_VAE}, \S\ref{sec:exp_NVAE}, and
\S\ref{sec:exp_RAE}, we use $N=200$ samples and $M=20$ Hutchinson probe vectors due to their computational and memory costs. For all other experiments, we use $N=1000$ samples and $M=50$ Hutchinson probe vectors.

\section{Experimental Details}
\label{app:exp_detail}

In this section, we provide the model architectures, training configurations, and sampling parameters used in our experiments. 
Each subsection corresponds to the experimental setting described in the main text.

\subsection{FI and FIR Comparison in the Gaussian Toy Setting}

For the experiment in \S\ref{sec:gaussian_toy}, we first describe the architecture of the tiny diffusion model used in this setting. 
The denoising network in the tiny diffusion model is implemented as an MLP (multilayer perceptron). We use a hidden dimension of \(128\) with \(3\) hidden layers, and apply sinusoidal embeddings of dimension \(128\) to both the input data and the diffusion time.

For all experiments, we use \(50{,}000\) training samples with batch size \(256\) and train for \(200\) epochs. The diffusion process uses \(100\) timesteps with a linear noise \(\beta\)-schedule. We use the AdamW optimizer with a learning rate of \(5 \times 10^{-4}\).

\subsection{FIR Deviation in the Gaussian Toy Setting}

This subsection provides the experimental details for the study in 
\S\ref{sec:gaussian_deviation}. The model and training configuration follows 
the Gaussian toy experiment in \S\ref{sec:gaussian_toy}. 

The only modification concerns the input embedding used in the tiny diffusion model. 
In \S\ref{sec:gaussian_toy}, sinusoidal embeddings were applied to the input 
coordinates. Since the experiments in this section consider data whose input 
dimension may exceed two, we instead use the identity embedding for all tiny 
diffusion models.

\subsection{FI and FIR Comparison across ImageNet Latent Diffusion Models}

This subsection provides the experimental details for the study described in 
\S\ref{sec:exp_RAE} on ImageNet data.

We evaluate RAE (Representation Autoencoder)~\citep{zheng2025diffusion}, RAEv2 (Improved Baselines with Representation Autoencoders)~\citep{singh2026improved}, SiT (Scalable Interpolant Transformer)~\citep{ma2024sit}, REPA-E (Representation Alignment for End-to-End VAE Tuning)~\citep{leng2025repa}, and JiT (Just image Transformer)~\citep{li2026back} using their official GitHub implementations and publicly available pretrained checkpoints.
No additional training or fine-tuning is performed. We choose these models because their latent-space or pixel-space diffusion models use closely related transformer-based architectures, and we select checkpoints with similar model scale trained on ImageNet at $256 \times 256$ resolution. 

Specifically, for JiT, we choose the Huge-scale checkpoint with patch size $16$ (JiT/$16$)\footnote{We use the official JiT implementation: \url{https://github.com/LTH14/JiT}.}. For RAE, we choose the checkpoint with a DINOv2-B representation encoder without alignment guidance, together with the corresponding DiT-based XL-scale latent diffusion model for sample generation\footnote{We use the official RAE implementation: \url{https://github.com/bytetriper/RAE/tree/main}.}. For RAEv2, we choose the checkpoint with a DINOv3-L encoder, where the multi-layer representation is obtained by summing the last $7$ layers\footnote{We use the official RAEv2 implementation: \url{https://github.com/nanovisionx/RAEv2/tree/main}.}. For SiT, we choose the XL-scale checkpoint with patch size $2$\footnote{We use the official SiT implementation: \url{https://github.com/willisma/sit}.}. For REPA-E, we choose the checkpoint with the SD-VAE tokenizer and a SiT-based XL-scale generator with patch size $2$\footnote{We use the official REPA-E implementation: \url{https://github.com/End2End-Diffusion/REPA-E}.}.

Among these methods, JiT performs generation directly in pixel space, while the other methods first encode images into a latent representation and perform generation using a latent diffusion model. For each method, we generate samples using the sampling parameters recommended by its official implementation. For evaluation, we use $50{,}000$ ImageNet validation samples as the reference set and generate $50{,}000$ samples from each model to compute the metrics.

For the experiments in \S\ref{sec:exp_RAE}, we compute the geometric
diagnostics introduced by \citet{zhong2026diffusing} using their
official implementation,\footnote{\url{https://github.com/KlingAIResearch/diffusing-right-space}}

\subsection{Computing Resources}
All experiments were run on NVIDIA A100-SXM4-80GB GPUs. Each GPU has 80GB of memory.

\section{Gaussian Toy Experiment: Derivations and Additional Results}
\label{app:plot_gaussian_toy}

\subsection{Derivation Relevant for the Experiment in \S\ref{sec:gaussian_toy}}

In this section, we provide derivations for the experiment in \S\ref{sec:gaussian_toy}. 
We consider samples $\bx \sim \mathcal{N}(\bzero,\bI_2)$ and construct encoded representations $\bz = E(\bx)$, where the encoder $E$ is applied pointwise. 
For $E$, we consider functions given by common activation functions, including ReLU and Leaky ReLU (with negative slope $\alpha$).

We distinguish the deterministic pushforward laws from the
Gaussian-smoothed laws used in the diffusion experiments.
Proposition~\ref{prop:fisher_bi_lip} concerns intrinsic FI under a
deterministic pushforward when the logarithmic volume Jacobian has the
required weak regularity. By contrast,
Lemma~\ref{lem:gaussian_smoothing_finite_FI} applies to every
Gaussian-smoothed latent law. Below, we derive the exact smoothed
density and a model-free FI formula for ReLU and Leaky ReLU under the
Gaussian toy distribution.

\subsubsection{Exact Gaussian-smoothed FI for ReLU and Leaky ReLU}
\label{app:gaussian_activation_smoothed_FI}

Let
\[
    \psi_\alpha(x)
    :=
    \begin{cases}
        x, & x>0,\\
        \alpha x, & x\le0,
    \end{cases}
    \qquad
    \alpha\ge0.
\]
Thus, $\alpha=0$ corresponds to ReLU, while $\alpha>0$ corresponds
to Leaky ReLU.

For $v>0$, write
\[
    \varphi_v(y)
    :=
    \frac{1}{\sqrt{2\pi v}}
    \exp\left(-\frac{y^2}{2v}\right),
\]
let $\varphi:=\varphi_1$, and let $\Phi$ denote the standard Gaussian
CDF.

\begin{proposition}[Exact smoothed activation density]
\label{prop:gaussian_activation_smoothed_FI}
Let $X,G\sim\mathcal N(0,1)$ be independent and define
\[
    Y_{\alpha,\tau}
    :=
    \psi_\alpha(X)+\sqrt{\tau}\,G,
    \qquad
    \tau>0.
\]
Then $Y_{\alpha,\tau}$ has density
\begin{equation}
\label{eq:activation_smoothed_density}
\begin{aligned}
p_{\alpha,\tau}(y)
={}&
\varphi_{1+\tau}(y)
\Phi\left(
    \frac{y}{\sqrt{\tau(1+\tau)}}
\right)
\\
&+
\varphi_{\alpha^2+\tau}(y)
\Phi\left(
    -\frac{\alpha y}
    {\sqrt{\tau(\alpha^2+\tau)}}
\right).
\end{aligned}
\end{equation}
For $\alpha=0$, the second term is interpreted as
\[
    \frac12\varphi_\tau(y).
\]

Define
\[
    a_\tau
    :=
    \frac{1}{\sqrt{\tau(1+\tau)}},
    \qquad
    b_{\alpha,\tau}
    :=
    \frac{\alpha}
    {\sqrt{\tau(\alpha^2+\tau)}}.
\]
The derivative of the density is
\begin{align}
p_{\alpha,\tau}'(y)
={}&
\varphi_{1+\tau}(y)
\left[
    -\frac{y}{1+\tau}\Phi(a_\tau y)
    +
    a_\tau\varphi(a_\tau y)
\right]
\nonumber\\
&+
\varphi_{\alpha^2+\tau}(y)
\left[
    -\frac{y}{\alpha^2+\tau}
    \Phi(-b_{\alpha,\tau}y)
    -
    b_{\alpha,\tau}
    \varphi(b_{\alpha,\tau}y)
\right].
\label{eq:activation_smoothed_density_derivative}
\end{align}

For the coordinatewise encoder
\[
    E_\alpha(x_1,x_2)
    =
    \bigl(
        \psi_\alpha(x_1),
        \psi_\alpha(x_2)
    \bigr)
\]
applied to
\[
    X\sim\mathcal N(\bzero,I_2),
\]
the two smoothed coordinates remain independent. Therefore, the exact
population FI is
\begin{equation}
\label{eq:activation_exact_smoothed_FI}
    \mathcal I_\alpha^{(2)}(\tau)
    =
    2
    \int_{\mathbb R}
    \frac{
        |p_{\alpha,\tau}'(y)|^2
    }{
        p_{\alpha,\tau}(y)
    }\,
    dy
    \le
    \frac{2}{\tau}.
\end{equation}
When $\alpha=1$, the encoder is the identity,
\[
    p_{1,\tau}(y)=\varphi_{1+\tau}(y),
\]
and hence
\[
    \mathcal I_1^{(2)}(\tau)
    =
    \frac{2}{1+\tau}.
\]
\end{proposition}

\begin{proof}
Split the defining integral according to the events $X>0$ and
$X\le0$. For the positive branch,
\[
\begin{aligned}
&
\int_0^\infty
\varphi(x)\varphi_\tau(y-x)\,dx
\\
&\qquad=
\varphi_{1+\tau}(y)
\mathbb P
\left(
    X>0
    \,\middle|\,
    X+\sqrt{\tau}G=y
\right)
\\
&\qquad=
\varphi_{1+\tau}(y)
\Phi\left(
    \frac{y}{\sqrt{\tau(1+\tau)}}
\right).
\end{aligned}
\]
Similarly, for the negative branch,
\[
\begin{aligned}
&
\int_{-\infty}^0
\varphi(x)\varphi_\tau(y-\alpha x)\,dx
\\
&\qquad=
\varphi_{\alpha^2+\tau}(y)
\Phi\left(
    -\frac{\alpha y}
    {\sqrt{\tau(\alpha^2+\tau)}}
\right).
\end{aligned}
\]
This proves \eqref{eq:activation_smoothed_density}. Differentiating
the displayed expression proves
\eqref{eq:activation_smoothed_density_derivative}.

The product formula for the two-dimensional density follows from the
independence of the coordinates. Its score separates coordinatewise,
giving \eqref{eq:activation_exact_smoothed_FI}. The upper bound is
also an immediate consequence of
Lemma~\ref{lem:gaussian_smoothing_finite_FI}.
\end{proof}

\subsection{Derivation Relevant for the Experiment in \S\ref{sec:gaussian_deviation}}

In this section, we provide derivations for the experiment in \S\ref{sec:gaussian_deviation}. 
We consider data $\by \sim \mathcal{N}(\bzero,\bI_2)$ on a 2-dimensional intrinsic manifold, embedded in $\mathbb{R}^D$ as $\bx=(y_1,y_2,0,\ldots,0)$, with latent representations $\bz=E(\bx)\in\mathbb{R}^d$.

\paragraph{Exact FIR formulas for the linear Gaussian experiments.}
The compact-support assumption of
Theorem~\ref{thm:FIR_stability_d_less_D_linear} does not apply
directly to the Gaussian law used in these toy experiments.
Fortunately, in this setting the data-space and latent-space FIRs
can be computed exactly.

More generally, let
\[
    \rho=\mathcal{N}(\bzero,\bI_m),
\]
and let $\lambda_1,\ldots,\lambda_m$ denote the eigenvalues of
$\bA^\top\bA$. Since the nonzero eigenvalues of
$\bA\bA^\top$ and $\bA^\top\bA$ agree, Gaussian smoothing gives
\begin{equation}
\label{eq:exact_gaussian_data_FIR}
    \mathcal{R}^{(D)}(\mu_\tau)
    =
    \frac{m}{(1+\tau)^2}
    +
    \frac{D-m}{\tau^2},
\end{equation}
while
\begin{equation}
\label{eq:exact_gaussian_latent_FIR}
    \mathcal{R}^{(d)}(\mu_{Z,\tau})
    =
    \sum_{i=1}^m
    \frac{1}{(\lambda_i+\tau)^2}
    +
    \frac{d-m}{\tau^2}.
\end{equation}

For the equal-dimensional experiment $D=d=m=2$ with
\[
    \bA
    =
    \operatorname{diag}
    \left(
        \sqrt{1+\delta_0},
        \sqrt{1-\delta_0}
    \right),
\]
the eigenvalues are
\[
    \lambda_1=1+\delta_0,
    \qquad
    \lambda_2=1-\delta_0.
\]
Writing $a_\tau:=1+\tau$, we obtain
\begin{align}
    \mathcal{D}_{\mathcal R}(\tau)
    &=
    1-
    \frac{
        2/a_\tau^2
    }{
        (a_\tau+\delta_0)^{-2}
        +(a_\tau-\delta_0)^{-2}
    }
    \nonumber\\
    &=
    \frac{
        \delta_0^2
        \left(3a_\tau^2-\delta_0^2\right)
    }{
        a_\tau^2
        \left(a_\tau^2+\delta_0^2\right)
    }
    \le
    \frac{3\delta_0^2}{(1+\tau)^2}
    \le
    3\delta_0^2.
\label{eq:exact_gaussian_delta_deviation}
\end{align}
Thus, the population FIR deviation is quadratic in $\delta_0$
near $\delta_0=0$, and in particular it also satisfies an
$\mathcal{O}(\delta_0)$ upper bound for $\delta_0\in[0,1]$.

For the dimension-varying experiment, $\bA=\bI_m$, and hence
\begin{equation}
\label{eq:exact_gaussian_dimension_deviation}
    \mathcal{D}_{\mathcal R}^{\mathrm{sc}}(\tau)
    =
    \frac{D-d}{D-m}
    \frac{
        m/(1+\tau)^2
    }{
        m/(1+\tau)^2+(d-m)/\tau^2
    }
    \le
    \frac{D-d}{D-m}.
\end{equation}

\textbf{Linear Encoder}

We now specialize the exact Gaussian formulas above to the linear
encoder used in the experiment and verify its near-isometry
parameter. Take $D=d=m=2$ and consider the encoder
\[
\bz=\bA\bx,
\qquad
\bA=\mathrm{diag}\bigl(\sqrt{1+\delta_0},\,\sqrt{1-\delta_0}\bigr).
\]
Then
\[
\bA^\top \bA-\bI_2
=
\mathrm{diag}(1+\delta_0,\;1-\delta_0)-\bI_2
=
\mathrm{diag}(\delta_0,\;-\delta_0).
\]
Therefore,
\[
\|\bA^\top \bA-\bI_2\|_2=\delta_0.
\]
Thus, we can take 
\[
\delta=\delta_0.
\]

\textbf{Dimension-Varying Encoder}

We next specialize
\eqref{eq:exact_gaussian_dimension_deviation} to the
dimension-varying encoder. Take $D=512$ and $d\ge m=2$.
For $\bx = (y_1, y_2, 0, \ldots, 0) \in \mathbb{R}^D$, the encoder maps
\[
\bz = E(\bx) = (y_1, y_2, 0, \ldots, 0) \in \mathbb{R}^d .
\]

Since the encoder acts as the identity on the intrinsic subspace, we may take
\[
\bA = \bI_m .
\]
Therefore
\[
\bA^\top \bA - \bI_m = \mathbf{0},
\qquad
\|\bA^\top \bA - \bI_m\|_2 = 0,
\]
and the deviation parameter in Theorem~\ref{thm:FIR_stability_d_less_D_linear} satisfies
\[
\delta = 0 .
\]

\textbf{Nonlinear Encoder}
Consider $D=d=3$. For $\bx = (y_1,y_2,0)$, consider the encoder
\[
E((x_1,x_2,0),\varepsilon_0)
=
\left(
\frac{\sin(\varepsilon_0 x_1)}{\varepsilon_0},
\;
x_2,
\;
\frac{1-\cos(\varepsilon_0 x_1)}{\varepsilon_0}
\right),
\qquad \varepsilon_0>0.
\]

We first derive the linear deviation parameter $\delta$ and the dimension penalty term in \Cref{thm:FIR_stability_nonlinear}.
Since $D=d$, the dimension-penalty term in \Cref{thm:FIR_stability_nonlinear} satisfies
\[
\frac{D-d}{D-m}=0.
\]

Next we derive the linear deviation parameter $\delta$ in \Cref{thm:FIR_stability_nonlinear}. Let $M \subset \mathbb{R}^3$ denote the plane
\[
M = \{(x_1,x_2,0) : x_1,x_2 \in \mathbb{R}\}.
\]
Then the distribution $\mu$ supported on $M$ is the Gaussian measure
\[
d\mu(x_1,x_2)
=
\frac{1}{2\pi} e^{-(x_1^2+x_2^2)/2} \, dx_1 dx_2 .
\]
Let $\bx_0=(a,b,0)$. With respect to the orthonormal tangent basis
$\{e_1,e_2\}$ of $T_{\bx_0}M$, the restricted Jacobian is
\[
A_{\bx_0}
=
DE(\bx_0)\big|_{T_{\bx_0}M}
=
\begin{pmatrix}
\cos(\varepsilon_0a) & 0 \\
0 & 1 \\
\sin(\varepsilon_0a) & 0
\end{pmatrix}
\in\mathbb R^{3\times2}.
\]
Therefore,
\[
A_{\bx_0}^{\top}A_{\bx_0}
=
\begin{pmatrix}
\cos^2(\varepsilon_0a)+\sin^2(\varepsilon_0a) & 0 \\
0 & 1
\end{pmatrix}
=
I_2.
\]
Consequently,
\[
    \delta_{\bx_0}
    =
    \|A_{\bx_0}^{\top}A_{\bx_0}-I_2\|_{\mathrm{op}}
    =
    0.
\]

We next compute the localized residual appearing in
Theorem~\ref{thm:FIR_stability_nonlinear}. Since the input manifold
$M$ is a plane,
\[
    \operatorname{reach}(M)=\infty.
\]
The only nonzero second derivative of $E$ is in the $e_1$ direction:
\[
    D^2E(x)[e_1,e_1]
    =
    \left(
        -\varepsilon_0\sin(\varepsilon_0x_1),
        0,
        \varepsilon_0\cos(\varepsilon_0x_1)
    \right).
\]
Therefore,
\[
    \|D^2E(x)\|_{\mathrm{op}}
    =
    \varepsilon_0
\]
for every $x\in M$.

For
\[
    x\in
    \operatorname{supp}\lambda_{\bx_0,\tau}^{(r)},
\]
we have
\[
    \|x-\bx_0\|\le r\sqrt{\tau}.
\]
The second-order Taylor theorem therefore gives
\[
\begin{aligned}
\|E(x)-F_{\bx_0}(x)\|
&\le
\frac{\varepsilon_0}{2}
\|x-\bx_0\|^2
\\
&\le
\frac{r^2\tau}{2}\varepsilon_0.
\end{aligned}
\]
Thus,
\[
    \epsilon_{\bx_0,\tau}
    \le
    \frac{r^2\tau}{2}\varepsilon_0,
\]
uniformly in the expansion point $\bx_0$.

Whenever
\[
    \frac{r^2\varepsilon_0\sqrt{\tau}}{2}
    \le a_0,
\]
Theorem~\ref{thm:FIR_stability_nonlinear} gives
\[
    \left|
        1-
        \frac{
            \mathcal R(P_{\bx_0,\tau})
        }{
            \mathcal R(Q_{\bx_0,\tau})
        }
    \right|
    \le
    C r^2\varepsilon_0\sqrt{\tau}.
\]
The experiment in
Figure~\ref{fig:gaussian_FIR_deviation}(c) evaluates a global FIR
deviation for an unbounded Gaussian distribution, so it illustrates
the dependence on $\varepsilon_0$ qualitatively but is not a direct
numerical verification of this localized bound.

\subsection{Supplementary Plots}

In \S\ref{sec:gaussian_toy},
Figure~\ref{fig:gaussian_activation}(a,b) compares FI and FIR for Gaussian
data encoded using several common activation functions, including
ReLU, Leaky ReLU, GELU, sigmoid, and tanh. Here, we provide an
extended comparison focusing on ReLU and Leaky ReLU over additional
values of the negative slope $\alpha$. The corresponding FI and FIR
curves as functions of the noise variance $\tau$ are shown in
Figure~\ref{fig:gaussian_relu_activation}.

\begin{figure}[t!bp]
\centering
\setcounter{subfigure}{2}
\begin{subfigure}{0.48\linewidth}
    \centering
    \includegraphics[width=\linewidth]{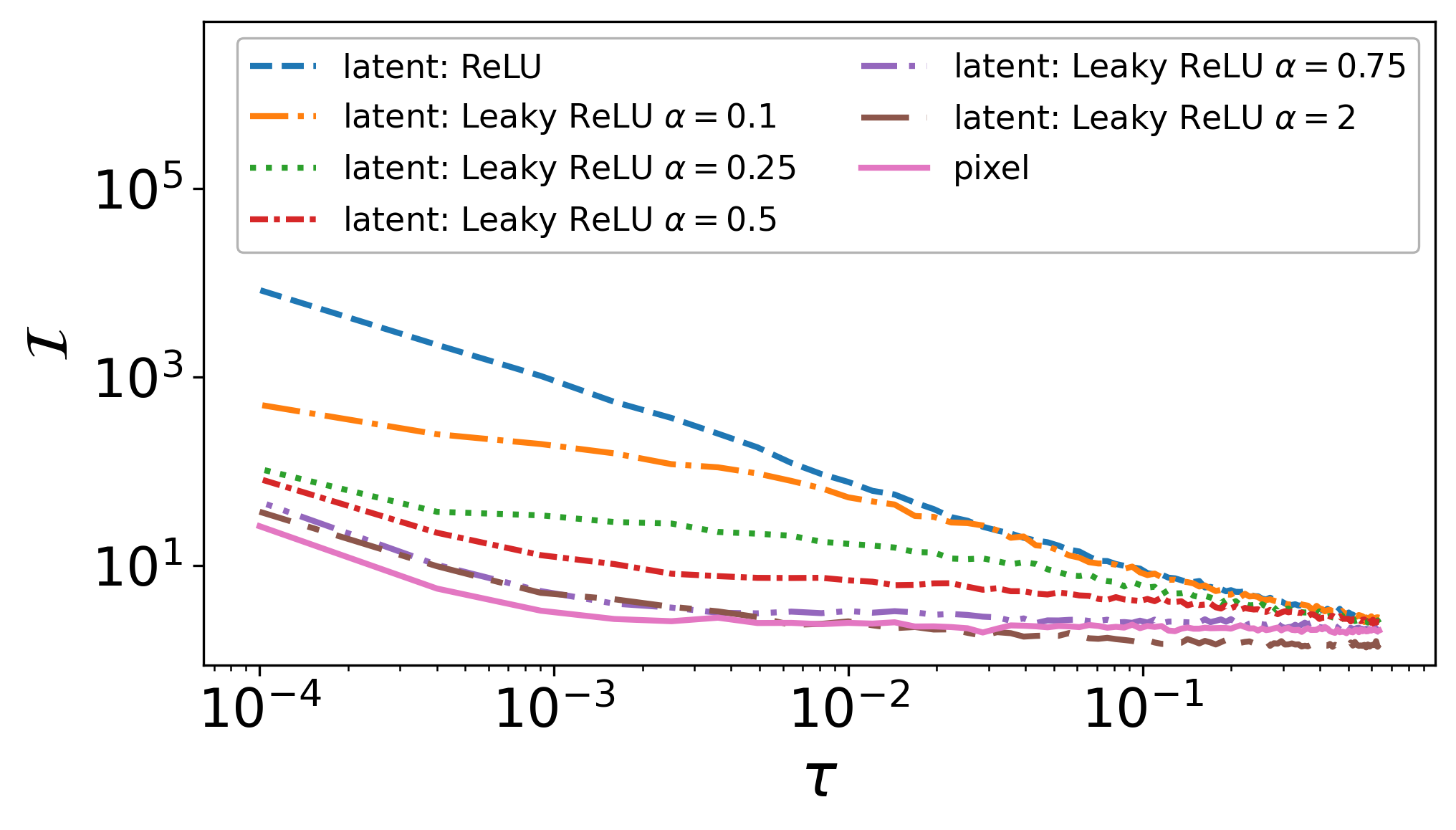}
    \caption{}
\end{subfigure}
\hspace{-2mm}
\begin{subfigure}{0.48\linewidth}
    \centering
    \includegraphics[width=\linewidth]{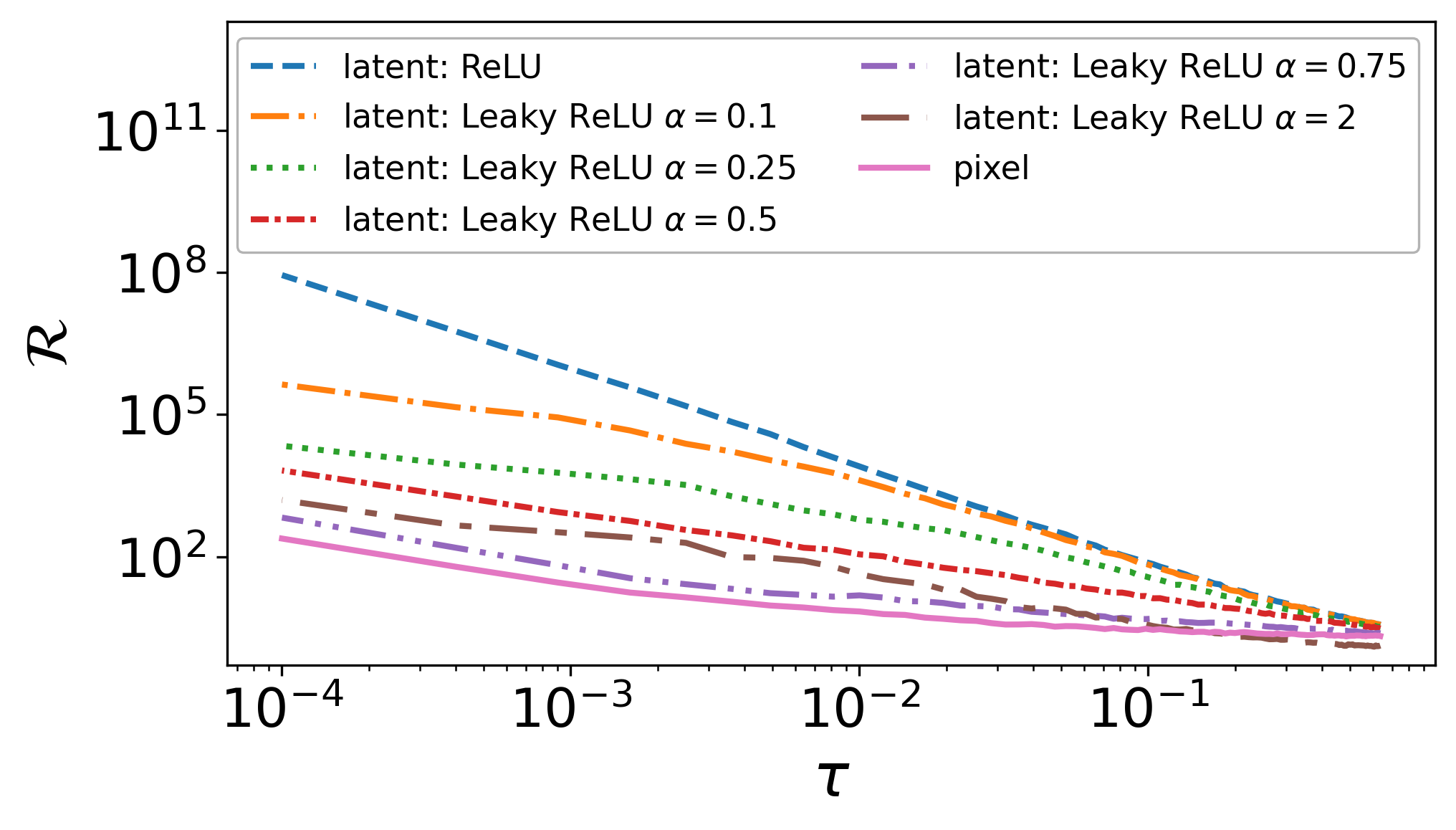}
    \caption{}
\end{subfigure}
\caption{
Values of (c) $\mathcal I$ and (d) $\mathcal R$ plotted versus the
noise variance $\tau$ for the ReLU-based experiments. \emph{Pixel}
curves correspond to models trained on
$\bx\sim\mathcal N(\bzero,\bI_2)$, while \emph{latent} curves
correspond to models trained on $\bz=E(\bx)$. For Leaky ReLU,
$\alpha$ denotes the negative slope.
}
\label{fig:gaussian_relu_activation}
\end{figure}

In \S\ref{sec:gaussian_deviation}, Figure~\ref{fig:gaussian_FIR_deviation} presents the FIR deviation $\mathcal{D}_{\mathcal R}$ as functions of the parameters $\delta_0$, $d$, and $\varepsilon_0$ for fixed noise variance $\tau$. Here we provide additional plots that instead show the dependence of $\mathcal{D}_{\mathcal R}$ on $\tau$ for the same experimental settings. These results are shown in Figures~\ref{app:fig:gaussian_FIR_deviation_delta}, \ref{app:fig:gaussian_FIR_deviation_d}, and \ref{app:fig:gaussian_FIR_deviation_eps}.

\begin{figure}[t!bp]
\centering
\begin{subfigure}{0.48\linewidth}
    \centering
    \includegraphics[width=\linewidth]{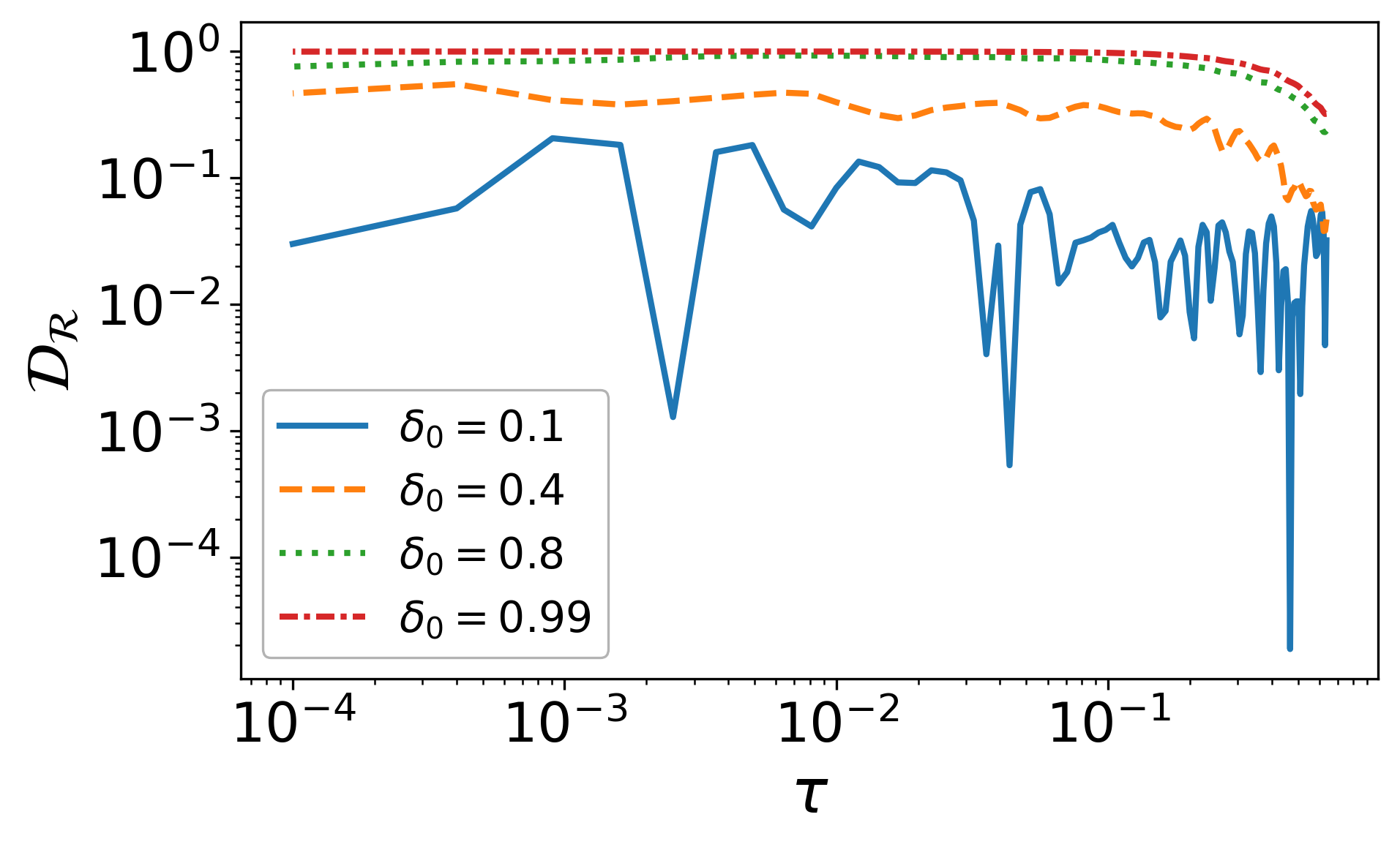}
\caption{}
\end{subfigure}
\begin{subfigure}{0.48\linewidth}
    \centering
    \includegraphics[width=\linewidth]{images/gaussian/gaussian_FIR_deviation_delta2.png}
\caption{}
\end{subfigure}
\caption{
FIR deviation $\mathcal{D}_{\mathcal R}$ vs.~(a) $\tau$ and (b) $\delta_0$ in toy settings. 
Data $\by\sim\mathcal N(\bzero,\bI_2)$ are embedded as $\bx=(y_1, y_2,0,\ldots,0)\in\mathbb R^D$ and encoded to $\bz=E(\bx)\in\mathbb R^d$. We consider $D=d=2$ and $E(\bx)=\bA\bx$ with $\bA=\mathrm{diag}(\sqrt{1+\delta_0},\sqrt{1-\delta_0})$.
We compute $\mathcal{D}_{\mathcal R}$ from 
\( \mathcal{R}^{(D)}(\mu_\tau)\) and \( \mathcal{R}^{(d)}((\mu_Z)_\tau) \)
 using diffusion models trained on $\bx$ and $\bz$. Solid line $y=1.25\delta_0$ in (b) serves as a linear reference. 
}
\label{app:fig:gaussian_FIR_deviation_delta}
\end{figure}

\begin{figure}[t!bp]
\centering
\begin{subfigure}{0.48\linewidth}
    \centering
    \includegraphics[width=\linewidth]{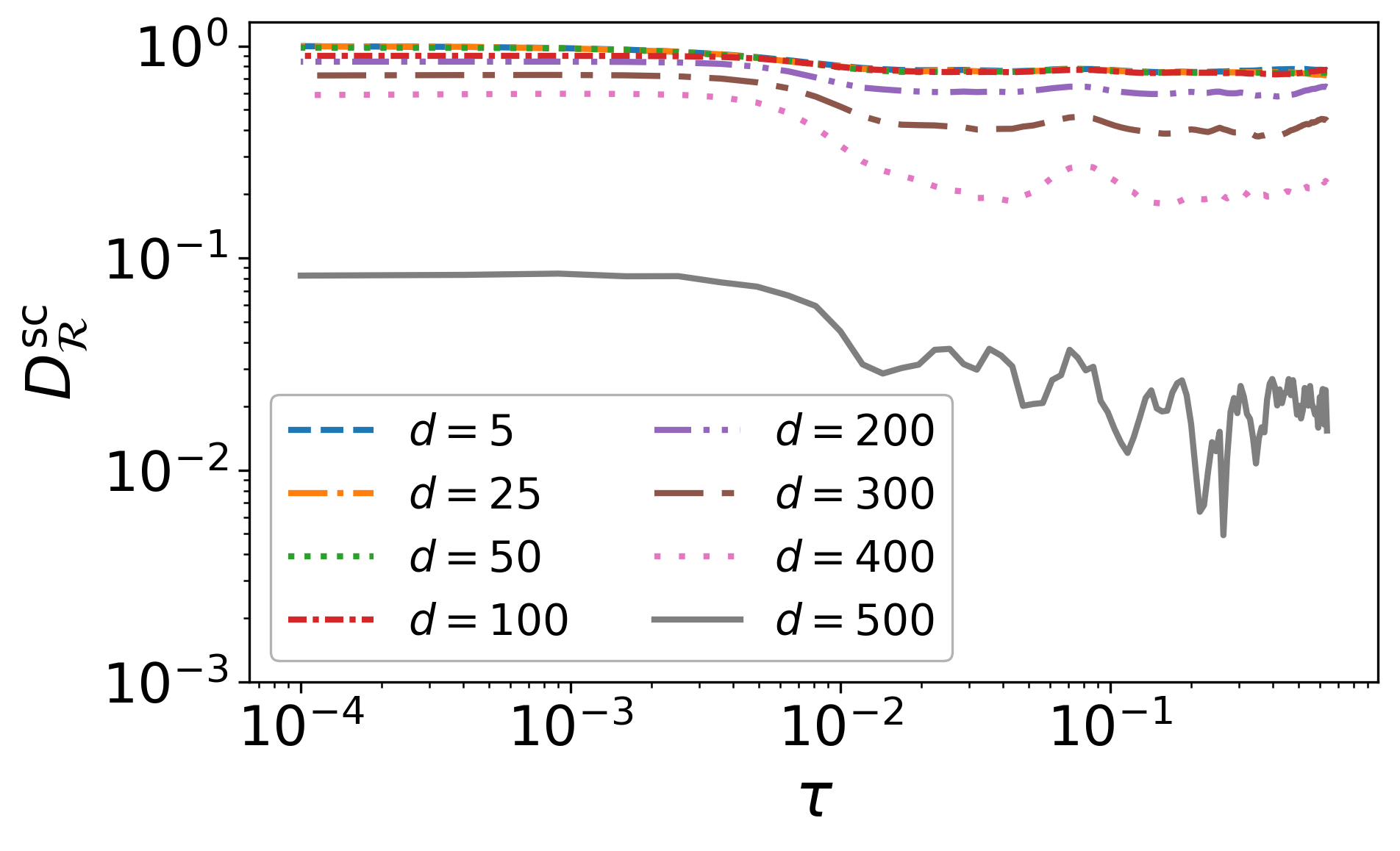}
\caption{}
\end{subfigure}
\begin{subfigure}{0.48\linewidth}
    \centering
    \includegraphics[width=\linewidth]{images/gaussian/gaussian_FIR_deviation_d2.png}
\caption{}
\end{subfigure}
\caption{
FIR deviation $\mathcal{D}_{\mathcal R}$ vs.~(a) $\tau$ and (b) $d$ in toy settings. 
Data $\by\sim\mathcal N(\bzero,\bI_2)$ are embedded as $\bx=(y_1,y_2,0,\ldots,0)\in\mathbb R^{D}$ with $D=512$, and 
$\bz=(y_1,y_2,0,\ldots,0)\in\mathbb R^d$.
We compute $\mathcal{D}_{\mathcal R}$ from 
\( \mathcal{R}^{(D)}(\mu_\tau)\) and \( \mathcal{R}^{(d)}((\mu_Z)_\tau) \)
 using diffusion models trained on $\bx$ and $\bz$. Solid line $y=\frac{D-d}{D-2}$ in (b) serves as a linear reference. 
}
\label{app:fig:gaussian_FIR_deviation_d}
\end{figure}

\begin{figure}[t!bp]
\centering
\begin{subfigure}{0.48\linewidth}
    \centering
    \includegraphics[width=\linewidth]{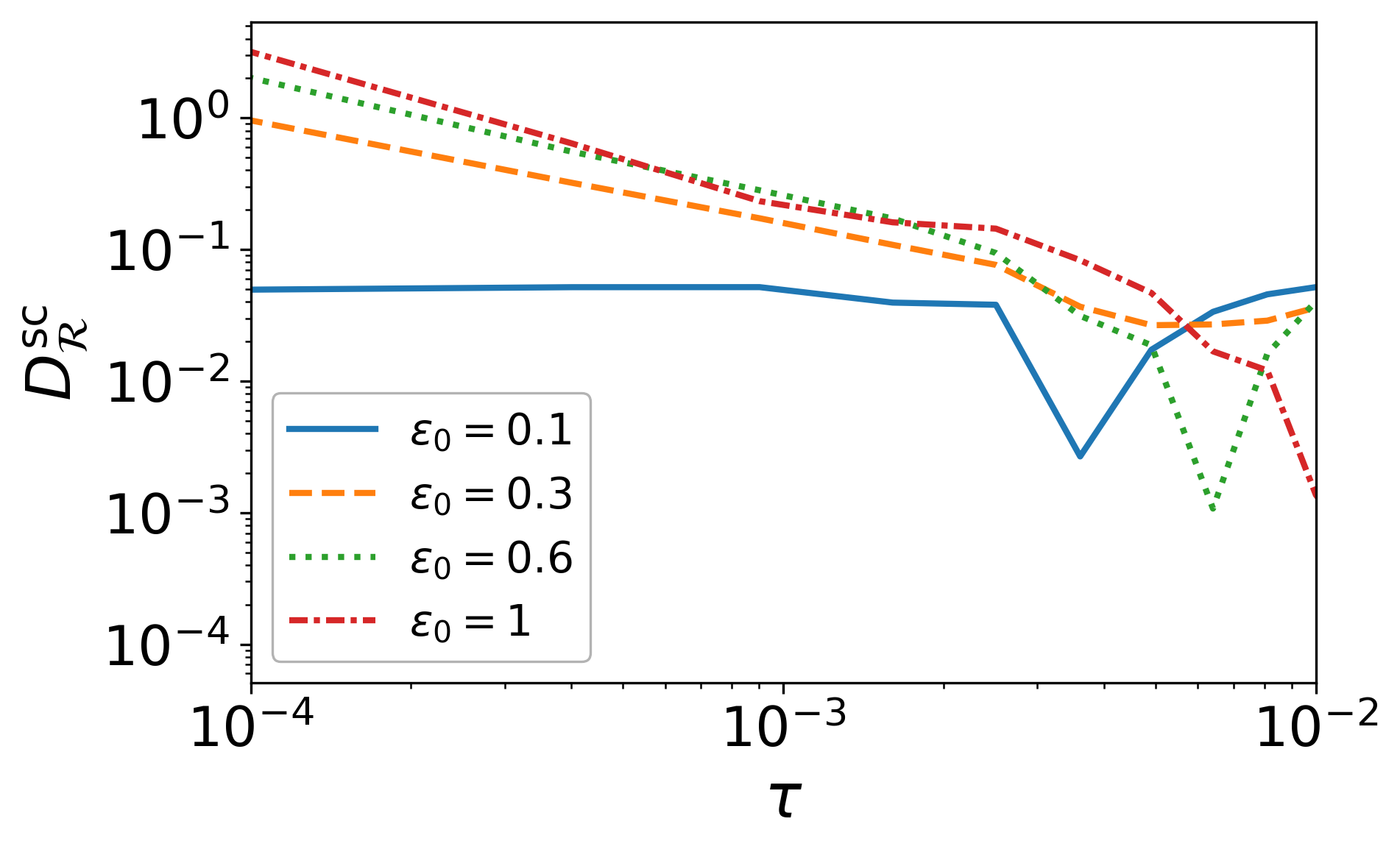}
\caption{}
\end{subfigure}
\begin{subfigure}{0.48\linewidth}
    \centering
    \includegraphics[width=\linewidth]{images/gaussian/gaussian_FIR_deviation_epsilon2.png}
\caption{}
\end{subfigure}
\caption{
FIR deviation $\mathcal{D}_{\mathcal R}$ vs.~(a) $\tau$ and (b) $\varepsilon_0$ in toy settings. 
Data $\by\sim\mathcal N(\bzero,\bI_2)$ are embedded as $\bx=(y_1, y_2,0,\ldots,0)\in\mathbb R^D$ and encoded to $\bz=E(\bx)\in\mathbb R^d$. We consider $D=d=3$ and $E((x_1,x_2,0),\varepsilon_0)=\bigl(\sin(\varepsilon_0 x_1)/\varepsilon_0,\; x_2,\; (1-\cos(\varepsilon_0 x_1))/\varepsilon_0\bigr)$.
We compute $\mathcal{D}_{\mathcal R}$ from 
\( \mathcal{R}^{(D)}(\mu_\tau)\) and \( \mathcal{R}^{(d)}((\mu_Z)_\tau) \)
 using diffusion models trained on $\bx$ and $\bz$. 
}
\label{app:fig:gaussian_FIR_deviation_eps}
\end{figure}

\section{Additional experiments}
\label{app:FFHQ_exp}
This section presents additional FI and FIR comparisons on FFHQ using latent representations from VAE, KL-AE, and NVAE encoders.

\subsection{FI and FIR Comparison on FFHQ with VAE and KL-AE Encodings}
\label{sec:exp_VAE}

\subsubsection{Experimental Setup}
We use the FFHQ dataset~\citep{karras2019style}, which contains
$70{,}000$ high-quality face images\footnote{Dataset license:
CC BY-NC-SA 4.0. Individual Flickr images are released under
CC BY 2.0, CC BY-NC 2.0, Public Domain Mark 1.0, CC0 1.0, or
U.S. Government Works licenses.}, and downsample the images to
$64\times64$. We compare FI and FIR in pixel space with those in
latent representations produced by a traditional VAE and a
KL-regularized autoencoder (KL-AE)~\citep{rombach2022high}. For the
latent-space experiments, images are encoded using either autoencoder,
and the resulting latents are scaled by their empirical standard
deviation following the latent-diffusion setup
\citep{rombach2022high}. We compare three diffusion models: a
pretrained pixel-space model and two models trained on the VAE and
KL-AE latent representations, respectively.

\textbf{VAE and KL-AE.}
The traditional VAE employs a hybrid encoder--decoder architecture
that combines convolutional layers with fully connected layers for
latent-space projection and feature reconstruction. We adopt the
architecture configuration provided for CelebA at $64\times64$
resolution in the implementation\footnote{\url{https://github.com/wonjunee/GPE_codes};
license: custom research-use license.} and set the latent dimension
to $256$. The VAE is trained on $64\times64$ FFHQ images using a
batch size of $100$ and a learning rate of $1\times10^{-4}$. We use a
$\beta$-annealing strategy that gradually increases the
KL-divergence weight. Specifically, we begin with warm-up stages using
$\beta\in{0,10^{-3},10^{-2}}$, training each stage for $5$ hours,
and then train with $\beta=0.1$ for $20$ hours on one GPU.

The KL-AE uses a fully convolutional architecture. We use the
pretrained KL autoencoder with downsampling factor $f=4$ from the
official latent-diffusion implementation\footnote{\url{https://github.com/CompVis/latent-diffusion/tree/main};
license: MIT.}. This model was trained on $256\times256$ OpenImages
data~\citep{kuznetsova2020open} as a general-purpose image
autoencoder. When applied to the $64\times64$ FFHQ images, its
downsampling factor produces latent representations of dimension
$3\times16\times16$.

\textbf{Diffusion Models.}
We use the convolutional U-Net architecture in the EDM
framework~\citep{karras2022elucidating}, following the official
implementation\footnote{\url{https://github.com/NVlabs/edm};
license: CC BY-NC-SA 4.0.}. For the pixel-space experiment, we use
the pretrained $64\times64$ FFHQ checkpoint provided in the EDM
repository. For the latent-space experiments, we train separate
diffusion models on the normalized representations produced by the
VAE and KL-AE. The pixel-space and latent-space models use the same
convolutional U-Net design and comparable training configurations,
with residual block configuration $(1,2,2,2)$ across resolution
levels.

For latent diffusion training, we use a batch size of $256$, dropout
rate $0.05$, and a learning-rate schedule consisting of a
$10{,}000$-kimg warm-up, where one kimg denotes one thousand training
images, followed by a constant learning rate of $2\times10^{-4}$.
Optimization is performed using Adam with $\beta_1=0.9$,
$\beta_2=0.999$, and $\epsilon=10^{-8}$. The models are trained for
$200{,}000$ kimg. The pretrained pixel-space model uses data
augmentation with probability $0.15$, whereas no data augmentation
is applied when training the latent-space diffusion models.

For sampling, we follow the standard EDM procedure using $40$ steps,
$\sigma_{\min}=0.002$, and $\sigma_{\max}=80$. We set
$S_{\mathrm{churn}}=0$ for the pixel-space and VAE latent models and
$S_{\mathrm{churn}}=40$ for the KL-AE latent model.

\subsubsection{Results}

\paragraph{FI and FIR.}

\begin{figure}[t!bp]
\centering
\begin{subfigure}{0.48\linewidth}
    \centering
    \includegraphics[width=\linewidth]{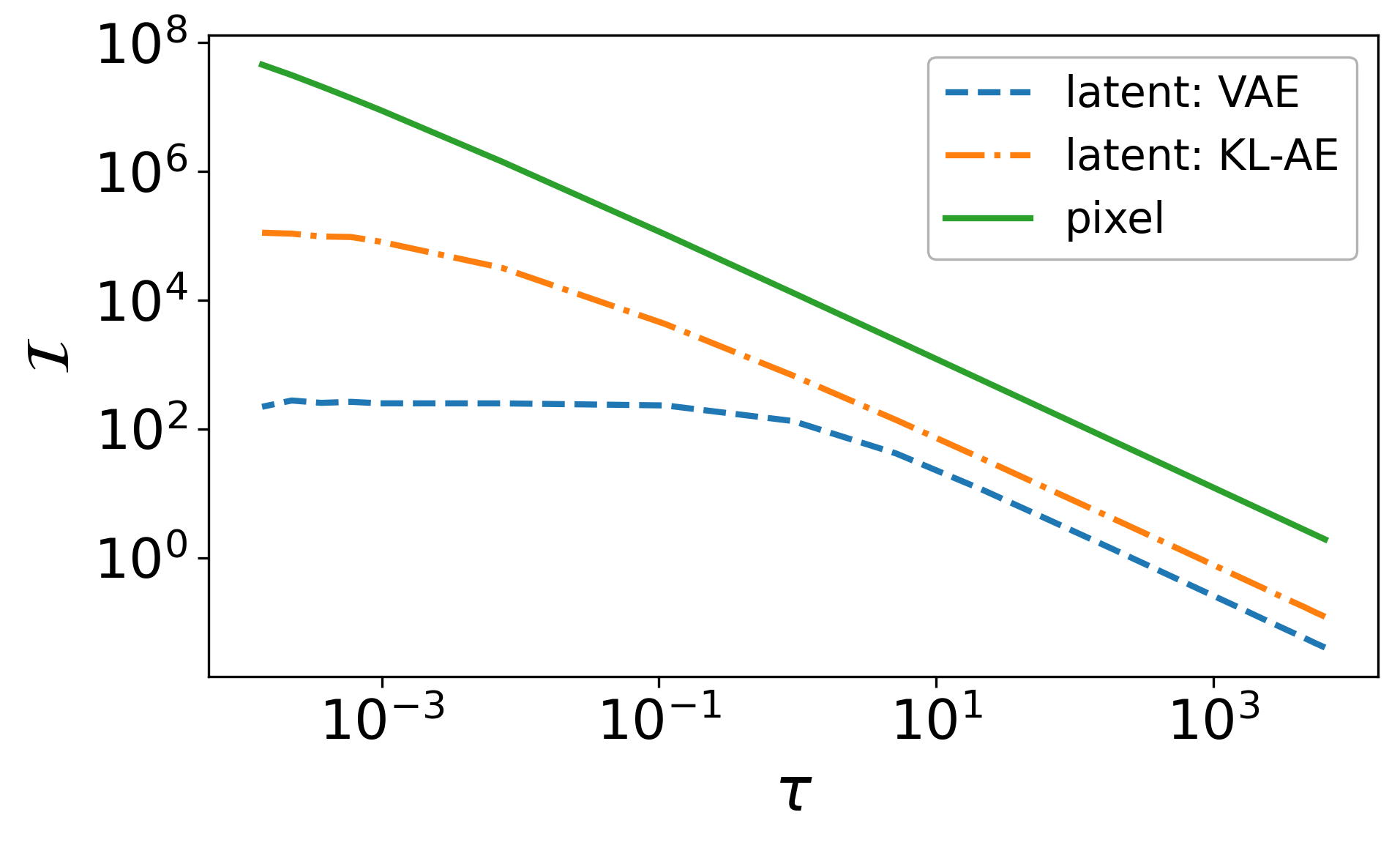}
\caption{}
\end{subfigure}
\begin{subfigure}{0.48\linewidth}
    \centering
    \includegraphics[width=\linewidth]{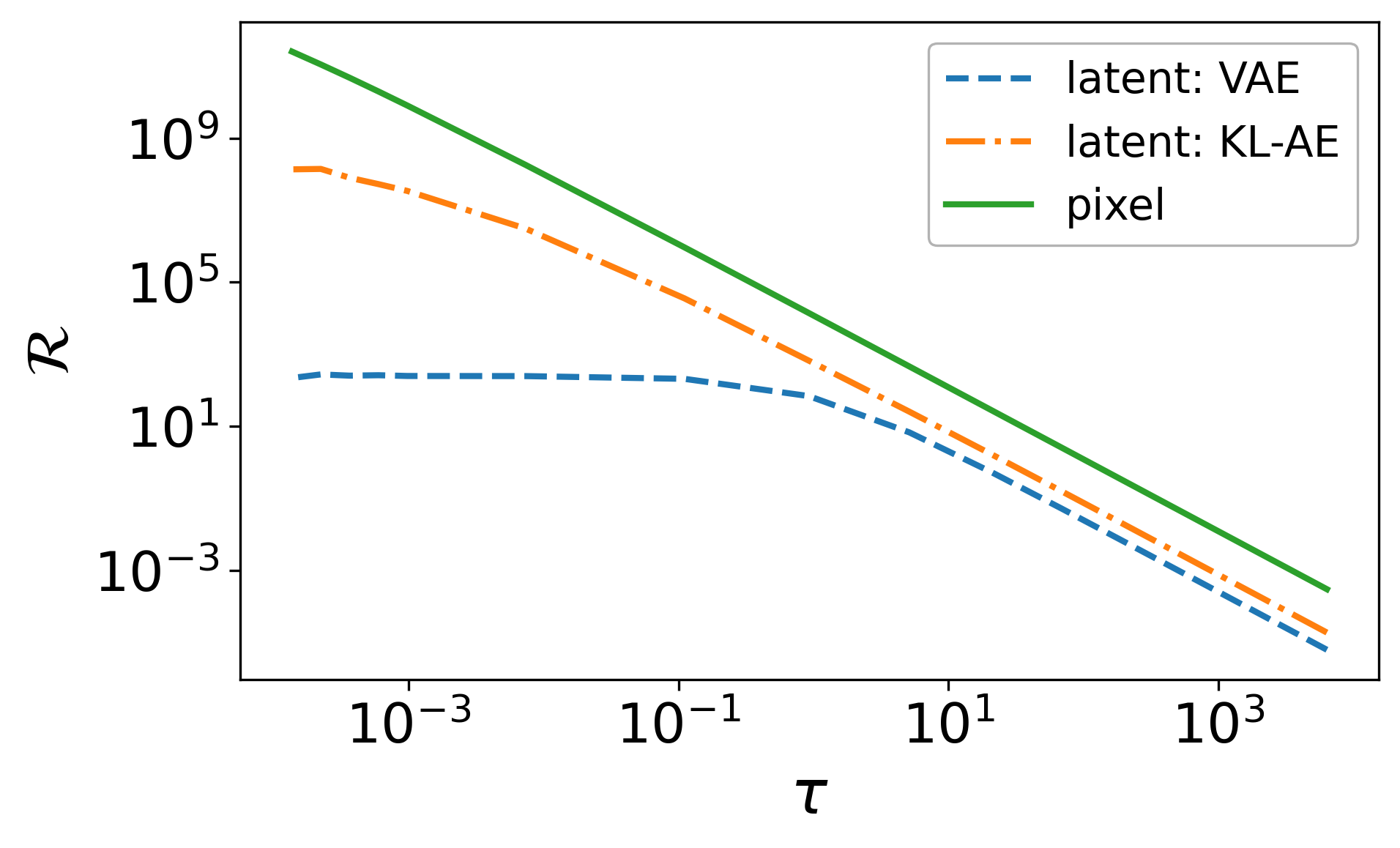}
\caption{}
\end{subfigure}
\caption{
Values of (a) $\mathcal I$ and (b) $\mathcal R$ plotted versus the noise variance $\tau$, computed from diffusion models trained on different data representations. The \emph{pixel} curves correspond to models trained directly on FFHQ images. The \emph{latent} curves correspond to models trained on latent representations of an image encoder (VAE or KL-AE) pretrained on FFHQ. We show $\sqrt{\tau} \in [0.01,80]$, excluding smaller $\tau$ due to numerical instability.
}
\label{fig:VAE_FI}
\end{figure}

\begin{figure}[t!bp]
\centering
\includegraphics[width=0.6\linewidth]{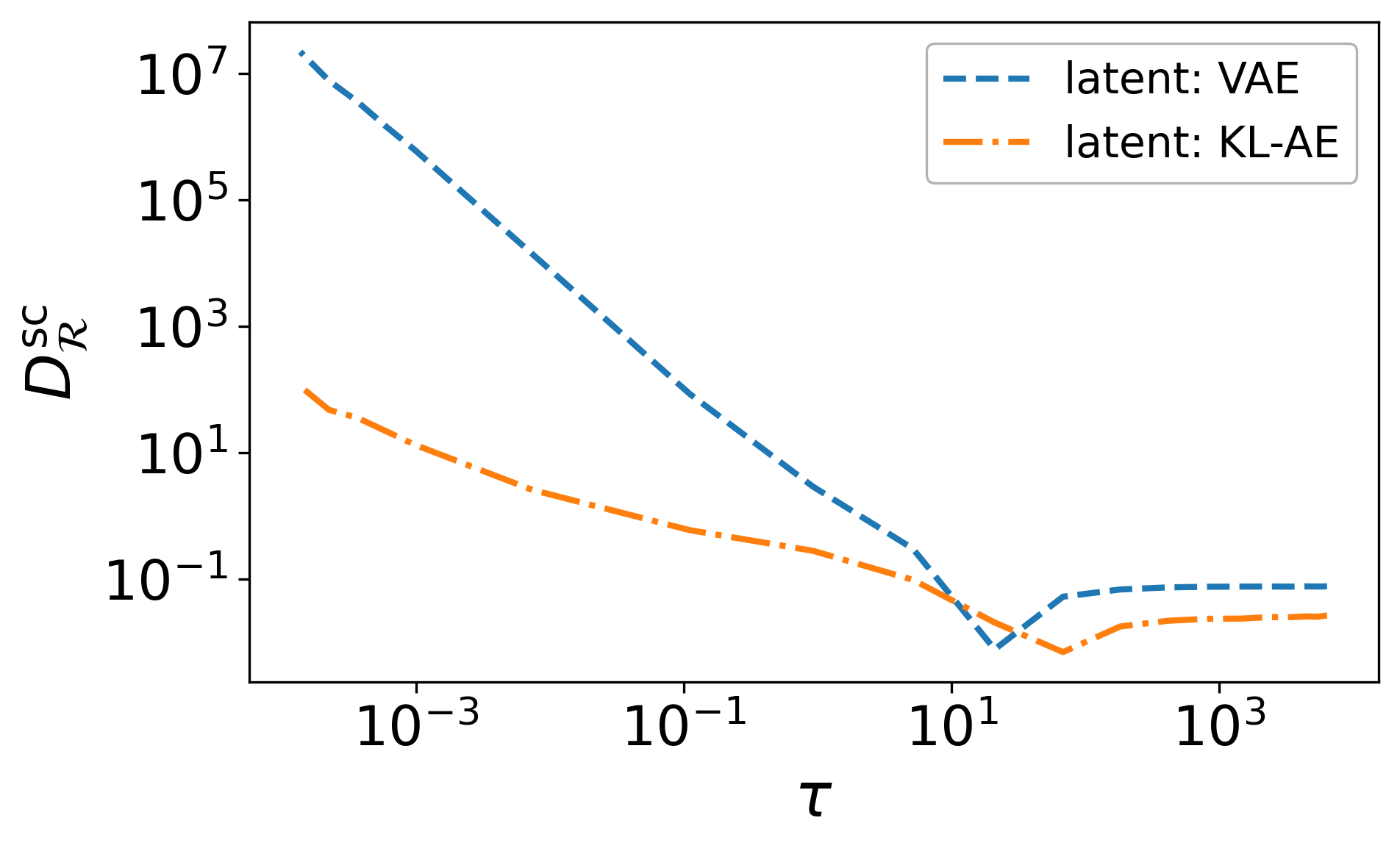}
\caption{
Dimension-normalized FIR deviation
$\mathcal{D}_{\mathcal R}^{\mathrm{sc}}$ vs.\ noise variance $\tau$
on FFHQ.
Here, $D=64\times64\times3$, $m=20$, and the latent dimension is
$d=256$ for the VAE and $d=3\times16\times16$ for the KL-AE.
FFHQ images $\bx$ are encoded as latents $\bz=E(\bx)$ using the VAE
or KL-AE.
}
\label{fig:VAE_FIR_deviation}
\end{figure}

Figure~\ref{fig:VAE_FI} compares the FI and FIR curves of the pixel-space, VAE latent-space, and KL-AE latent-space diffusion models, while Figure~\ref{fig:VAE_FIR_deviation} reports the scaled FIR deviation. For the scaling factor in $\mathcal D_{\mathcal R}^{\mathrm{sc}}$, we use the pixel dimension $D=3\times64\times64$, the latent dimensions $d=256$ for VAE and $d=3\times16\times16$ for KL-AE, and the intrinsic dimension $m=20$. The choice of $m$ follows intrinsic-dimension estimates of approximately $20$ for CelebA~\citep{pope2021intrinsic,liu2015deep}, another human-face dataset.

The KL-AE curves lie consistently closer to the pixel-space baseline than the VAE curves for both FI and FIR. KL-AE also has a smaller scaled FIR deviation, indicating better agreement with the pixel-space FIR.

The smaller FIR deviation of KL-AE is qualitatively consistent with the local mechanisms isolated by Theorem~\ref{thm:FIR_stability_nonlinear} and Corollary~\ref{cor:combined_local_bound}. KL-AE is fully convolutional and retains a spatial latent structure of size $3\times16\times16$, whereas the traditional VAE compresses each image into a $256$-dimensional vector. This architecture may better preserve local neighborhoods and may introduce smaller tangential distortion and local affine residuals.

The reported FFHQ quantity is a global FIR deviation, while the theorem concerns localized Gaussian-smoothed laws. We therefore interpret the empirical trends as qualitative evidence for the geometric mechanisms rather than as a direct numerical verification of the theorem or its constants.

\paragraph{Generation and reconstruction results.}
\label{app:model_performance}

\begin{figure}[t!bp]
\centering
\begin{subfigure}{0.32\linewidth}
    \centering
    \includegraphics[width=\linewidth]{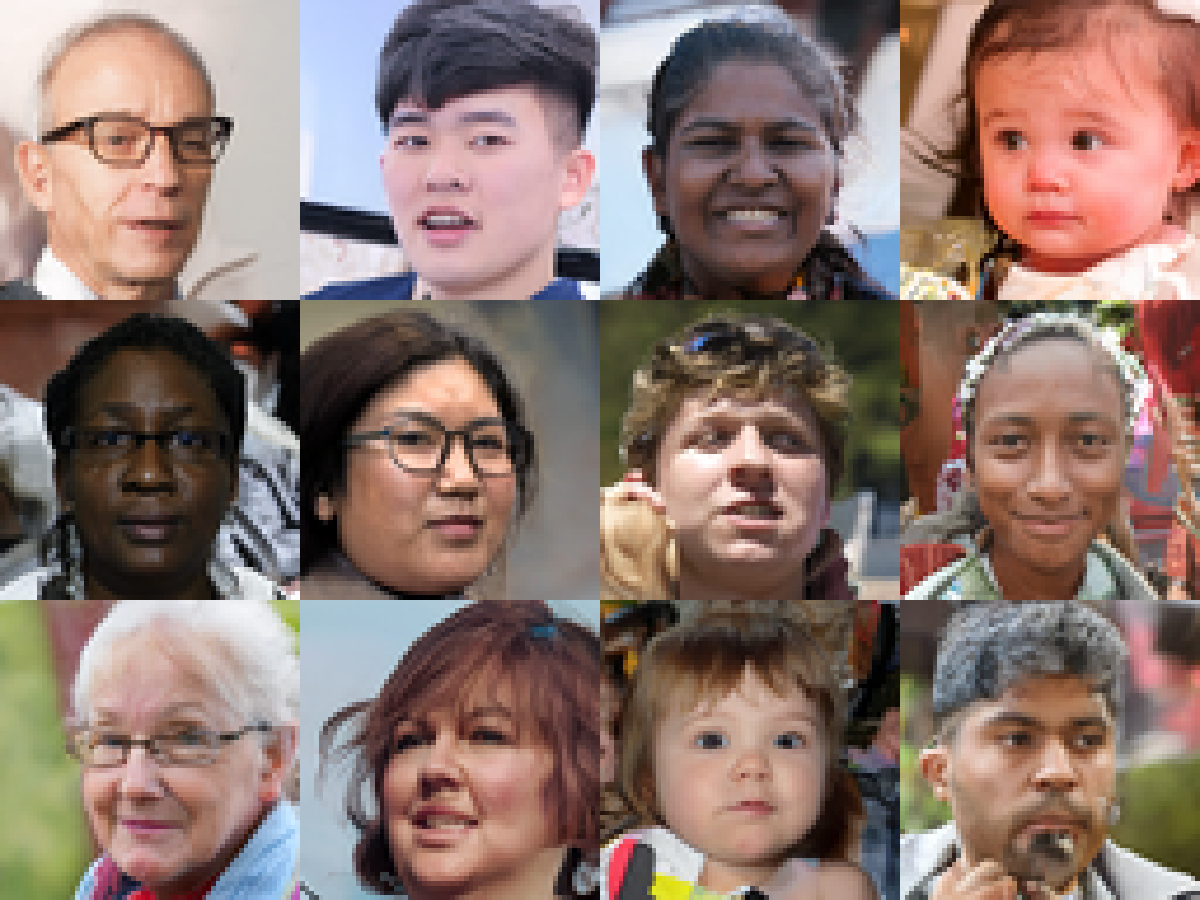}
\caption{}
\end{subfigure}
\hspace{-1.5mm}
\begin{subfigure}{0.32\linewidth}
    \centering
    \includegraphics[width=\linewidth]{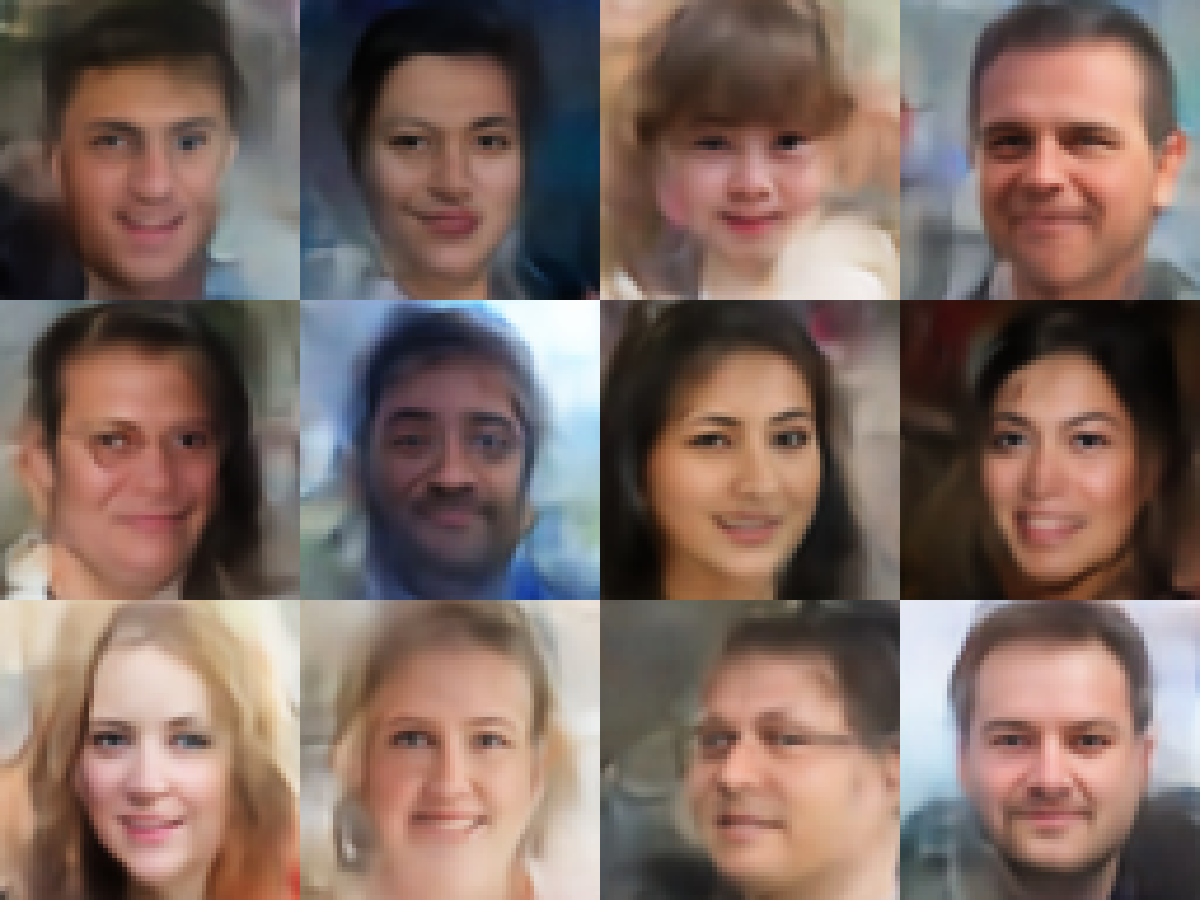}
\caption{}
\end{subfigure}
\hspace{-1.5mm}
\begin{subfigure}{0.32\linewidth}
    \centering
    \includegraphics[width=\linewidth]{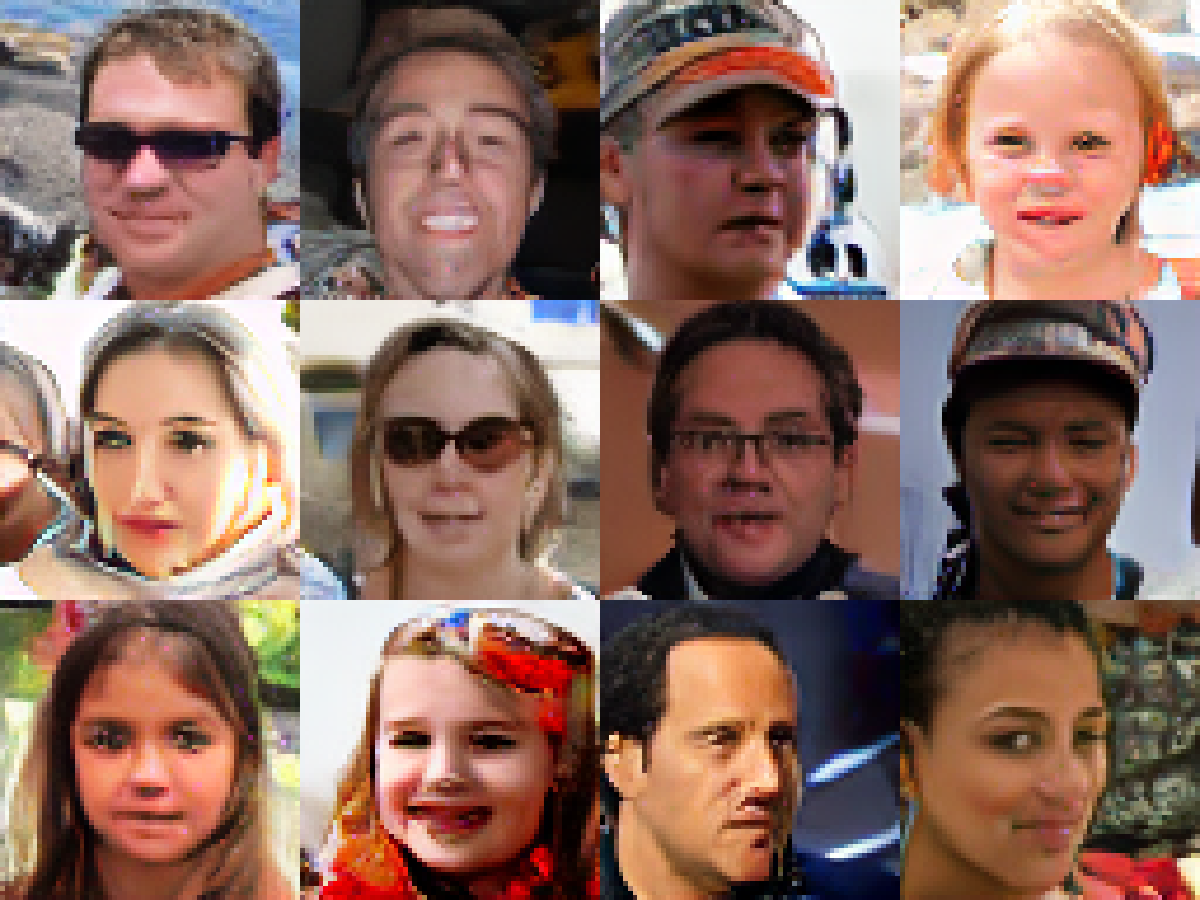}
\caption{}
\end{subfigure}
\caption{
(a) Samples from a diffusion model trained directly on $64\times64$ FFHQ. 
(b, c) Samples from latent diffusion models trained on (b) VAE and (c) KL-AE representations. 
Generated latents are mapped back to image space using their respective decoders.
}
\label{fig:generate_image}
\end{figure}

\begin{table}[htbp]
\centering
\caption{
Reconstruction MSE (mean squared error) for the VAE and KL-AE trained on FFHQ.
}
\setlength{\tabcolsep}{4pt}
\begin{tabular}{lc}
\toprule
\textbf{Model} & \textbf{Reconstruction MSE} \\
\midrule
VAE & 0.0094 \\
KL-AE & 0.41 \\
\bottomrule
\end{tabular}
\label{app:tab:reconstruction_metrics}
\end{table}

\begin{table}[htbp]
\centering
\caption{
Generation quality on FFHQ. 
FID (Fréchet Inception Distance) is computed using two reference distributions: the original FFHQ images and the reconstructions produced by the corresponding autoencoder. 
The first FID column measures generation quality relative to the original image distribution, while the second measures generation quality relative to the reconstruction distribution of the corresponding latent model. 
For the pixel model, images are generated by sampling a diffusion model trained directly on FFHQ. 
For the VAE and KL-AE models, images are generated by sampling diffusion models trained in the latent space and decoding the generated latents using the corresponding decoder. 
All diffusion models use the U-Net architecture within the EDM framework.
}
\setlength{\tabcolsep}{4pt}
\begin{tabular}{lcc}
\toprule
\textbf{Model} & \textbf{FID vs. FFHQ} & \textbf{FID vs. Reconstructions} \\
\midrule
Pixel Diffusion & 2.56 & -- \\
VAE (Latent Diffusion) & 62.51 & -- \\
KL-AE (Latent Diffusion) & 14.97 & 3.58 \\
\bottomrule
\end{tabular}
\label{app:tab:generation_metrics}
\end{table}

Figure~\ref{fig:generate_image} shows representative samples generated by the three diffusion models. Table~\ref{app:tab:reconstruction_metrics} reports the reconstruction MSE for the VAE and KL-AE on FFHQ, while Table~\ref{app:tab:generation_metrics} reports the FID scores of the pixel-space and latent-space diffusion models.

For KL-AE, we use a pretrained checkpoint trained on $256\times256$ OpenImages data. When applied to the $64\times64$ FFHQ images, its decoder tends to sharpen the reconstructions relative to the original downsampled images, leading to a larger reconstruction MSE. To account for this autoencoder-induced distribution shift, we also report its generation FID relative to KL-AE reconstructions.

The generation results are consistent with the FI/FIR comparison. The traditional VAE latent diffusion model performs poorly. For KL-AE, the FID relative to reconstructions is much lower than the FID relative to FFHQ, but remains higher than the pixel-space FID. The generated samples similarly show that KL-AE produces higher visual quality than VAE, although it still falls short of the pixel-space baseline.

\paragraph{Empirical estimation of $c$ and $C$.}
\label{app:lipschitz_c_VAE}

\begin{table}[t]
\centering
\caption{Empirical estimates $(\hat c,\hat C)$ of the constants $(c,C)$ in Proposition~\ref{prop:fisher_bi_lip} for encoders of the VAE and SD-VAE trained on FFHQ. The ratio $\hat C/\hat c$ provides an empirical estimate of the bi-Lipschitz constant.}
\begin{tabular}{lccc}
\toprule
Model & $\hat c$ & $\hat C$ & $\hat C/ \hat c$ \\
\midrule
VAE & 0.20 & 3.86 & 19.67 \\
SD-VAE & 1.60 & 4.83 & 3.02 \\
\bottomrule
\end{tabular}
\label{app:tab:VAE_bilipschitz_constants}
\end{table}

For the VAE and KL-AE encoders, we estimate the pairwise distance
distortion corresponding to the constants $c$ and $C$ in
Proposition~\ref{prop:fisher_bi_lip}. In the proposition, these
constants are defined through the tangential Jacobian bounds
\[
c\|\bv\|
\le
\|J_E(\bx)\bv\|
\le
C\|\bv\|,
\qquad
\bx\in\mathcal M,\quad
\bv\in T_{\bx}\mathcal M.
\]
To estimate these constants empirically, we consider the corresponding
distance inequality
\[
\hat c\|\bx_1-\bx_2\|
\le
\|E(\bx_1)-E(\bx_2)\|
\le
\hat C\|\bx_1-\bx_2\|.
\]
To estimate $(\hat c,\hat C)$, we sample $5000$ images
$\{\bx_i\}$ from the FFHQ dataset and compute their encoded
representations $\{E(\bx_i)\}$. For each pair $(i,j)$, we evaluate the
distortion ratio
\[
r_{ij}
=
\frac{\|E(\bx_i)-E(\bx_j)\|}
     {\|\bx_i-\bx_j\|}.
\]
Using all pairs, we estimate the constants as
\[
\hat c\approx\min_{i<j}r_{ij},
\qquad
\hat C\approx\max_{i<j}r_{ij}.
\]
The resulting estimates are reported in
Table~\ref{app:tab:VAE_bilipschitz_constants}, together with the ratio
$\hat C/\hat c$, which provides an empirical measure of the
bi-Lipschitz distortion.

As shown in Table~\ref{app:tab:VAE_bilipschitz_constants}, KL-AE has
a substantially smaller distortion ratio than the traditional VAE,
with $\hat C/\hat c=3.02$ compared with $19.67$. Its ratio is therefore
closer to one, which is qualitatively consistent with its improved FI
alignment and the near-isometry mechanism in
Proposition~\ref{prop:fisher_bi_lip}.

These finite-sample extrema do not certify the global bi-Lipschitz
constants $c$ and $C$, and we do not estimate the logarithmic
volume-Jacobian term $\widetilde\varepsilon$. Moreover,
Proposition~\ref{prop:fisher_bi_lip} concerns intrinsic FI under
deterministic pushforward, whereas the plotted quantities are ambient
Gaussian-smoothed FIs. The comparison is therefore descriptive rather
than a direct numerical verification of the proposition.

\paragraph{MMSE results.}
\label{app:dmmse}

\begin{figure}[t!bp]
\centering
\includegraphics[width=0.5\linewidth]{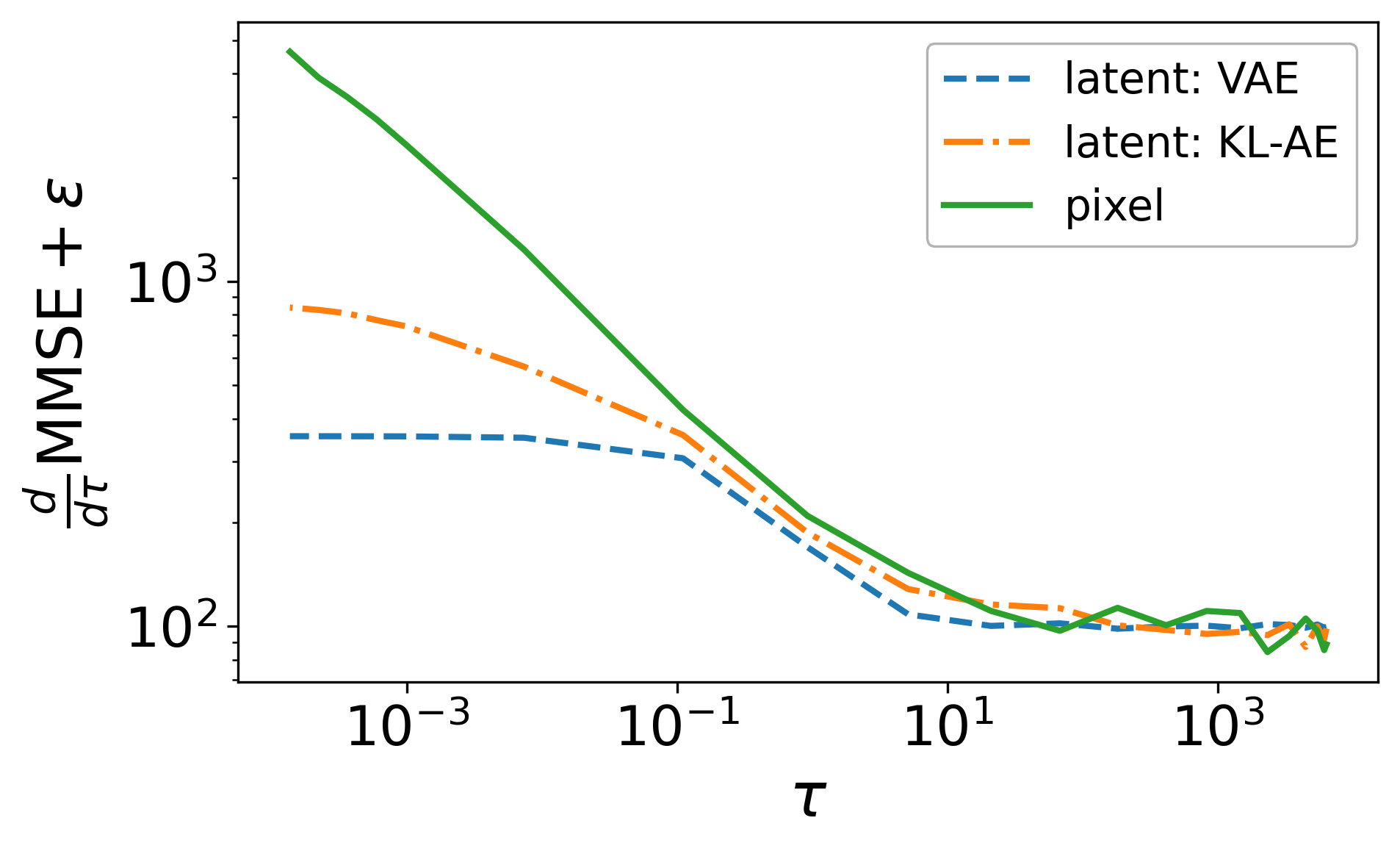}
\caption{
Sensitivity of MMSE to \(\tau\), plotted against the noise variance \(\tau\), computed from diffusion models trained on different data representations. 
The \emph{pixel} curve corresponds to a model trained directly on FFHQ images, while the \emph{latent} curves correspond to models trained on VAE and KL-AE latent representations. 
We show \(\sqrt{\tau}\in[0.01,80]\), excluding smaller noise levels due to numerical instability. 
For visualization on a logarithmic scale, we plot \(\frac{d}{d\tau}\mathrm{MMSE}+\varepsilon\) with \(\varepsilon=100\).}
\label{fig:dMMSE}
\end{figure}

Figure~\ref{fig:dMMSE} shows the sensitivity of MMSE to $\tau$, as defined in Equation~\ref{eq:denoising_resistance}. We set $k=D=3\times64\times64$ for the pixel-space model and $k=d$ for the latent-space models, with $d=256$ for VAE and $d=3\times16\times16$ for KL-AE.

The pixel-space curve has the largest sensitivity at small $\tau$, indicating that the MMSE changes rapidly in the low-noise regime. Both latent curves lie below the pixel curve for small $\tau$, suggesting that the latent representations are less sensitive to fine-scale noise perturbations than the original pixel space. The KL-AE curve follows the decreasing trend of the pixel curve more closely than the VAE.

\subsection{FI and FIR Comparison on FFHQ Data with NVAE Encoding}
\label{sec:exp_NVAE}

\subsubsection{Experimental Setup}
We conduct experiments on the FFHQ dataset~\citep{karras2019style}, using images downsampled to $64\times64$ resolution.
For the pixel-space experiment, we use the same diffusion model described in \S\ref{sec:exp_VAE}. For the latent-space experiments, we use a pretrained NVAE~\citep{vahdat2020nvae} to encode the images at multiple spatial resolutions and train a separate diffusion model on the latent representations selected at each resolution.

\textbf{NVAE.}
We use a publicly available pretrained NVAE model trained on FFHQ.\footnote{We use the publicly available pretrained NVAE model: \url{https://github.com/NVlabs/NVAE}; license: NVIDIA Source Code License.}
NVAE is a hierarchical variational autoencoder whose encoder produces multiple latent tensors at different spatial resolutions. To remain consistent with the pixel-space experiment, we provide $64\times64$ images as input to the encoder. Although the pretrained model was originally trained on $256\times256$ images, its fully convolutional encoder can process different input resolutions without architectural modifications.

The encoder produces latent tensors of size $20\times d_{\bz}\times d_{\bz}$ at spatial resolutions $d_{\bz}\in{32,16,8,4,2}$, where $20$ is the number of channels. When multiple latent tensors occur at the same resolution, we select the one closest to the input space in the encoder hierarchy. We omit $d_{\bz}=32$ because its dimensionality, $20\times32\times32$, exceeds that of the input image and therefore does not provide a compressed representation.

\textbf{Diffusion Models.}
For each selected latent resolution, we train a separate diffusion model under the EDM framework~\citep{karras2022elucidating}, using a convolutional U-Net as the score network. The architecture and training configuration follow those described in \S\ref{sec:exp_VAE}. For the smallest latent resolutions, we simplify the U-Net to accommodate the limited spatial dimensions: we use residual block configurations $(1,2,2)$ for $d_{\bz}=4$ and $(1,2)$ for $d_{\bz}=2$.

\begin{figure}[t!bp]
\centering
\begin{subfigure}{0.48\linewidth}
    \centering
    \includegraphics[width=\linewidth]{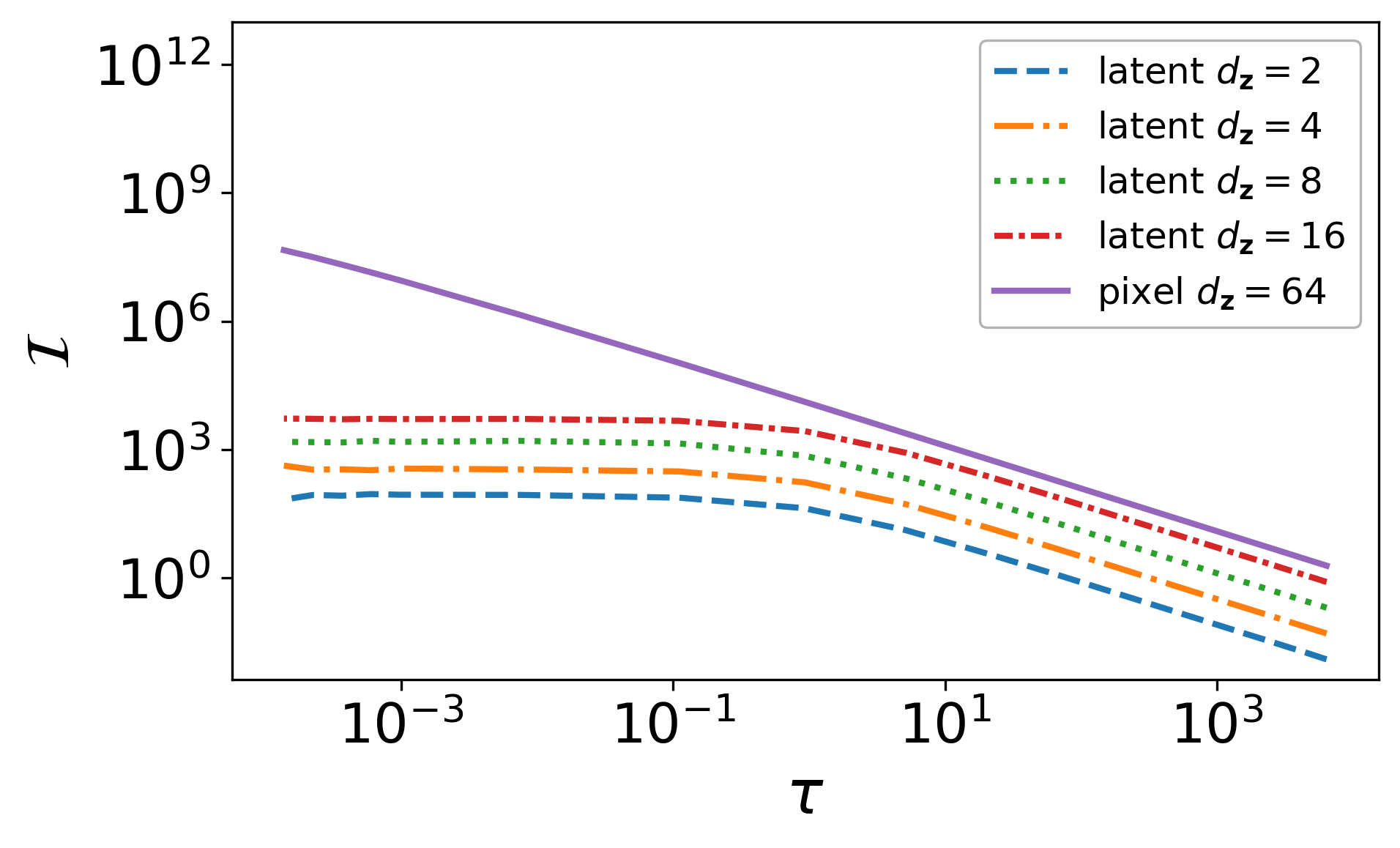}
\caption{}
\end{subfigure}
\begin{subfigure}{0.48\linewidth}
    \centering
    \includegraphics[width=\linewidth]{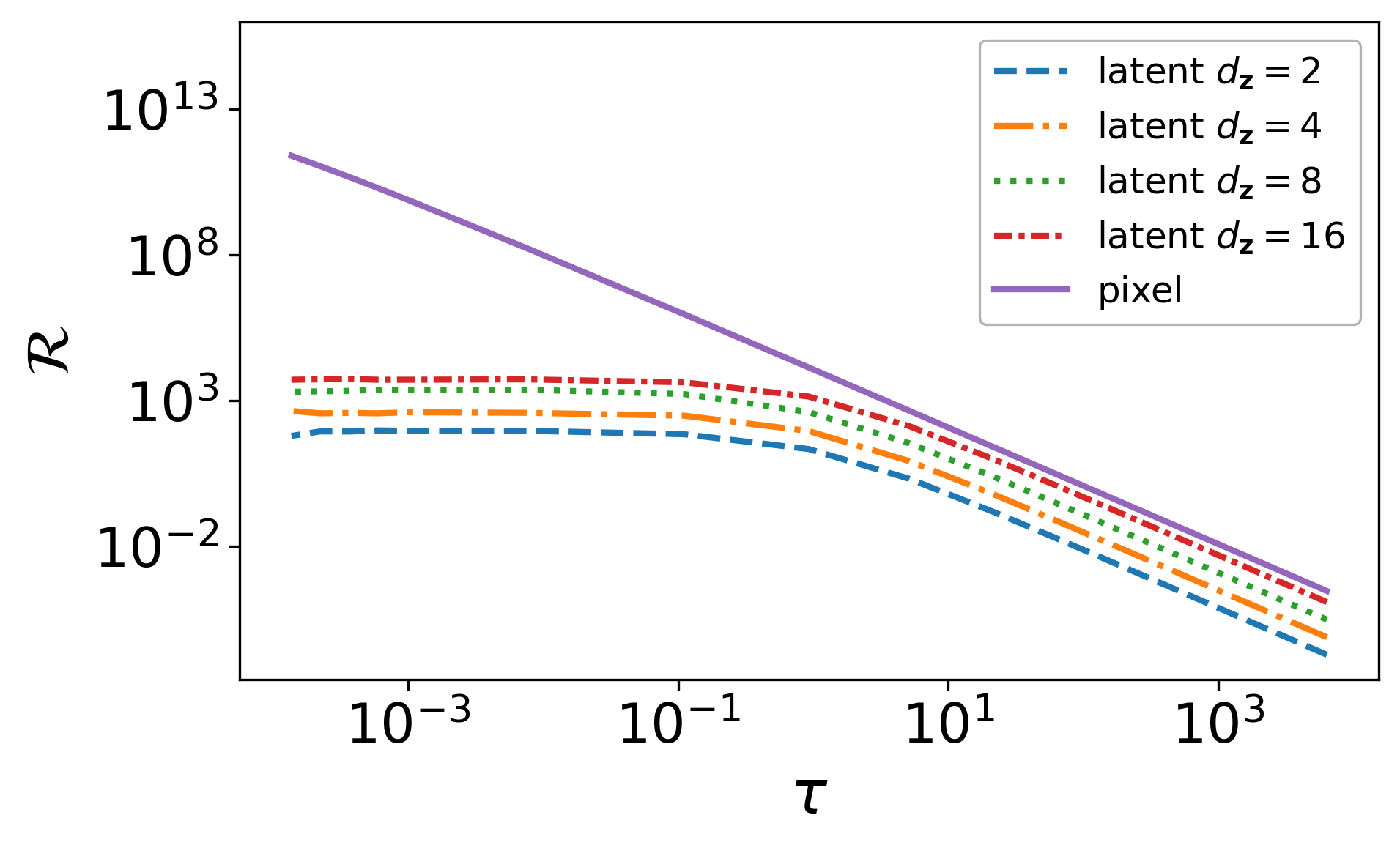}
\caption{}
\end{subfigure}
\caption{
Values of (a) $\mathcal I$ and (b) $\mathcal R$ plotted versus the noise variance $\tau$ for  diffusion models trained on different data representations. 
\emph{Pixel} and \emph{latent} curves denote models trained on FFHQ images and their pretrained NVAE latents, respectively. NVAE was pretrained on FFHQ, with spatial size $20 \times d_{\bz} \times d_{\bz}$. 
We show $\sqrt{\tau} \in [0.01,80]$, excluding smaller $\tau$ due to numerical instability.
}
\label{fig:NVAE_FI}
\end{figure}

\subsubsection{Results} Figure~\ref{fig:NVAE_FI} compares FI and FIR computed using the above pixel-space and latent diffusion models. As \(d_{\bz}\) increases, the FI and FIR curves for the latent models move closer to the pixel baseline. This behavior can be explained by the hierarchical structure of NVAE. Latent variables with larger \(d_{\bz}\) lie closer to the input space and therefore preserve more of the local geometry of the data manifold, resulting in less distortion in the mapping from data space to the latent representation.

We also observe that for both FI and FIR, the pixel and latent curves remain closely aligned for large \(\tau\), while the deviation increases noticeably once \(\tau\) becomes smaller. A possible explanation is that when \(\tau\) is large, the distribution \(p_\tau\) is strongly smoothed by Gaussian noise, so the corresponding score functions become dominated by noise and the resulting FIR values carry limited geometric information about the underlying representation.

\section{Spectra and Connection to Fisher Geometry}
\label{app:spectra}
In the main text, we introduce Fisher information (FI) and Fisher information rate (FIR) as tools for analyzing the diffusability of autoencoder representations.
In this section, we examine spectral analysis as an alternative perspective for studying diffusability. Corresponding to the experiments in Sections~\ref{sec:exp_VAE} and~\ref{sec:exp_NVAE}, which compare FI and FIR between pixel space and latent space, we analyze the power spectra of FFHQ images and of latent representations produced by the VAE, KL-AE, and NVAE encoders. This analysis provides a complementary view of the signal structure and allows us to contrast the information revealed by spectral behavior with that captured by FI and FIR.

We use the same $64\times64$ resolution FFHQ images and the same VAE, KL-AE, and NVAE models trained on FFHQ data as in the experiments in Sections~\ref{sec:exp_VAE} and~\ref{sec:exp_NVAE}. For FFHQ images, KL-AE latent representations, and NVAE latent representations, which are two-dimensional spatial tensors, we compute the two-dimensional power spectrum by applying the Fast Fourier Transform (FFT) over the spatial dimensions. For multi-channel inputs, the resulting spectra are averaged across channels to obtain a single spectrum per sample.

For the latent representations produced by the VAE encoder, the latent variable for each sample is a one-dimensional vector. To analyze their frequency characteristics, we compute the one-dimensional power spectrum using the discrete Fourier transform (FFT) along the latent dimension. Each latent vector is normalized prior to the transform.

Both the frequency axis $k$ and the power values $P$ are normalized by the signal length (or spatial resolution for the 2D case). 
To obtain a stable estimate, the spectrum is averaged over $10{,}000$ samples drawn from the dataset. 
The zero-frequency (DC) component corresponding to the mean signal level is excluded from the plots.

Figure~\ref{app:fig:latent_spectrum}(a) shows the power spectrum of FFHQ images. 
Figures~(b) and (c) show the spectra of NVAE latent representations at different spatial resolutions, while Figures~(d) and (e) show the spectra of the VAE and KL-AE latent representations. 
Although the spectra in (b), (c), and (e) also decrease with frequency, their shapes differ from the pixel-space spectrum in (a), and the decay behavior varies across resolutions, indicating a mismatch with the spectral structure of the original images.
Unlike the pixel spectrum in (a), the VAE latent spectrum in (d) exhibits strong oscillations and lacks a clear monotonic decay, further deviating from the spectral behavior of the original images.

Overall, the latent spectra do not always align with the spectrum of the original images. This suggests that spectral statistics alone may not be the most reliable diagnostic for assessing diffusability. In contrast, the FI and FIR analysis in Sections~\ref{sec:exp_VAE} and~\ref{sec:exp_NVAE} provides a more consistent characterization of the diffusability of the corresponding representations.
We also note that spectral statistics are not permutation invariant: for example, randomly shuffling the pixels of an image would destroy its spectrum by removing the spatial structure on which the Fourier transform relies. In contrast, FI and FIR depend on the geometry of the probability density through the score norm $\|\nabla_{\bx}\log p(\bx)\|^2$, and are therefore invariant under permutations of coordinates.

\begin{figure}[t]
\centering

\begin{subfigure}{0.48\linewidth}
    \centering
    \includegraphics[width=\linewidth]{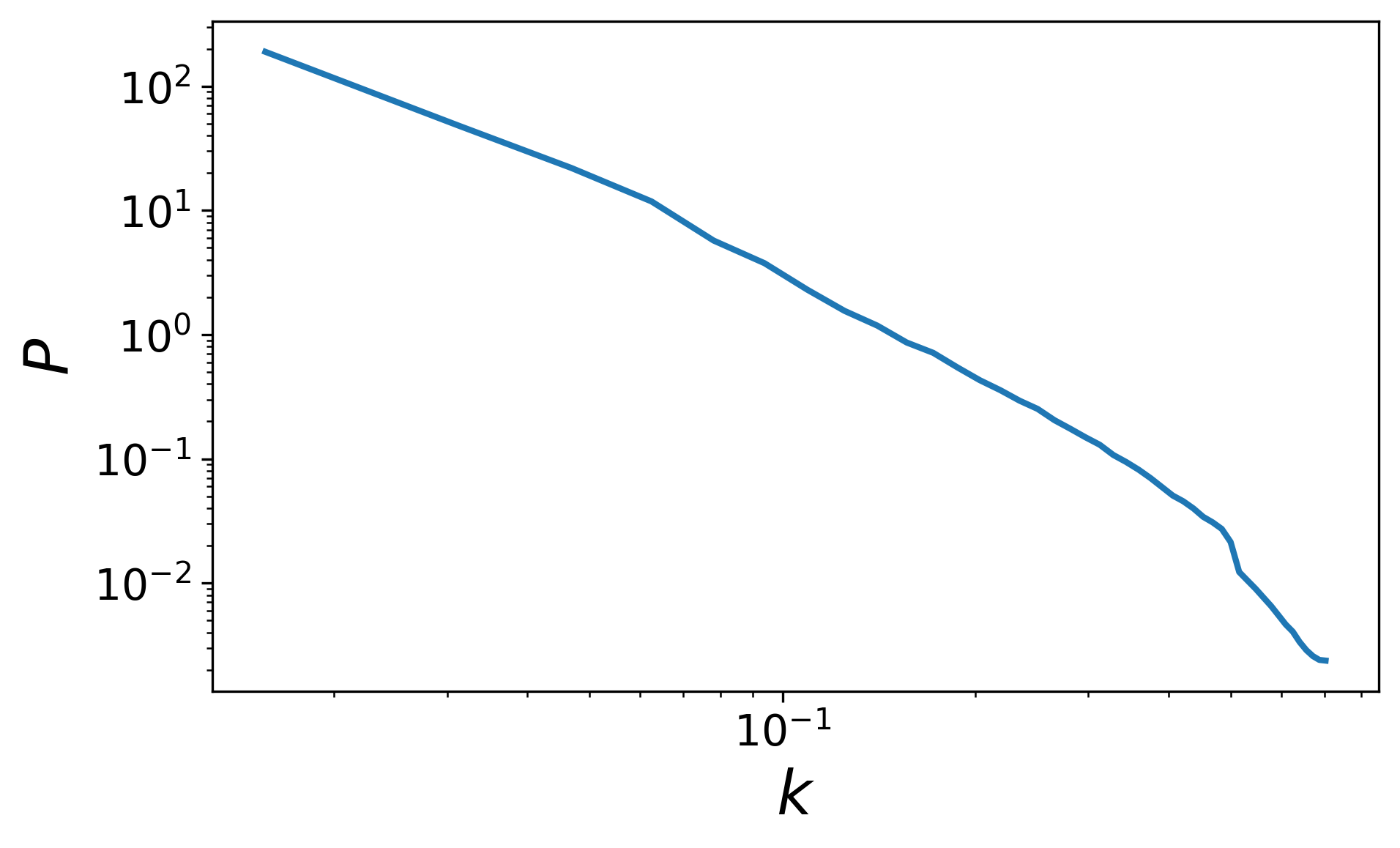}
    \caption{FFHQ Images}
\end{subfigure}
\hfill

\vspace{0.6em}

\begin{subfigure}{0.48\linewidth}
    \centering
    \includegraphics[width=\linewidth]{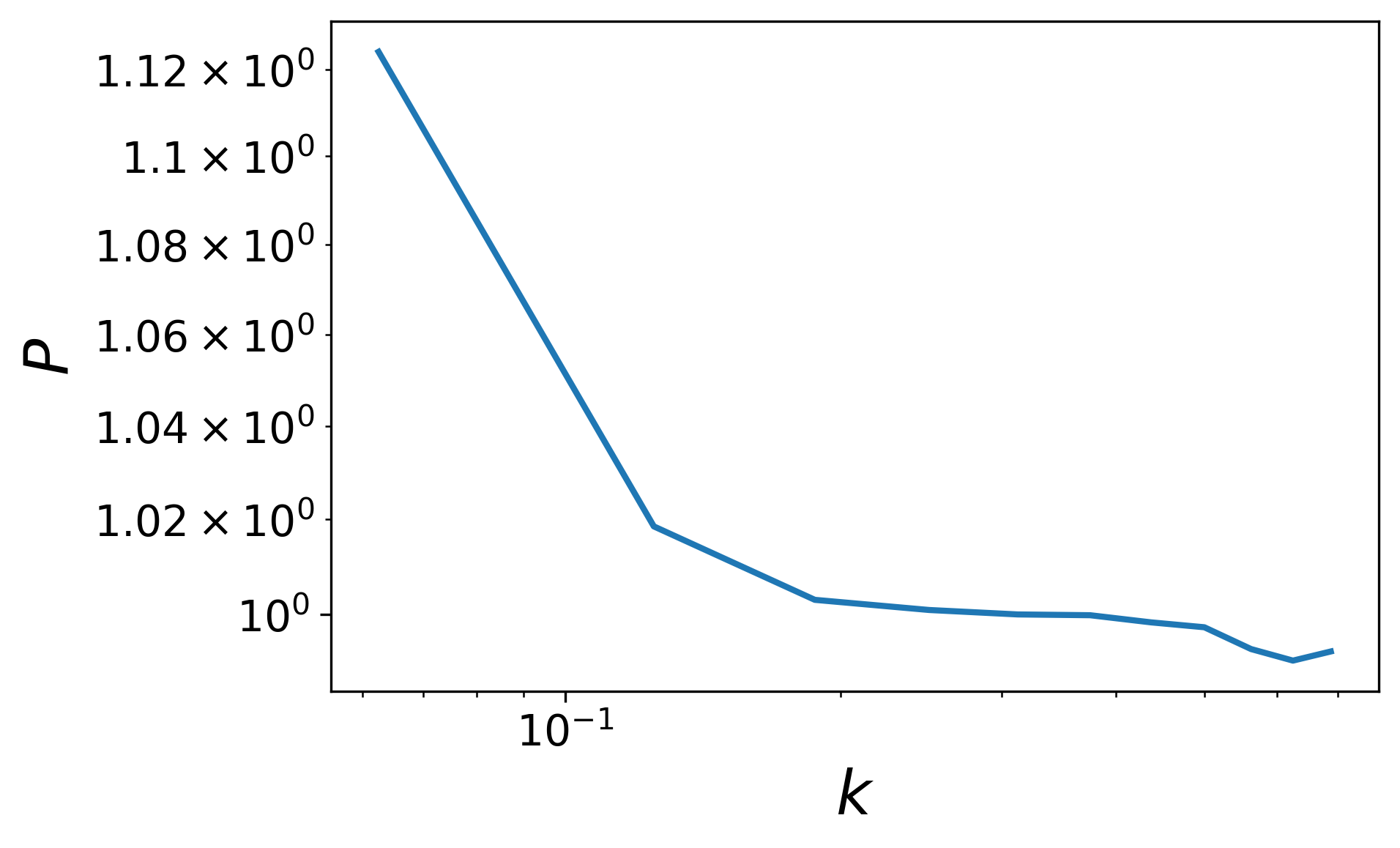}
    \caption{NVAE latents ($d_z=16$)}
\end{subfigure}
\hfill
\begin{subfigure}{0.48\linewidth}
    \centering
    \includegraphics[width=\linewidth]{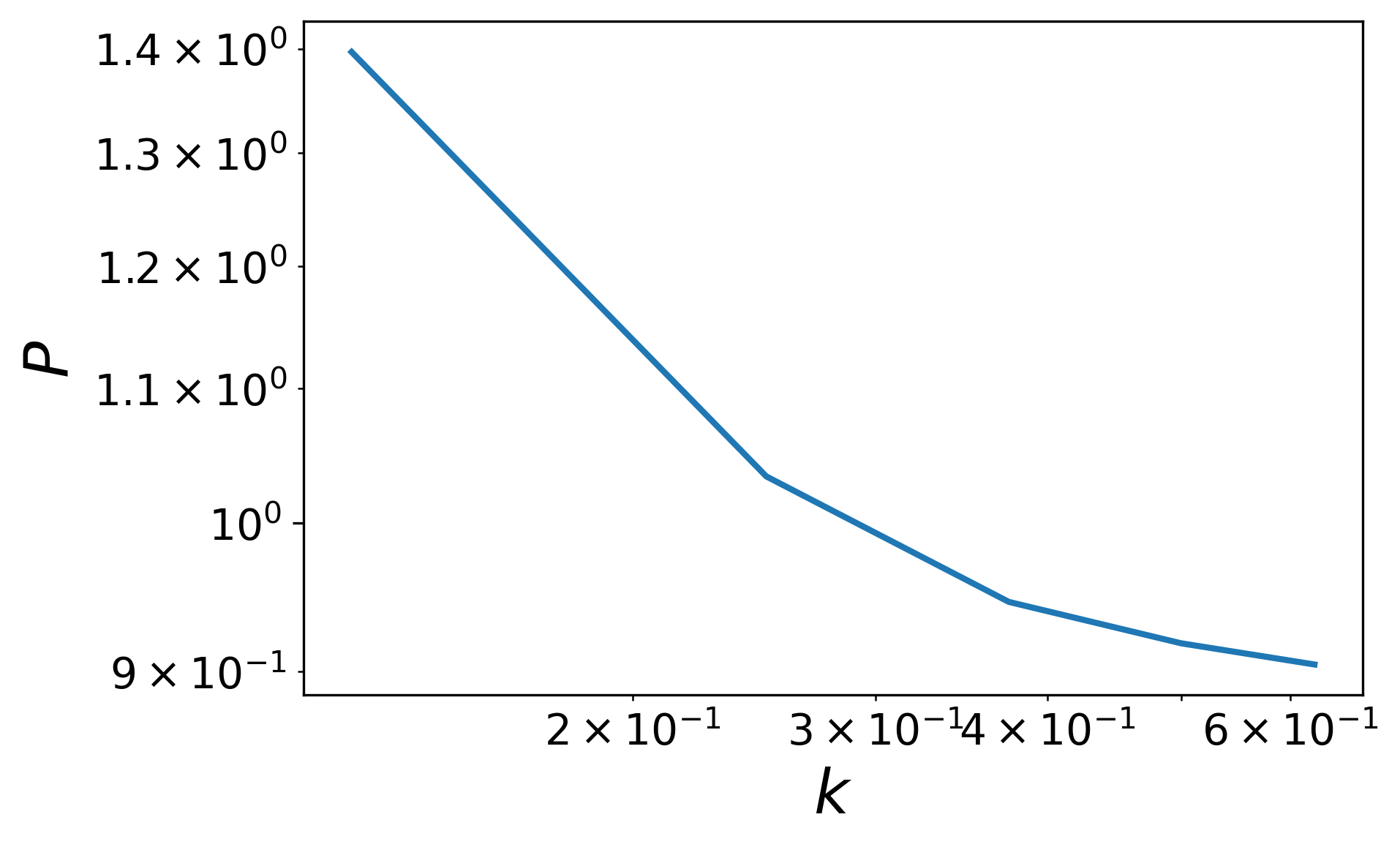}
    \caption{NVAE latents ($d_z=8$)}
\end{subfigure}

\vspace{0.6em}

\begin{subfigure}{0.48\linewidth}
    \centering
    \includegraphics[width=\linewidth]{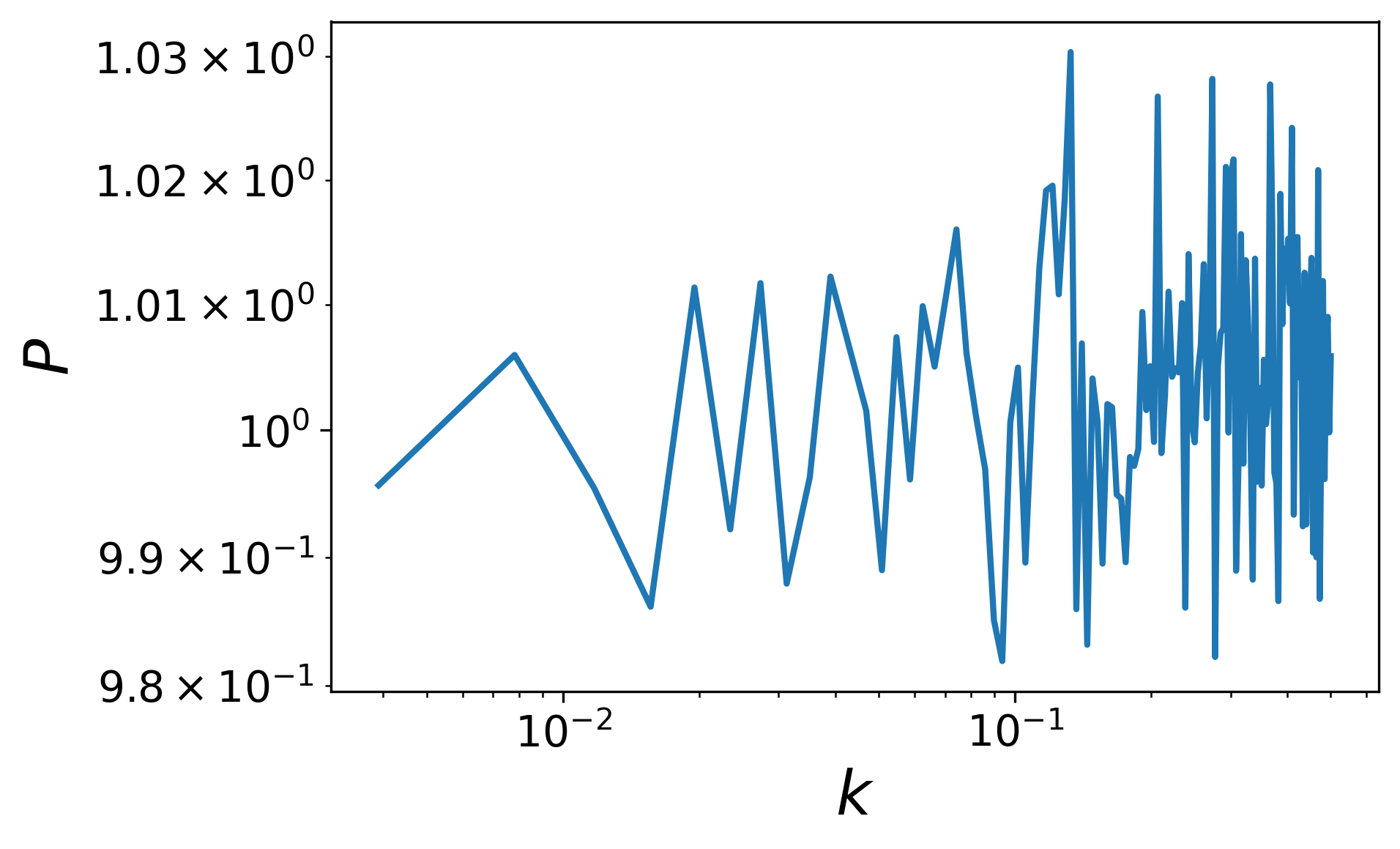}
    \caption{VAE latents}
\end{subfigure}
\hfill
\begin{subfigure}{0.48\linewidth}
    \centering
    \includegraphics[width=\linewidth]{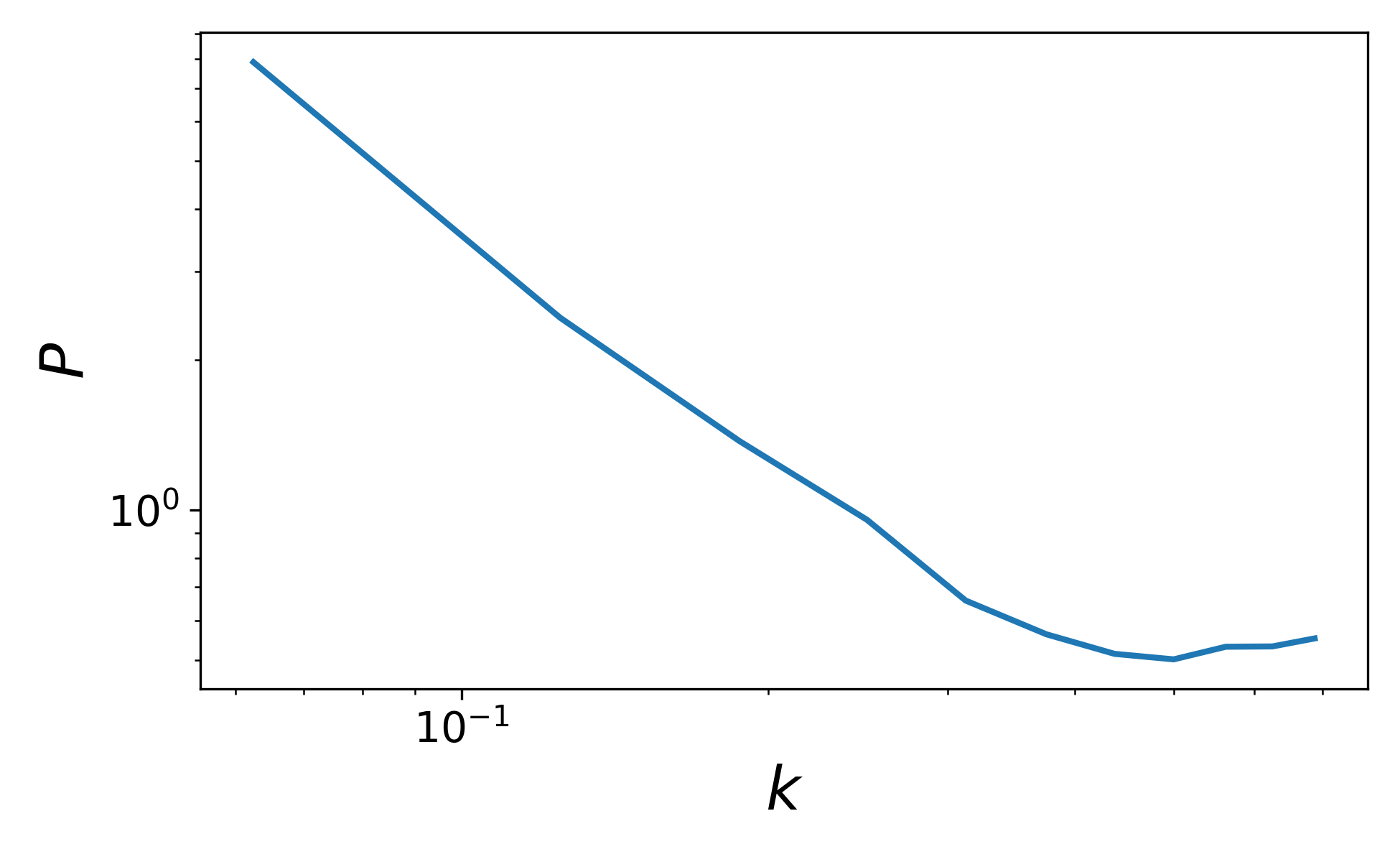}
    \caption{KL-AE latents}
\end{subfigure}

\caption{
(a) Power spectrum of FFHQ images.
(b)(c) Power spectra of NVAE latent representations with spatial resolutions $d_z=16$ and $8$.
(d)(e) Power spectra of VAE and KL-AE latent representations. The NVAE, VAE, and KL-AE models are pretrained on FFHQ images.
For FFHQ images, NVAE latents, and KL-AE latents, the spectrum is computed using the two-dimensional FFT over the spatial dimensions and averaged over channels.
For the one-dimensional latent vectors produced by the VAE encoder, the spectrum is computed using a one-dimensional FFT along the latent dimension.
The spectra are averaged over $10{,}000$ samples drawn from the dataset.
Both the frequency axis $k$ and the power values $P$ are normalized by the signal length (or spatial resolution for the 2D case).
The zero-frequency (DC) component corresponding to the mean signal level is excluded.}
\label{app:fig:latent_spectrum}
\end{figure}

\end{document}